\documentclass[letterpaper, 10 pt, conference]{ieeeconf}  

\IEEEoverridecommandlockouts  

\usepackage{moreverb,url}

\usepackage[colorlinks,bookmarksopen,bookmarksnumbered,citecolor=red,urlcolor=red]{hyperref}

\newcommand\BibTeX{{\rmfamily B\kern-.05em \textsc{i\kern-.025em b}\kern-.08em
T\kern-.1667em\lower.7ex\hbox{E}\kern-.125emX}}

\usepackage{times}
\usepackage{fix-cm}
\usepackage{etex}

\usepackage{dblfloatfix}

\usepackage{nag}

\makeatletter
\@ifpackageloaded{xcolor}{}{%
\usepackage[table,x11names,dvipsnames,svgnames]{xcolor}%
}
\makeatother

\usepackage{colortbl}

\usepackage{graphicx}
\usepackage{wrapfig}

\definecolorset{RGB}{lyft}{}{Red,194,39,36;Sunset,202,53,33;Orange,205,68,20;Amber,200,117,42;Yellow,242,169,52;Citron,186,188,44;Lime,112,159,33;Green,56,139,31;Mint,45,118,56;Teal,52,133,135;Cyan,60,132,202;Blue,55,94,248;Indigo,64,13,247;Purple,115,42,248;Pink,176,25,145;Rose,176,32,75}

\usepackage{cite}

\usepackage{microtype}

\usepackage{array}
\usepackage{multirow}
\usepackage{booktabs}
\usepackage{makecell} 

\ifcsname labelindent\endcsname
\let\labelindent\relax
\fi
\usepackage[inline]{enumitem}

\usepackage{subfig}

\newenvironment{lenumerate}[2][]
{\begin{enumerate}[label=(#2\arabic*),leftmargin=0.2in,itemindent=0.15in,#1]}
{\end{enumerate}}

\setlist*[enumerate,1]{label={\itshape\arabic*)}}

\makeatletter
\newcommand{\paragraphswithstop}{%
\let\copyparagraph\paragraph%
\renewcommand\paragraph[1]{\copyparagraph{##1.}}%
}
\makeatother

\usepackage[framemethod=tikz]{mdframed}

\usepackage{suffix}

\usepackage{environ}

\makeatletter
\newsavebox{\boxifnotempty}
\newcommand{\displayifnotempty}[3]{\sbox\boxifnotempty{#2}\setbox0=\hbox{\usebox{\boxifnotempty}\unskip}%
\ifdim\wd0=0pt
\else
 #1\usebox{\boxifnotempty}#3%
\fi%
}
\newcommand{\ifempty}[2]{\setbox0=\hbox{#1\unskip}%
\ifdim\wd0=0pt%
 #2%
\fi%
}
\newcommand{\ifnotempty}[2]{\setbox0=\hbox{#1\unskip}%
\ifdim\wd0>0pt%
 #2%
\fi%
}
\makeatother

\usepackage{algpseudocode}
\usepackage{algorithm}

\usepackage{scrlfile}

\makeatletter
\newcommand*\newstoreddef[1]{
  \BeforeClosingMainAux{%
    \immediate\write\@auxout{%
      \string\restoredef{#1}{\csname #1\endcsname}%
    }%
  }%
}
\newcommand*{\restoredef}[2]{
  \expandafter\gdef\csname stored@#1\endcsname{#2}%
}
\newcommand*{\storeddef}[1]{
  \@ifundefined{stored@#1}{0}{\csname stored@#1\endcsname}%
}
\makeatother

\usepackage{pageslts}
\NewEnviron{tee}{\BODY\typeout{Marker Tee [start] ^^J \BODY ^^JMaker Tee [end]}}

\input{math}
\newcommand{\real}[1]{\mathbb{R}^{#1}{}}

\newcommand{\bmat}[1]{\begin{bmatrix}#1\end{bmatrix}}

\newcommand{\transpose}{^\mathrm{T}}

\newcommand{\defeq}{\doteq}

\DeclarePairedDelimiter{\abs}{\lvert}{\rvert}

\DeclarePairedDelimiter{\norm}{\lVert}{\rVert}

\newcommand{\vct}[1]{\mathbf{#1}}

\DeclareMathOperator{\rank}{rank}

\DeclareMathOperator*{\argmin}{\arg\!\min}
\DeclareMathOperator*{\argmax}{\arg\!\max}

\DeclareMathOperator{\trace}{tr}

\DeclareMathOperator{\stack}{stack}

\newcommand{\intersect}{\cap}

\newcommand{\subjectto}{\textrm{subject to }}

\providecommand{\va}{\vct{a}}

\providecommand{\vb}{\vct{b}}

\providecommand{\vc}{\vct{c}}

\providecommand{\vd}{\vct{d}}

\providecommand{\ve}{\vct{e}}

\providecommand{\vp}{\vct{p}}

\providecommand{\vq}{\vct{q}}

\providecommand{\vr}{\vct{r}}

\providecommand{\vt}{\vct{t}}

\providecommand{\vu}{\vct{u}}

\providecommand{\vv}{\vct{v}}

\providecommand{\vw}{\vct{w}}

\providecommand{\vx}{\vct{x}}

\providecommand{\vy}{\vct{y}}

\providecommand{\vz}{\vct{z}}

\providecommand{\mA}{\vct{A}}
\providecommand{\mB}{\vct{B}}
\providecommand{\mC}{\vct{C}}

\providecommand{\mE}{\vct{E}}

\providecommand{\mG}{\vct{G}}

\providecommand{\mI}{\vct{I}}

\providecommand{\mM}{\vct{M}}

\providecommand{\mQ}{\vct{Q}}
\providecommand{\mR}{\vct{R}}
\providecommand{\mS}{\vct{S}}
\providecommand{\mT}{\vct{T}}
\providecommand{\mU}{\vct{U}}
\providecommand{\mV}{\vct{V}}
\providecommand{\mW}{\vct{W}}
\providecommand{\mX}{\vct{X}}
\providecommand{\mY}{\vct{Y}}
\providecommand{\mZ}{\vct{Z}}

\providecommand{\cB}{\mathcal{B}}
\providecommand{\cC}{\mathcal{C}}

\providecommand{\cE}{\mathcal{E}}

\providecommand{\cI}{\mathcal{I}}

\providecommand{\cP}{\mathcal{P}}

\providecommand{\cR}{\mathcal{R}}
\providecommand{\cS}{\mathcal{S}}
\providecommand{\cT}{\mathcal{T}}
\providecommand{\cU}{\mathcal{U}}
\providecommand{\cV}{\mathcal{V}}
\providecommand{\cW}{\mathcal{W}}

\providecommand{\cY}{\mathcal{Y}}

\def\IR{\mathbb R}			
\def\IS{\mathbb S}
\newcommand{\SO}[1]{\mathbf{SO}(#1)}

\usepackage{units}

\newcommand{\newcolorlabel}[2]{%
  \expandafter\newcommand\csname #1\endcsname[1]{%
    \colorbox{#2}{\color{white}\textsf{\textbf{##1}}}}%
}

\newcommand{\newcommenter}[2]{%
  \expandafter\newcommand\csname #1\endcsname[1]{%
    \fcolorbox{#2}{#2}{\color{white}\textsf{\textbf{#1}}}
    {\color{#2}##1}}%
  \expandafter\newcommand\csname at#1\endcsname{%
    \fcolorbox{#2}{#2}{\color{white}\textsf{\textbf{@#1}}}
    {\color{#2}}}%
  \expandafter\newcommand\csname #1hl\endcsname[2]{%
    \colorbox{#2}{\color{white}\textsf{\textbf{#1}}}\sethlcolor{Azure2}\hl{##2}~%
    \expandafter\ifx\csname commentarrow\endcsname\relax$\leftarrow$\else \commentarrow[#2]\fi~%
    {\color{#2}##1}}%
  \expandafter\newcommand\csname #1st\endcsname[2]{%
    \colorbox{#2}{\color{white}\textsf{\textbf{#1}}}\sout{##2}~%
    \expandafter\ifx\csname commentarrow\endcsname\relax$\leftarrow$\else \commentarrow[#2]\fi~%
    {\color{#2}##1}}%
}
\newcommenter{TODO}{DodgerBlue1}
\newcommenter{rtron}{Green3}

\usepackage{comment}

\usepackage{pdfcomment}

\usepackage{soul}

\usepackage[normalem]{ulem}

\usepackage{csquotes}

\usepackage{tikz}
\usetikzlibrary{calc}
\usetikzlibrary{matrix}
\usetikzlibrary{chains}
\usetikzlibrary{shapes.geometric}
\usetikzlibrary{arrows.meta}
\usetikzlibrary{decorations.pathreplacing}
\usetikzlibrary{backgrounds}

\tikzset{
  dim above/.style={to path={\pgfextra{
        \pgfinterruptpath
        \draw[>=latex,|->|] let
        \p1=($(\tikztostart)!1.5em!90:(\tikztotarget)$),
        \p2=($(\tikztotarget)!1.5em!-90:(\tikztostart)$)
        in(\p1) -- (\p2) node[pos=.5,sloped,above]{#1};
        \endpgfinterruptpath
      }
    }
  },
  dim double above/.style={to path={\pgfextra{
        \pgfinterruptpath
        \draw[>=latex,|->|] let
        \p1=($(\tikztostart)!3em!90:(\tikztotarget)$),
        \p2=($(\tikztotarget)!3em!-90:(\tikztostart)$)
        in(\p1) -- (\p2) node[pos=.5,sloped,above]{#1};
        \endpgfinterruptpath
      }
    }
  },
  dim below/.style={to path={\pgfextra{
        \pgfinterruptpath
        \draw[>=latex,|->|] let 
        \p1=($(\tikztostart)!-1em!-90:(\tikztotarget)$),
        \p2=($(\tikztotarget)!-1em!90:(\tikztostart)$)
        in (\p1) -- (\p2) node[pos=.5,sloped,below]{#1};
        \endpgfinterruptpath
      }
    }
  },
}

\tikzset{
    right angle quadrant/.code={
        \pgfmathsetmacro\quadranta{{1,1,-1,-1}[#1-1]}     
        \pgfmathsetmacro\quadrantb{{1,-1,-1,1}[#1-1]}},
    right angle quadrant=1, 
    right angle length/.code={\def\rightanglelength{#1}},   
    right angle length=2ex, 
    right angle symbol/.style n args={3}{
        insert path={
            let \p0 = ($(#1)!(#3)!(#2)$) in     
                let \p1 = ($(\p0)!\quadranta*\rightanglelength!(#3)$), 
                \p2 = ($(\p0)!\quadrantb*\rightanglelength!(#2)$) in 
                let \p3 = ($(\p1)+(\p2)-(\p0)$) in  
            (\p1) -- (\p3) -- (\p2)
        }
    }
}

\newcommand{\pgfextractangle}[3]{%
    \pgfmathanglebetweenpoints{\pgfpointanchor{#2}{center}}
                              {\pgfpointanchor{#3}{center}}
    \global\let#1\pgfmathresult  
}

\usetikzlibrary{shapes.arrows}
\newcommand{\commentarrow}[1][Azure4]{\tikz[baseline=-3pt]{\node[shape border uses incircle, fill=#1,rotate=180,single arrow, inner sep=1pt, minimum size=6pt, single arrow head extend=2pt]{};}}

\tikzset{ax/.style={-latex,line width=2pt}}

\tikzset{camera/.style={fill=Sienna1,fill opacity=0.5},%
image plane/.style={draw=RoyalBlue3,line width=2pt}}

\newcommand{\Wframe}{\prescript{\cW}{}}

\newcommand{\Bframe}[1][]{\prescript{\cB_{#1}}{}}

\usepackage{bm}

\newcommand{\Mat}[1]{\begin{bmatrix}#1\end{bmatrix}}	
\newcommand{\vect}[1]{\bm{\mathrm{#1}}}								
\def\omat{\vect{0}} 	
\def\ident{\vect{I}} 	
\def\trans{\mathrm{T}}

\newcommand{\aleq}[1]{\begin{align}#1\end{align}} 				

\def\IR{\mathbb R}					

\newlength{\figwidth}
\def\autoref#1{\Fref{#1}}
\makeatletter
\def\mkfancyprefix#1#2{%
  \@namedef{fancyref#1labelprefix}{#1}%
  \begingroup\def\x{\endgroup\frefformat{vario}}%
    \expandafter\x\csname fancyref#1labelprefix\endcsname
      {\MakeLowercase{#2}\fancyrefdefaultspacing##1}%
  \begingroup\def\x{\endgroup\Frefformat{vario}}%
    \expandafter\x\csname fancyref#1labelprefix\endcsname
      {#2\fancyrefdefaultspacing##1}%
			\begingroup\def\x{\endgroup\frefformat{plain}}%
    \expandafter\x\csname fancyref#1labelprefix\endcsname
      {\MakeLowercase{#2}\fancyrefdefaultspacing##1}%
  \begingroup\def\x{\endgroup\Frefformat{plain}}%
    \expandafter\x\csname fancyref#1labelprefix\endcsname
      {#2\fancyrefdefaultspacing##1}%
}
\makeatother

\def\IR{\mathbb R}			
\def\IS{\mathbb S}

\def\vec{\mathrm{vec}}

\def\omat{\vect{0}} 
\def\ident{\vect{I}} 

\def\ave{\cE}

\def\trans{{\scriptstyle\ensuremath{\mathsf{T}}}}

\newcommand{\capgreek}[1]{\text{\boldmath{$#1$}}} 

\usepackage{pifont}
\newcommand{\cmark}{\ding{51}}%
\newcommand{\xmark}{\ding{55}}%
\newcommenter{todo}{Firebrick1}
\newcommenter{tomwu}{blue}
\newcommenter{skipped}{white!80!red}

\usepackage{titlesec}

\titleformat{\subsubsection}[block]   
  {\normalfont\normalsize\bfseries}
  {\thesubsubsection}{1em}{}

\titlespacing*{\subsubsection}
  {0pt}{1.0ex plus .2ex minus .2ex}{0.6ex}

\usepackage{multicol}
\usepackage{orcidlink}
\hypersetup{
    colorlinks=true,
    linkcolor=blue,   
    citecolor=blue,   
    urlcolor=blue     
}

\title{\LARGE \bf IKSPARK: Obstacle-Aware Inverse Kinematics via Convex Optimization}

\author{Liangting Wu and Roberto Tron
\thanks{The authors are with the Department of Mechanical Engineering, Boston University, 110 Cummington Mall, Boston, MA 02215, USA.
Emails: {\tt\small tomwu@bu.edu, tron@bu.edu}.
The authors gratefully acknowledge the support by NSF award FRR-2212051.
Code available online: \href{https://bitbucket.org/liangtingwu/ikspark}{https://bitbucket.org/liangtingwu/ikspark}.
}
}

\begin{document}

\maketitle
\thispagestyle{empty}
\pagestyle{empty}

\begin{abstract}
Inverse kinematics (IK) is central to robot control and motion planning, yet its nonlinear kinematic mapping makes it inherently nonconvex and particularly challenging under complex constraints. We present IKSPARK (Inverse Kinematics using Semidefinite Programming And RanK minimization), an obstacle-aware IK solver for robots with diverse morphologies, including open and closed kinematic chains with spherical, revolute, and prismatic joints. 
Our formulation expresses IK as a semidefinite programming (SDP) problem with additional rank-1 constraints on symmetric matrices with fixed traces. IKSPARK first solves the relaxed SDP, whose infeasibility certifies infeasibility of the original IK problem, and then recovers a rank-1 solution using iterative rank-minimization methods with proven local convergence. Obstacle avoidance is handled through a convexified formulation of mixed-integer constraints. Extensive experiments show that IKSPARK computes highly accurate solutions across various kinematic structures and constrained environments without post-processing. In obstacle-rich settings, especially fixed workcell environments, IKSPARK achieves substantially higher success rates than traditional nonlinear optimization methods.
\end{abstract}

\section{Introduction}
\begin{figure*}[t]
  \centering
  \subfloat[]{%
    \includegraphics[width=0.4\textwidth]{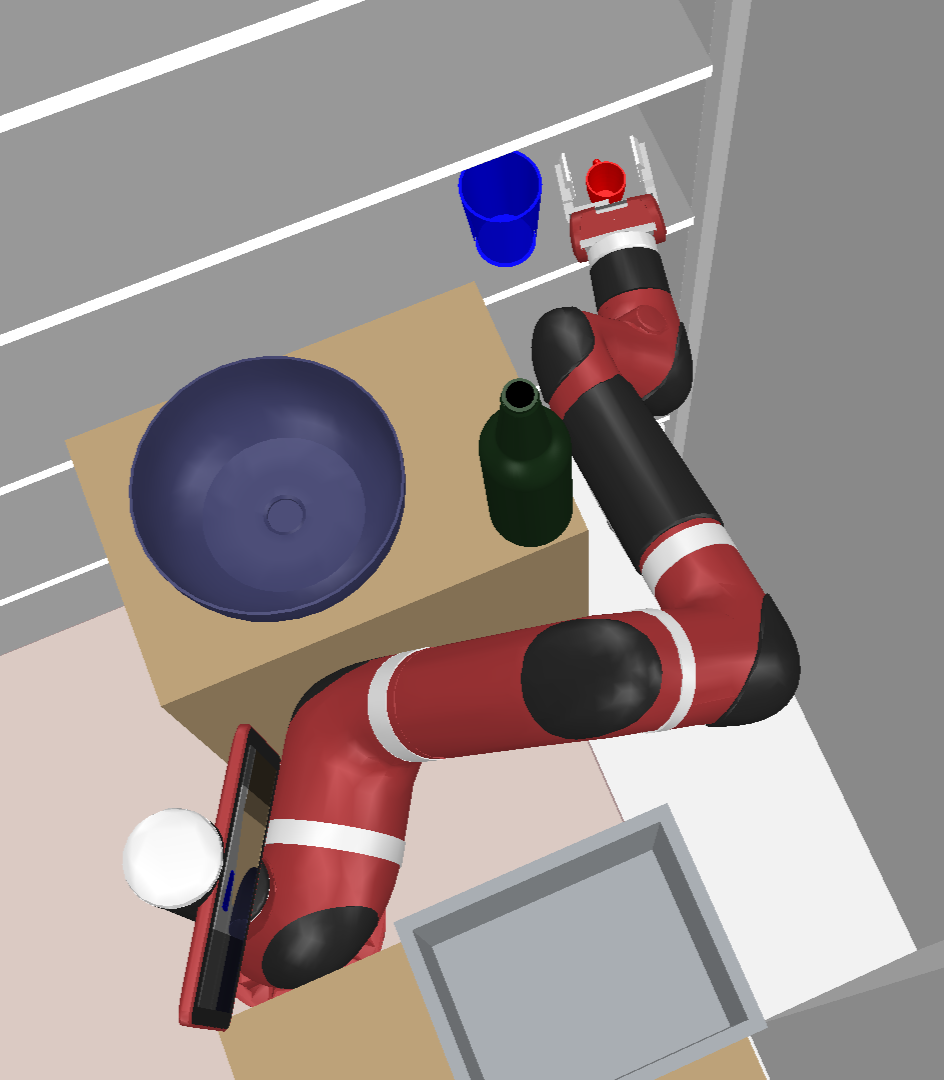}%
    \label{fig:sawyer_grasp_demo}}
  \subfloat[]{%
    \includegraphics[width=0.336\textwidth]{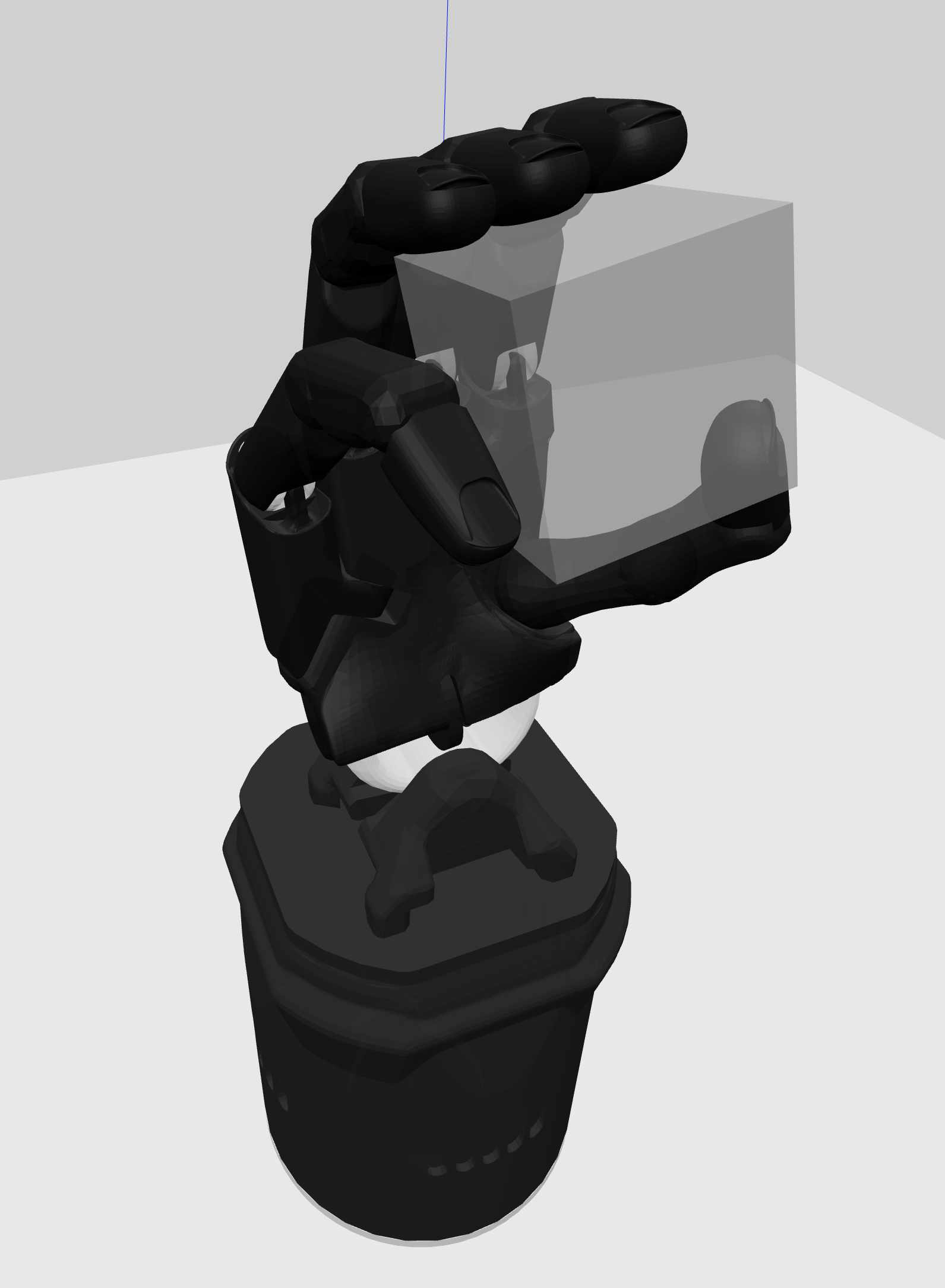}%
    \label{fig:hand_grasp_demo}}
  \caption{Our IK solver can find obstacle-free configurations for Sawyer grasping a mug (left) and Shadow Dexterous Hand holding a cube (right).}
  \label{fig:demos}
\end{figure*}

\noindent Over the past several decades, robot manipulators have gained widespread applications in areas such as manufacturing, medical surgeries, and aerospace. A fundamental problem for robots in these settings is inverse kinematics (IK) \cite{siciliano2009modelling}, where one needs to determine the values of the joint configurations that result in a given desired position and orientation of the end-effector.

Despite its importance, solving the IK problem remains challenging for several reasons:
\begin{enumerate}
\item The kinematic map from joint configurations to end-effector poses is generally nonlinear.
\item Depending on the robot structure and the target pose, the problem may admit no solution, a finite number of distinct solutions, or infinitely many solutions.
\item The problem is often subject to additional nonlinear constraints arising from robot kinematics and task requirements, including joint limits, collision avoidance, and closed kinematic chains.
\item Infeasibility is common in practice, yet certifying that no feasible solution exists is itself challenging.
\end{enumerate}

In this paper, we propose an IK solver named IKSPARK (Inverse Kinematics using Semidefinite Programming And RanK minimization). Instead of using joint angles, we parameterize the robot inverse kinematics problem over the set of rotation matrices $\SO{3}$. We show that, by using this parameterization, we can write the kinematic constraints of the robot as convex constraints of rotation matrices. To overcome the nonlinearity brought by the manifold $\SO{3}$, we introduce a semidefinite relaxation of the kinematic constraints followed by a rank minimization algorithm.

The main contributions of this work are as follows:
\begin{itemize}
\item We parametrize the problem as a function of the rotation of reference frames of each link, allowing us to easily incorporate a variety of constraints and arrangements of links, covering:
\begin{itemize}
    \item spherical joints;
    \item revolute joints with and without angle limits;
    \item prismatic joints;
    \item open/closed kinematic chains.
\end{itemize}
\item We develop a relaxation of the manifold of robot configurations as a combination of linear and semidefinite constraints on constant-trace matrices. Notably, we show that our relaxation is convex and bounded, it contains every kinematically feasible solution, and is tight in the sense that every kinematically feasible solution is on the boundary of the relaxed set. 
Moreover, we can use the relaxation as a sound method to check for kinematic feasibility.
\item The variables in our SDP formulation are carefully chosen; specifically, we use a $4\times 4$ PSD (positive semidefinite) matrix for each revolute joint and an $8\times 8$ PSD matrix for each prismatic joint.
This is in contrast to other SDP-based IK solvers, where either a single large PSD matrix is used for the entire robot or multiple PSD matrices of larger dimensions are used for each joint.
\item We propose a novel rank minimization algorithm to project any solution of the relaxed problem to a solution in the original IK problem. The algorithm is based on the maximization of the largest eigenvalue of matrices with fixed trace over the relaxed set. We provide local convergence guarantees, and we show that, if the algorithm converges to a rank-1 solution, then it will exactly satisfy all the constraints of the original IK problem (including the $\SO{3}$ rotation constraints).
\item We incorporate obstacle-avoidance constraints into the IK problem by formulating them as nonconvex constraints with mixed-integer variables and then convexifying them into linear constraints. Through extensive experiments, we show that our solver can find obstacle-free configurations in a variety of environments with obstacles with a higher success rate than nonlinear optimization methods.
\end{itemize}
With respect to our preliminary work in \cite{wu2023cdc}, we include prismatic joints, use of quaternions to reduce the number of variables, an alternative rank minimization approach for problems with uncertain minimal costs, obstacle avoidance, and a method for motion generation.

\section{Related Work}
Previous work has shown that a finite number of analytical solutions for manipulators with up to 6-DOF exist \cite{lee1988displacement}, and can be derived in algebraic form \cite{raghavan1993inverse,husty2007new}. The popular solver IKFast \cite{diankov2010automated} generalizes this method and automatically computes IK solutions in closed form. However, analytical methods are generally unavailable for robots with higher DOFs. On the other hand, numerical methods have been successful in solving the IK problem, producing numerous efficient inverse-kinematics solvers such as CCD \cite{kenwright2012inverse}, triangulation \cite{muller2007triangulation}, and FABRIK \cite{aristidou2011fabrik}.
These solvers often perturb joint angles iteratively to decrease the distance between the end-effector and the target. Despite their efficiency, kinematic constraints such as collision avoidance, multiple end-effectors, and closed chains are either ignored or require ad-hoc modifications.

Other approaches cast inverse kinematics as a nonlinear optimization problem and solve it numerically, as in \cite{Naour2019kinematics,Beeson2015Trac}.
The robotics toolbox Drake \cite{drake} provides a general framework for modeling and solving obstacle-aware IK problems \cite{tedrake2023manipulation} with powerful nonlinear optimization solvers, including SNOPT \cite{gill2005snopt}, IPOPT \cite{wachter2006ipopt}, and NLOPT \cite{johnson2014nlopt}.
Although these methods are often computationally efficient, they generally cannot guarantee convergence to a global optimum from arbitrary initial guesses.

Instead of solving the nonlinear IK problem directly, several works relax the nonlinear constraints and solve one or multiple approximate convex problems. 
Unlike nonlinear programming, these formulations can typically provide certificates of infeasibility, do not require initial guesses, and can be solved with many convex optimization solvers.
In practice, these methods generally show a slower computational speed than nonlinear programming methods, but they can be more robust to local minima.
For example, \cite{dai2019global} propose \texttt{GlobalInverseKinematics}, a mixed-integer convex optimization approach that can either certify that the problem is infeasible or solve globally for an approximate solution. Closest to our work are \texttt{SDP-IK} \cite{yenamandra2019convex} and \texttt{CIDGIK} \cite{giamou2022convex}, both of which formulate IK as SDPs with additional low-rank constraints. However, the two methods rely on different parameterizations: \texttt{SDP-IK} uses global rigid-body transformations, whereas \texttt{CIDGIK} adopts a distance-geometric formulation. They also treat the rank constraints differently: \texttt{SDP-IK} drops them to obtain a convex relaxation, while \texttt{CIDGIK} introduces a rank-minimization procedure.

Compared with \texttt{SDP-IK} \cite{yenamandra2019convex}, our approach uses a simpler, lower-dimensional formulation. Moreover, unlike \texttt{SDP-IK}, the recovered rotations are already numerically close to valid elements of $\SO{3}$, and therefore do not require projection onto $\SO{3}$ as a post-processing step, thus avoiding projection errors. Compared with \texttt{CIDGIK} \cite{giamou2022convex}, our method uses a parameterization whose dimension grows linearly, rather than quadratically, with the robot DoFs. Specifically, we represent the rotation of each link with a $7\times7$ PSD matrix. We further show that this can be reduced to a $4\times 4$ PSD matrix using a quaternion-based formulation. Furthermore, our method employs a different rank-minimization scheme that maximizes the eigenvalues of positive-semidefinite matrices with constant trace. In addition, we establish local convergence of the proposed method to feasible solutions.
We provide a comparison of our method with several other convex optimization-based IK solvers in Table~\ref{tab:convex_ik_compare}.

Some other work, \cite{li2020robot,yang2019quaternion,yang2020teaser,yang2022certifiable}, investigate semidefinite relaxation of problems beyond IK that involve rotations. The tightness of the semidefinite relaxation techniques for such problems is evaluated in \cite{saunderson2015sdp,bandeira2017tightness,peng2022semidefinite}.

When additional constraints are introduced, IK problems become more challenging. Obstacle avoidance is particularly difficult because of its nonlinearity. Many works integrate obstacle avoidance into Jacobian-based solvers using task-priority methods \cite{di_lillo_safety-related_2018,khatib_task_2020}, where the collision constraint is encoded as a desired task value for the robot to achieve.
Other work enforces obstacle avoidance through nonlinear constraints in nonconvex optimization problems solved using nonlinear programming techniques \cite{weingartshofer2023optimization,yang2024obstacle,marangoz2023dawnik}. 

Closer to our approach,
\cite{dai2019global} and \cite{giamou2022convex} present mixed-integer programming and semidefinite relaxation formulations, respectively, for obstacle avoidance, which are different from the convexification approach used in this paper. 

\begin{table}[htb]
  \centering
  \resizebox{0.48\textwidth}{!}{
  \begin{tabular}{c c c c c}
    \toprule
    & IKSPARK & GlobalIK & SDP-IK & CIDGIK \\
    \midrule
    Optimization type & SDP & MIP & SDP & SDP \\
    Joint limits & \cmark & \cmark & \cmark & \xmark \\
    Obstacle avoidance & \cmark & \cmark & \xmark & \cmark \\
    Self-collisions & \xmark & \xmark & \xmark & \cmark \\
    Prismatic joints & \cmark & \xmark & \xmark & \xmark \\
    Number of variables & $(4\times4)n_l$ & $(30+9m)n_l$ & \makecell{$(9\times 9+$\\$3\times3)3n_j$} & \makecell{$(2n_j+3)$\\$\times(2n_j+3)$} \\
    \bottomrule
  \end{tabular}
  }
  \caption{Comparison of IKSPARK with other inverse kinematics solvers based on convex optimization. Here, $n_j$ denotes the number of revolute joints, $n_l$ the number of links, and $m$ the number of intervals used in the rotation relaxation.}
  \label{tab:convex_ik_compare}
\end{table}

\section{Parameterization}
This section discusses some notations and how to model general kinematic chains using rotations and translations.

\subsection{General notation}
We use $\ident_d$ to denote the identity matrix of dimension $d$, and we use $\ve_i$ to define a standard basis column vector with $1$ in the $i$-th entry and zero elsewhere. We denote the set of $n\times n$ symmetric matrices as $\IS^n$ and the set of $n\times n$ positive semidefinite matrices as $\IS^n_{+}$. For a matrix $\mM$, we denote $\mM(a_1:a_2,b_1:b_2)$ as the block of entries in $\mM$ covering the $a_1$-th to $a_2$-th rows and $b_1$-th to $b_2$-th columns.
We denote $\vec(\mM)$ as the vectorization of matrix $\mM$. The matrix inner product is defined as $\langle\mA,\mB\rangle=\trace(\mA\transpose\mB)$, where $\trace(\cdot)$ is the matrix trace.

We make use of the vectorization property of the Kronecker product $\otimes$: $(\mB^\trans \otimes \mA)\vec(\mX)=\vec(\mA\mX\mB)$, for any $\mA,\mB$ and $\mX$ of appropriate dimensions.

\subsection{Kinematic chains}
We define a world reference frame $\cW$ and associate with each robot link a reference frame $\cB_i$, $i\in\{1,\ldots,n\}$. For each revolute joint, the $z$-axis of the reference frame of its child is aligned with the joint axis.

We represent the robot structure by a graph $G=(\cV,\ave)$, where $\cV$ is the set of link indices and $\ave$ is the set of ordered pairs describing parent-child relations between links. In particular, $(i,j)\in\ave$, with $i,j\in\cV$, if link $i$ is the parent of link $j$. For robots composed only of spherical, revolute, and prismatic joints, $\ave$ can be partitioned into three disjoint subsets, $\ave_s$, $\ave_r$, and $\ave_p$, corresponding to spherical, revolute, and prismatic joints, respectively. We denote by $n_p$ the cardinality of $\ave_p$. 

We define $\cV_t$ and $\cV_r$ as the subsets of $\cV$ whose translations and rotations, respectively, are to be determined. We denote by $n_r$ the cardinality of $\cV_r$. We use the subscripts $base$ and $ee$ to denote the base frame (i.e., a frame rigidly attached to $\cW$) and the end-effector, respectively. We assume there exists a path from the base, $base=p_1$, to the end-effector, $ee=p_n$, given by $\mathcal{P}_{fk}=\{p_1,p_2,\ldots,p_n\}\subseteq\cV$, where $(p_i,p_{i+1})\in\ave$ for all consecutive nodes $p_i,p_{i+1}\in\mathcal{P}_{fk}$. We denote the set of edges along this forward-kinematics path by $\ave_{fk}=\{(p_i,p_{i+1})\mid (p_i,p_{i+1})\in\ave\}$.

\subsection{Modeling kinematic chains using rotations and translations}\label{sec:parameter}
The pose $\{\mR_i\in\SO{3},\mT_i\in\IR^{3}\}$ represents the rigid body transformation (rotation and translation) from the reference frame $\mathcal{B}_i$ to the world frame, i.e., $\mR_i = \Wframe\mR_{\mathcal{B}_i}$ and $\mT_i=\Wframe\mT_{\mathcal{B}_i}$. 
To simplify the notation, we denote the relative rotation $\Bframe[i]\mR_{\mathcal{B}_j}$ and translation $\Bframe[i]\mT_{\cB_j}$ from $\mathcal{B}_j$ to $\mathcal{B}_i$ as ${}^i\mR_j$, and ${}^{i}\mT_j$, respectively.

In this paper, we assign a reference frame $\cB_i$ to each link and parameterize our problem on subsets of $\{\mR_i,\mT_i\}$.
This parameterization allows us to formulate IK as convex optimization problems, as we discuss in Section \ref{sec:opt}.


\begin{figure}[hbt]
  \centering
  \tikzset{every picture/.style={line width=0.75pt}} 

\begin{tikzpicture}[x=0.75pt,y=0.75pt,yscale=-1,xscale=1]

\draw   (349.52,84.65) .. controls (343.82,84.44) and (340.86,83.46) .. (342.9,82.48) .. controls (344.94,81.49) and (351.22,80.86) .. (356.92,81.08) .. controls (362.62,81.29) and (365.59,82.27) .. (363.55,83.25) .. controls (361.5,84.24) and (355.23,84.87) .. (349.52,84.65)(340.54,89) .. controls (323.1,88.34) and (314.64,85.06) .. (321.64,81.67) .. controls (328.64,78.29) and (348.46,76.07) .. (365.9,76.73) .. controls (383.35,77.39) and (391.81,80.67) .. (384.81,84.06) .. controls (377.8,87.44) and (357.98,89.66) .. (340.54,89) ;
\draw    (258.5,118.75) -- (264.5,134.75) ;
\draw    (319.2,104.06) .. controls (331.31,112.46) and (375.04,107.03) .. (393.07,97.39) ;
\draw    (203.5,119.25) -- (212,138.25) ;
\draw   (224.47,110.25) .. controls (220.52,111.72) and (218.22,111.89) .. (219.34,110.64) .. controls (220.45,109.39) and (224.55,107.19) .. (228.49,105.73) .. controls (232.44,104.27) and (234.74,104.09) .. (233.62,105.34) .. controls (232.51,106.59) and (228.41,108.79) .. (224.47,110.25)(219.57,115.75) .. controls (207.5,120.22) and (200.81,120.38) .. (204.63,116.09) .. controls (208.44,111.81) and (221.32,104.71) .. (233.39,100.24) .. controls (245.46,95.76) and (252.15,95.61) .. (248.33,99.89) .. controls (244.52,104.18) and (231.64,111.28) .. (219.57,115.75) ;
\draw    (318.96,84.55) -- (319.2,104.06) ;
\draw    (212,138.25) .. controls (222.62,140.56) and (245.13,129.79) .. (256,119.25) ;
\draw    (249.5,96.75) .. controls (273.96,92.43) and (292.27,101.59) .. (318.96,101.84) ;
\draw    (244.5,103.25) .. controls (268.96,98.93) and (303.8,106.66) .. (329.17,107.65) ;
\draw    (387.39,100.36) -- (387.39,119.01) ;
\draw    (326.2,126.75) .. controls (346.3,130.38) and (378.58,130.05) .. (387.39,119.01) ;
\draw    (256,119.25) .. controls (280,113.25) and (294.25,123.46) .. (326.2,126.75) ;
\draw    (387.39,81.71) .. controls (416.79,80.59) and (413.21,84.66) .. (436.27,89.6) ;
\draw    (393.07,97.39) .. controls (403.5,94.25) and (412.5,96.75) .. (431,104.75) ;
\draw    (195,164.25) .. controls (207.5,161.25) and (212.5,158.25) .. (221.5,156.75) ;
\draw    (382.53,78.66) .. controls (401.5,77.75) and (408.5,77.25) .. (427.58,81.64) ;
\draw [color={rgb, 255:red, 0; green, 24; blue, 255 }  ,draw opacity=1 ][line width=1.5]    (354.75,97.39) -- (354.55,69.02) ;
\draw [shift={(354.53,66.02)}, rotate = 89.6] [color={rgb, 255:red, 0; green, 24; blue, 255 }  ,draw opacity=1 ][line width=1.5]    (14.21,-4.28) .. controls (9.04,-1.82) and (4.3,-0.39) .. (0,0) .. controls (4.3,0.39) and (9.04,1.82) .. (14.21,4.28)   ;
\draw [color={rgb, 255:red, 0; green, 24; blue, 255 }  ,draw opacity=1 ][line width=1.5]    (354.75,97.39) -- (390.07,97.39) ;
\draw [shift={(393.07,97.39)}, rotate = 180] [color={rgb, 255:red, 0; green, 24; blue, 255 }  ,draw opacity=1 ][line width=1.5]    (14.21,-4.28) .. controls (9.04,-1.82) and (4.3,-0.39) .. (0,0) .. controls (4.3,0.39) and (9.04,1.82) .. (14.21,4.28)   ;
\draw [color={rgb, 255:red, 0; green, 24; blue, 255 }  ,draw opacity=1 ][line width=1.5]    (354.75,97.39) -- (377.96,120.61) ;
\draw [shift={(380.08,122.73)}, rotate = 225] [color={rgb, 255:red, 0; green, 24; blue, 255 }  ,draw opacity=1 ][line width=1.5]    (14.21,-4.28) .. controls (9.04,-1.82) and (4.3,-0.39) .. (0,0) .. controls (4.3,0.39) and (9.04,1.82) .. (14.21,4.28)   ;
\draw [color={rgb, 255:red, 17; green, 15; blue, 244 }  ,draw opacity=1 ][line width=1.5]  [dash pattern={on 5.63pt off 4.5pt}]  (354.74,45.43) -- (354.74,146.35) -- (354.74,170.6) ;
\draw    (431,104.75) .. controls (427.87,96.54) and (429.87,96.04) .. (436.27,89.6) ;
\draw    (427.58,81.64) .. controls (426.37,83.04) and (428.87,90.04) .. (436.27,89.6) ;
\draw    (188.5,150.75) .. controls (200.5,147.25) and (209.51,142.27) .. (220.29,138.54) ;
\draw    (221.5,156.75) .. controls (238.3,156.25) and (259.5,145.25) .. (264.5,134.75) ;
\draw    (195,164.25) .. controls (197,157.75) and (188,157.75) .. (188.5,150.75) ;
\draw    (187,144.75) .. controls (196.5,140.25) and (203.5,138.25) .. (210.79,133.54) ;
\draw    (187,144.75) .. controls (187.5,148.25) and (170.5,151.75) .. (188.5,150.75) ;
\draw    (182,158.25) .. controls (193,166.75) and (182,172.25) .. (195,164.25) ;
\draw    (180,151) .. controls (178,155.25) and (179.5,157.25) .. (182,158.25) ;
\draw [color={rgb, 255:red, 208; green, 2; blue, 27 }  ,draw opacity=1 ][line width=1.5]    (233.3,123.46) -- (220.87,95.27) ;
\draw [shift={(219.66,92.52)}, rotate = 66.19] [color={rgb, 255:red, 208; green, 2; blue, 27 }  ,draw opacity=1 ][line width=1.5]    (14.21,-4.28) .. controls (9.04,-1.82) and (4.3,-0.39) .. (0,0) .. controls (4.3,0.39) and (9.04,1.82) .. (14.21,4.28)   ;
\draw [color={rgb, 255:red, 208; green, 2; blue, 27 }  ,draw opacity=1 ][line width=1.5]    (233.3,123.46) -- (259.78,111.91) ;
\draw [shift={(262.53,110.72)}, rotate = 156.44] [color={rgb, 255:red, 208; green, 2; blue, 27 }  ,draw opacity=1 ][line width=1.5]    (14.21,-4.28) .. controls (9.04,-1.82) and (4.3,-0.39) .. (0,0) .. controls (4.3,0.39) and (9.04,1.82) .. (14.21,4.28)   ;
\draw [color={rgb, 255:red, 208; green, 2; blue, 27 }  ,draw opacity=1 ][line width=1.5]    (233.3,123.46) -- (262.52,137.06) ;
\draw [shift={(265.24,138.32)}, rotate = 204.96] [color={rgb, 255:red, 208; green, 2; blue, 27 }  ,draw opacity=1 ][line width=1.5]    (14.21,-4.28) .. controls (9.04,-1.82) and (4.3,-0.39) .. (0,0) .. controls (4.3,0.39) and (9.04,1.82) .. (14.21,4.28)   ;
\draw [color={rgb, 255:red, 208; green, 2; blue, 27 }  ,draw opacity=1 ][line width=1.5]  [dash pattern={on 5.63pt off 4.5pt}]  (210.35,71.92) -- (256.26,175) ;

\draw (331.36,87.46) node [anchor=north west][inner sep=0.75pt]  [color={rgb, 255:red, 0; green, 1; blue, 247 }  ,opacity=1 ] [align=left] {$\displaystyle \mathcal{B}_{j}$};
\draw (212.32,119) node [anchor=north west][inner sep=0.75pt]   [align=left] {$\displaystyle \mathcal{\textcolor[rgb]{0.82,0.01,0.11}{B}}\textcolor[rgb]{0.82,0.01,0.11}{_{i}}$};
\draw (267.33,131) node [anchor=north west][inner sep=0.75pt]   [align=left] {$\displaystyle \textcolor[rgb]{0.82,0.01,0.11}{x_{i}}$};
\draw (266.5,100.33) node [anchor=north west][inner sep=0.75pt]   [align=left] {$\displaystyle \textcolor[rgb]{0.82,0.01,0.11}{y_{i}}$};
\draw (225.5,78.67) node [anchor=north west][inner sep=0.75pt]   [align=left] {$\displaystyle \textcolor[rgb]{0.82,0.01,0.11}{z_{i}}$};
\draw (382.83,119.33) node [anchor=north west][inner sep=0.75pt]   [align=left] {$\displaystyle \textcolor[rgb]{0.03,0.08,0.96}{x_{j}}$};
\draw (394.83,93.33) node [anchor=north west][inner sep=0.75pt]   [align=left] {$\displaystyle \textcolor[rgb]{0.03,0.08,0.96}{y_{j}}$};
\draw (333.17,54.33) node [anchor=north west][inner sep=0.75pt]   [align=left] {$\displaystyle \textcolor[rgb]{0.03,0.08,0.96}{z_{j}}$};

\end{tikzpicture}
  \caption{Two connected links $(i,j)\in\ave_r$, each associated with a reference frame ($\cB_i$ and $\cB_j$).}
  \label{fig:frames}
\end{figure}

A general set of variables to be solved for the IK problem is
\aleq{
  \vx = \{\mT_{i}\in \IR^{3},\mR_{j}\in\SO{3}\mid i\in\cV_t,j\in\cV_r\}.\label{eq:xdefine}
}

\section{Kinematic constraints}
In this section, we model translation constraints between connected links as \emph{linear} equality constraints, derive \emph{linear} constraints for revolute-joint axes and angle limits, and then address the kinematic and translation constraints for prismatic joints.

\subsection{Revolute joints}
For two links connected by a revolute joint, the robot geometry determines their relative transformation. In particular, the links share a common rotation axis, and their relative rotation depends on the joint angle, which is constrained to lie within a given interval. We show that these conditions can be expressed as linear constraints on the rotations.

\subsubsection{Joint translations}\label{sec:jointtranslation}

For two links connected by a joint, $(i,j)\in\ave$, the following relation holds:
\aleq{
\mT_j-\mT_i = \mR_i{}^i\mT_{j}.\label{eq:translation relation}
}
Here, the translation ${}^i\mT_j$ is determined by the robot geometry, whereas the remaining quantities are variables to be solved for. Using this relation, we can express the end-effector translation as a function of the rotations along the edge set $\ave_{fk}$:
\begin{equation}\label{eq:m_te}
  \begin{aligned}
    \mT_{ee} &= \mT_{base}+\sum_{(i,j)\in\ave_{fk}} (\mT_j-\mT_i)\\
             &= \mT_{base}+\sum_{(i,j)\in\ave_{fk}} \mR_i{}^i\mT_j\\
             &=\mT_{base}+\sum_{(i,j)\in\ave_{fk}} ({}^i\mT_j^\trans\otimes\ident_3)\vec(\mR_i).
  \end{aligned}
\end{equation}
If the chain $\ave_{fk}$ contains only revolute or spherical joints, then \eqref{eq:m_te} is linear in the rotations. If prismatic joints are present, ${}^i\mT_j$ is no longer constant; in that case, we will show that $\mT_{ee}$ can still be written as a linear function of the rotations and an additional variable.

These linear expressions allow us to eliminate the translations $(\mT_i,\mT_j)$ algebraically from the problem formulation and solve only for the rotations; this is discussed further in Section \ref{sec:fk_linear}. If the links form a chain and at least one translation $\mT_i$ is known, for example, because the robot base is fixed, then all remaining translations can be recovered from \eqref{eq:translation relation} once the rotations and any additional variables associated with prismatic joints have been determined.

\subsubsection{Revolute joint axis constraints}
For each pair of links $(i,j)\in\ave_r$ that are connected with a revolute joint, the orientations $\mR_i$ and $\mR_j$ are related by the equation
\aleq{
  \mR_{j}=\mR_i\mR_e\mR_{\theta}\label{eq:jointrelation1}
}
where $\mR_{\theta}:\IR\mapsto\SO{3}$ is a function of the joint angle $\theta$ (defined such that $\mR_{\theta}=\ident$ when $\theta=0$), and $\mR_e$ is a parameter defined as the rotation from $\cB_j$ to $\cB_i$ when $\theta=0$. Without loss of generality, we assume that $\mR_{\theta}$ is a rotation about the $z$--axis, meaning that frames $\cB_i$ (after applying the rotation $\mR_e$) and $\cB_j$ share the same $z$--axis, that is:
\aleq{\label{eq:axis constraint}
  \mR_i\mR_e\ve_3-\mR_{j}\ve_3=\omat.
}
Using the vectorization property of the Kronecker product, we have
  \begin{equation}
    (\ve_3^\trans\mR_e^\trans\otimes\ident)\vec(\mR_i)-(\ve_3^\trans\otimes\ident)\vec(\mR_j)=\omat.\label{eq:jointaxis1}
  \end{equation}

\subsubsection{Revolute joint angle limits}\label{sec:jointlimit}
In physical systems, the joint angle $\theta$ in \eqref{eq:jointrelation1} is limited to an interval $[ -\phi_1,\phi_2 ]$. Without loss of generality, we can assume this interval to be symmetric, i.e., $\phi_1=\phi_2=\alpha$ (if not, we can shift the zero angle of $\theta$ so that it is in the middle of the interval). With this assumption, the joint angle constraint becomes 
\begin{equation}
    \abs{\theta} \leq \alpha. \label{eq:angle limit constraint}
\end{equation}

We introduce a formulation of the angle limit constraints that is linear in the rotations 
and exactly captures \eqref{eq:angle limit constraint}, as shown in the following proposition.

\begin{proposition}\label{prop:jointangle}
  For a robot with links $\cV$ connected by relation $\ave$ and variables defined by \eqref{eq:xdefine}, the revolute joint angle limit constraint is satisfied if $\{\mR_i\}$ are rotations, and for every $(\mR_i,\mR_j), (i,j)\in\cE_r$, it holds that 
  \aleq{
    \mR_i\mR_e\ve_1 -\mR_j\ve_1\in \cS\bigl(\sqrt{2-2\cos(\alpha_{ij})}\bigr),\label{eq:jointlimitballbound}
  }
  where $\cS(r)$ is a ball with radius $r$ and centered at the origin.
\end{proposition}
\begin{proof}
  Since the rotations $\mR_j$ and $\mR_i\mR_e$ in \eqref{eq:axis constraint} share the same $z$-axis and their $x$-- and $y$--axes are on the same plane. The angle $\theta$ can be seen as the angular displacement from the $x$-axis of $\mR_i\mR_e$ to that of $\mR_j$. Therefore, for two vectors $\vw_i = \mR_i\mR_e\ve_1$ and $\vw_j = \mR_j\ve_1$, we have $\theta = \angle(\vw_i,\vw_j)$ and \eqref{eq:angle limit constraint} becomes $\abs{\angle(\vw_i,\vw_j)}\leq\alpha_{ij}$. Substituting \eqref{eq:jointrelation1}, and the expressions for $\vw_i,\vw_j$ into $\norm{\vw_i-\vw_j}^2$ we have $\norm{\vw_i-\vw_j}^2=2-2\cos(\theta)\leq 2-2\cos(\alpha_{ij})$ for $\theta\in[-\alpha_{ij},\alpha_{ij}]$, which gives the bound in \eqref{eq:jointlimitballbound} (see Fig. \ref{fig:jointlimit} for a visualization). 
\end{proof}
It is worth mentioning that \eqref{eq:jointlimitballbound} is algebraically equivalent to \cite[Eq. (13)]{dai2019global}, but in different form. 

\begin{figure}[htb]
  \centering
  \tikzset{every picture/.style={line width=0.75pt}} 

\begin{tikzpicture}[x=0.75pt,y=0.75pt,yscale=-1,xscale=1]

\draw  [draw opacity=0][fill={rgb, 255:red, 144; green, 19; blue, 254 }  ,fill opacity=0.37 ] (349.46,71.28) .. controls (359.56,86.21) and (365.49,104.37) .. (365.49,123.95) .. controls (365.49,142.96) and (359.91,160.63) .. (350.34,175.31) -- (275.95,123.95) -- cycle ;
\draw  [draw opacity=0][fill={rgb, 255:red, 248; green, 231; blue, 28 }  ,fill opacity=0.33 ] (131.31,123.92) .. controls (131.31,94.44) and (155.21,70.54) .. (184.69,70.54) .. controls (214.17,70.54) and (238.07,94.44) .. (238.07,123.92) .. controls (238.07,153.4) and (214.17,177.3) .. (184.69,177.3) .. controls (155.21,177.3) and (131.31,153.4) .. (131.31,123.92) -- cycle ;
\draw   (132.23,123.92) .. controls (132.23,94.44) and (156.13,70.54) .. (185.61,70.54) .. controls (215.09,70.54) and (238.99,94.44) .. (238.99,123.92) .. controls (238.99,153.4) and (215.09,177.3) .. (185.61,177.3) .. controls (156.13,177.3) and (132.23,153.4) .. (132.23,123.92) -- cycle ;
\draw  [color={rgb, 255:red, 249; green, 156; blue, 2 }  ,draw opacity=1 ] (242.09,123.92) -- (231.31,157.12) -- (203.06,177.64) -- (168.15,177.64) -- (139.91,157.12) -- (129.12,123.92) -- (139.91,90.71) -- (168.15,70.19) -- (203.06,70.19) -- (231.31,90.71) -- cycle ;
\draw [color={rgb, 255:red, 74; green, 144; blue, 226 }  ,draw opacity=1 ]   (184.69,123.92) -- (380.9,123.92) ;
\draw [color={rgb, 255:red, 65; green, 117; blue, 5 }  ,draw opacity=1 ]   (275.55,124.59) -- (354.65,161.09) ;
\draw [shift={(357.37,162.35)}, rotate = 204.78] [fill={rgb, 255:red, 65; green, 117; blue, 5 }  ,fill opacity=1 ][line width=0.08]  [draw opacity=0] (8.93,-4.29) -- (0,0) -- (8.93,4.29) -- cycle    ;
\draw [color={rgb, 255:red, 208; green, 2; blue, 27 }  ,draw opacity=1 ]   (184.69,123.92) -- (272.55,123.92) ;
\draw [shift={(275.55,123.92)}, rotate = 180] [fill={rgb, 255:red, 208; green, 2; blue, 27 }  ,fill opacity=1 ][line width=0.08]  [draw opacity=0] (8.93,-4.29) -- (0,0) -- (8.93,4.29) -- cycle    ;
\draw [color={rgb, 255:red, 65; green, 117; blue, 5 }  ,draw opacity=1 ] [dash pattern={on 0.84pt off 2.51pt}]  (275.55,123.92) -- (348.84,71.43) ;
\draw  [draw opacity=0] (316.12,94.38) .. controls (322.03,102.48) and (325.57,112.41) .. (325.73,123.15) -- (275.55,123.92) -- cycle ; \draw  [color={rgb, 255:red, 65; green, 117; blue, 5 }  ,draw opacity=1 ] (316.12,94.38) .. controls (322.03,102.48) and (325.57,112.41) .. (325.73,123.15) ;  
\draw [color={rgb, 255:red, 65; green, 117; blue, 5 }  ,draw opacity=1 ] [dash pattern={on 0.84pt off 2.51pt}]  (275.55,123.92) -- (348.76,174.31) ;
\draw   (185.61,123.92) .. controls (185.61,74.24) and (225.87,33.98) .. (275.55,33.98) .. controls (325.22,33.98) and (365.49,74.24) .. (365.49,123.92) .. controls (365.49,173.59) and (325.22,213.86) .. (275.55,213.86) .. controls (225.87,213.86) and (185.61,173.59) .. (185.61,123.92) -- cycle ;
\draw [color={rgb, 255:red, 65; green, 117; blue, 5 }  ,draw opacity=1 ] [dash pattern={on 0.84pt off 2.51pt}]  (200.79,175.06) -- (275.55,124.59) ;
\draw [color={rgb, 255:red, 65; green, 117; blue, 5 }  ,draw opacity=1 ] [dash pattern={on 0.84pt off 2.51pt}]  (201.46,72.69) -- (275.55,123.92) ;
\draw [color={rgb, 255:red, 65; green, 117; blue, 5 }  ,draw opacity=1 ]   (275.55,124.59) -- (197.33,86.31) ;
\draw [shift={(194.64,84.99)}, rotate = 26.08] [fill={rgb, 255:red, 65; green, 117; blue, 5 }  ,fill opacity=1 ][line width=0.08]  [draw opacity=0] (8.93,-4.29) -- (0,0) -- (8.93,4.29) -- cycle    ;
\draw [color={rgb, 255:red, 144; green, 19; blue, 254 }  ,draw opacity=1 ]   (185.61,123.92) -- (193.96,87.91) ;
\draw [shift={(194.64,84.99)}, rotate = 103.06] [fill={rgb, 255:red, 144; green, 19; blue, 254 }  ,fill opacity=1 ][line width=0.08]  [draw opacity=0] (8.93,-4.29) -- (0,0) -- (8.93,4.29) -- cycle    ;

\draw (338.26,94.07) node [anchor=north west][inner sep=0.75pt]    {$\textcolor[rgb]{0.25,0.46,0.02}{\alpha }\textcolor[rgb]{0.25,0.46,0.02}{_{ij}}$};
\draw (162.3,72.3) node [anchor=north west][inner sep=0.75pt]  [color={rgb, 255:red, 65; green, 117; blue, 5 }  ,opacity=1 ]  {$-\mathbf{w}_{j}$};
\draw (340,138.3) node [anchor=north west][inner sep=0.75pt]  [color={rgb, 255:red, 65; green, 117; blue, 5 }  ,opacity=1 ]  {$\mathbf{w}_{j}$};
\draw (217.9,123.92) node [anchor=north west][inner sep=0.75pt]    {$\textcolor[rgb]{0.82,0.01,0.11}{\mathbf{w}_{i}}$};
\draw (141.1,92.6) node [anchor=north west][inner sep=0.75pt]    {$\textcolor[rgb]{0.56,0.07,1}{\mathbf{w}}\textcolor[rgb]{0.56,0.07,1}{_{i}}\textcolor[rgb]{0.56,0.07,1}{-}\textcolor[rgb]{0.56,0.07,1}{\mathbf{w}}\textcolor[rgb]{0.56,0.07,1}{_{j}}$};

\end{tikzpicture}
  \caption{The joint limit between two links $(i,j)\in\ave_r$ can be written as an angle limit between two unit vectors $\vw_i = \mR_i\mR_e\ve_1$ and $\vw_j = \mR_j\ve_1$ (purple sector), which can be further bounded by a ball (painted yellow) on $\vw_i-\vw_j$. The ball can then be approximated by linear inequalities (orange polygon).}
  \label{fig:jointlimit}
\end{figure}

\subsection{Prismatic Joints}\label{sec:prismatic}
For two links $(i,j)$ connected by a prismatic joint $(i,j)\in\ave_p$, the physical limit of the joint restricts that \begin{enumerate*}
    \item the orientation is preserved throughout the sliding of the links, and \item the relative distance of the two links is bounded.
\end{enumerate*}
A visualization of these relations is shown in Fig. \ref{fig:prismatic}, where the two links are sliding along a common axis with a bounded displacement.
The orientation is preserved when $\mR_i = \mR_j$.
Meanwhile \eqref{eq:translation relation} holds and can be written as 
\begin{equation}\label{eq:prismatic1}
\begin{aligned}
    \mT_j &= \mT_i+(\tau_l+\tau_i(\tau_u-\tau_l))\mR_i\mR_p\ve_3 
\end{aligned}
\end{equation}
where $\mR_p$ is a known orientation such that the $z$-axes of $\cB_i$ and $\cB_j$ are in the same direction, and $\tau_i\in [0,1]$ is a \emph{bounded} variable that represents the extension of the joint, and $\tau_l,\tau_u$ are the lower- and upper-limits of the joint extension. To simplify, we can assume without loss of generality that $\mR_p=\mI$ (if not, we can perform a change of variable $\tilde{\mR}_i=\mR_i\mR_p$ and the actual $\mR_i$ can be recovered by multiplying $\mR_p^\trans$). In this way, \eqref{eq:prismatic1} becomes $\mT_j = \mT_i+(\tau_l+\tau_i(\tau_u-\tau_l))\mR_i\ve_3$. We then give the following definition of prismatic joints.

\begin{figure}[htb]
  \centering
  \input{prismatic_concept}
  \caption{Two links connected by a prismatic joint, where the two associated reference frames follow a bounded displacement along the common $z$-axis.}
  \label{fig:prismatic}
\end{figure}

\begin{definition}\label{def:prismatic}
    For a robot with links $\cV$ connected by relation $\ave$ and variables defined by \eqref{eq:xdefine}, a prismatic joint $(i,j)\in\ave_p$ is defined by the constraints:
    \begin{equation}\label{eq:prismatic_trans}
    \begin{aligned}
        &\mR_i=\mR_j, \\
        &\mR_i,\mR_j\in \SO{3}\\
        &\mT_j = \mT_i+(\tau_l+\tau_i(\tau_u-\tau_l))\mR_i^{(3)},\\
        &\tau_i\in [0,1],
    \end{aligned}
    \end{equation}
where $\mR_i^{(3)}$ is the third column of $\mR_i$. 
\end{definition}

With \eqref{eq:prismatic_trans}, we can rewrite \eqref{eq:m_te} as 
\begin{equation}\label{eq:ee trans forward kinematics}
\begin{aligned}
    \mT_{ee} =&\mT_{base}+\sum_{(i,j)\in \ave_{fk}\cap(\ave_{r}\cup \ave_{s})} ({}^i\mT_j^\trans\otimes\ident_3)\vec(\mR_i)\\
    &+\sum_{(i,j)\in\ave_{fk}\cap\ave_{p}}(\tau_l+\tau_i(\tau_u-\tau_l))\mR_i^{(3)}.
\end{aligned}
\end{equation}
Equation \eqref{eq:ee trans forward kinematics} means that the end-effector position is a function linear in rotations and bilinear in $\{\tau_i\}$ and $\{\mR_i\}$. 
We note that the spherical joints are included in \eqref{eq:ee trans forward kinematics}, and their translational contribution takes the same form as that of the revolute joints. This is because a spherical joint can be seen as a revolute joint without a fixed joint axis. 
\begin{remark}\label{rem:closedloop}
  Observe that \eqref{eq:ee trans forward kinematics} enables us to impose additional structural constraints on the robot, such as closed kinematic chains. For example, consider a situation where two manipulators are working collaboratively with their end-effectors rigidly attached. For each of them, the end-effector position is a function of its rotations $\{\mR_i\}$ and $\{\tau_i\}$. To fulfill the cooperation requirements, we can simply let these two functions be equal to each other, resulting in a constraint on the closed chain.
\end{remark}

\section{Modeling and Relaxation of the Feasible Set}\label{sec:relaxation}
In this section, we introduce how to model and relax the feasible set defined by the group of kinematic constraints. Our goal is to develop linear or semidefinite constraints.

\subsection{Relaxation of the feasible set for revolute joints}
Observe that the condition for the joint axis in \eqref{eq:jointaxis1} is linear in the vectorized rotations.
We therefore define
\aleq{
  \vu=\stack(\{\vec(\mR_{i})\}_{i\in\cV_r}).\label{eq:udefine}
}
and concatenate \eqref{eq:jointaxis1}  for each $(i,j)\in\ave_r$, obtaining the constraint
\begin{equation}
  \mA_{\textrm{axis}}\vu=\vb_{\textrm{axis}}.
\end{equation}

Next, from Proposition \ref{prop:jointangle}, the joint angle limit constraint requires that for every pair $(i,j)\in\ave_r$, $(\mR_i,\mR_j)$ satisfies the ball bound \eqref{eq:jointlimitballbound}, which can be approximated using linear inequalities, namely, a polyhedron. Specifically, for each $(i,j)\in\ave_r$, we choose multiple points on the ball that bounds $\vw_i-\vw_j$ in \eqref{eq:jointlimitballbound}, and the polyhedron is defined as the polyhedron formed by all faces tangent to the ball at the selected points. Because $\vw_i$ and $\vw_j$ are linear in $\vu$, the linear inequalities for all $(i,j)\in\ave_r$ can then be concatenated as
\begin{equation}
  \mA_{\textrm{angle}}\vu\leq\vb_{\textrm{angle}}.
\end{equation}

\subsection{Relaxation of $\SO{3}$}\label{sec:so3relax}
The definition of $\vu$ requires that each $\mR_{i}\in\SO{3}$, i.e.,
\aleq{
  \mR_{i}^\trans\mR_{i}=\ident_3 \text{ and }
  \det(\mR_{i}) = +1.\label{eq:rotationvalid}
}
These constraints are nonlinear in $\vu$. We propose a novel way to relax $\SO{3}$ using convex constraints.
For $\mR_{i}=\Mat{\mR_{i}^{(1)} &\mR_{i}^{(2)} &\mR_{i}^{(3)}}$, equation \eqref{eq:rotationvalid} is equivalent to
\begin{equation}\label{eq:rotationvalidseparate}
  \begin{aligned}
    \norm{\mR_{i}^{(1)}}&=1\\
    \norm{\mR_{i}^{(2)}}&=1\\
    \mR_{i}^{(1)}\cdot \mR_{i}^{(2)}&=0\\
    \mR_{i}^{(1)}\times \mR_{i}^{(2)}&=\mR_{i}^{(3)}
  \end{aligned}
\end{equation}
For every rotation $\mR_i,i\in\cV_r$, we define a new embedding variable 
\begin{equation}\label{eq:Ystructure}
  \begin{aligned}
    \mY_{i} = \Mat{\mR_{i}^{(1)}\\\mR_{i}^{(2)}\\1}\Mat{\mR_{i}^{(1)}\\\mR_{i}^{(2)}\\1}^\trans\in\IR^{7\times7}.
  \end{aligned}
\end{equation}

Observe that $\mY_i$ is a symmetric rank-1 matrix. Its upper-left $6\times 6$ block contains all pairwise products among the entries of $\mR_i^{(1)}$ and $\mR_i^{(2)}$, namely those associated with $(\mR_i^{(1)},\mR_i^{(1)})$, $(\mR_i^{(1)},\mR_i^{(2)})$, and $(\mR_i^{(2)},\mR_i^{(2)})$. The last column of $\mY_i$ contains $\mR_i^{(1)}$, $\mR_i^{(2)}$, and $1$.
For brevity, we use the shorthand 
\begin{equation}
  \mY:=\{\mY_i\}_{i\in\cV_r}\in\IR^{7\times 7\times n_r}
\end{equation}
in the remainder of the paper.

\begin{definition}\label{def:g definition}
We define the linear transformation $g(\mY): \mathbb{R}^{7\times 7\times n_r} \to \mathbb{R}^{9n_r}$ as follows:
\begin{enumerate}
    \item For each $\mY_i$, extract the first two $3\times 1$ subvectors from its last column, denoted by $\vy_{1i}$ and $\vy_{2i}$.
    \item Construct $\vy_{3i}$, corresponding to $\vy_{1i}\times \vy_{2i}$, using the appropriate entries from the upper-left $6\times 6$ block of $\mY_i$. Explicitly,
    \begin{equation}\label{eq:cross product linear fun}
        \vy_{3i} =
        \Mat{
        \mY_i(2,6)-\mY_i(3,5)\\
        \mY_i(3,4)-\mY_i(1,6)\\
        \mY_i(1,5)-\mY_i(2,4)}.
    \end{equation}
    \item For all $i\in\cV_r$, concatenate $\vy_{1i}$, $\vy_{2i}$, and $\vy_{3i}$ vertically in sequence.
\end{enumerate}
\end{definition}

The structure in \eqref{eq:Ystructure} offers several advantages:
\begin{enumerate}
    \item The first three equations in \eqref{eq:rotationvalidseparate} are linear in $\mY_i$ and can therefore be combined into a single linear equality constraint.
    \item The cross product $\mR_i^{(1)} \times \mR_i^{(2)}$ is linear in $\mY_i$, so $\mR_i^{(3)}$ is linear in $\mY$ via \eqref{eq:cross product linear fun}.
    \item Each column of $\mR_i$ is linear in $\mY_i$, and the mapping $g$ extracts these columns linearly. Hence, with $\mY=\{\mY_i\}_{i\in\cV_r}$, we have $\vu = g(\mY)$, which implies that any constraint linear in $\vu$ is also linear in $\mY$.
\end{enumerate}

\begin{definition}
    In order to enforce the relations in \eqref{eq:rotationvalidseparate} and the structure defined in \eqref{eq:Ystructure}, we define the constraint
    \begin{equation}\label{eq:A structure}
    \mA_{\textrm{structure}}\vec(\mY)=\vb_{\textrm{structure}}
    \end{equation}
    that imposes the following structure on $\mY$: for each $\mY_{i}$,
  \begin{enumerate}
  \item $\trace(\mY_i(1:3,1:3))=\trace(\mY_i(4:6,4:6))=1$;
  \item $\trace(\mY_i(1:3,4:6))=0$;
  \item $\mY_i(7,7)=1$.
  \end{enumerate}
\end{definition}

Let $\hat{g}:\real{7\times 7}\to\real{9}$ denote the linear transformation induced by Definition \ref{def:g definition} on each $7\times 7$ lifted matrix, namely, the operations described in steps 1 and 2 of that definition. The manifold constraint $\mR_i\in\SO{3}$ can then be encoded through the lifted matrix $\mY_i$ by imposing the additional constraints stated in the following proposition.
\begin{proposition}\label{prop:SO3}
    A real $3\times 3$ matrix $\hat{\mR}$ is on the set $\SO{3}$ if and only if $\vec(\hat{\mR})=\hat{g}(\hat{\mY})$, where $\hat{\mY}\in\IS^7_+$ satisfies \eqref{eq:A structure} and $\rank(\hat{\mY})=1$.
\end{proposition}
\begin{proof}
  For any $\hat{\mY}$ that satisfies \eqref{eq:A structure}, $\hat{\mY}\succeq 0$ and $\rank(\hat{\mY})=1$, we can write $\hat{\mY}$ as
  \aleq{
    \hat{\mY} = \Mat{\vy_1\\ \vy_2\\1}\Mat{\vy_1\\ \vy_2\\1}^\trans,\  \vy_1,\vy_2\in\IR^3.
  }
  The structural constraint \eqref{eq:A structure} acts on the entries of $\hat{\mY}$ such that
  \begin{equation}\label{eq:ontrace}
    \begin{aligned}
      \left\{
      \begin{array}{l}
        \trace(\vy_1\vy_1^\trans)=\trace(\vy_2\vy_2^\trans)=1\\
        \trace(\vy_1\vy_2^\trans)=0
      \end{array}
      \right.
    \end{aligned}
  \end{equation}
  which is equivalent to $\norm{\vy_1}=\norm{\vy_2}=1$ and $\vy_1^\trans\vy_2=0$. We form a new matrix $\tilde{\mR}=\Mat{\vy_1 &\vy_2 &\vy_1\times\vy_2}$. When $\rank(\hat{\mY})=1$, the linear operation in step 2 of Definition \ref{def:g definition} recovers $\vy_1\times\vy_2$ exactly and therefore $\tilde{\mR}=\hat{\mR}$. It is clear that $\tilde{\mR}$ satisfies \eqref{eq:rotationvalidseparate} and thus \eqref{eq:rotationvalid} and $\hat{\mR}=\tilde{\mR}\in\SO{3}$. 

  On the other hand, for any $\hat{\mR}\in\SO{3}$, we can always use \eqref{eq:Ystructure} to construct a corresponding rank-1 $\hat{\mY}$ that satisfies \eqref{eq:A structure}, $\hat{\mY}\succeq0$, and $\rank(\hat{\mY})=1$.
\end{proof}

\subsection{Relaxation of prismatic joints}
Paralleling $\vu$, let us denote $\capgreek{\tau}:=\{\tau_i\mid i\in\cV_p\}$, where $\cV_p:=\{i\mid(i,j)\in\ave_p\}$.
For a prismatic joint, the constraint \eqref{eq:prismatic_trans} is bi-linear in $\vu$ and $\capgreek{\tau}$ because of the term $\tau_i\mR_i^{(3)}$. To be able to include this constraint in the SDP problem, for the parent of each prismatic joint $i\in\cV_p$, we introduce a new embedding variable 
\begin{equation}\label{eq:Y_itau definition}
    \begin{aligned}
       \mY_{\tau i} &= \Mat{\sqrt{\tau_i}\mR_{i}^{(3)}\\\sqrt{1-\tau_i}\mR_{i}^{(3)}\\\sqrt{\tau_i}\\\sqrt{1-\tau_i}}\Mat{\sqrt{\tau_i}\mR_{i}^{(3)}\\\sqrt{1-\tau_i}\mR_{i}^{(3)}\\\sqrt{\tau_i}\\\sqrt{1-\tau_i}}^\trans\in\IR^{8\times8},
    \end{aligned}
\end{equation}
Similar to $\mY$, we use the shorthand $\mY_\tau$ for $\{\mY_{\tau i}\}_{i\in\cV_p}\in\IR^{8\times8\times n_p}$.
We will show that the prismatic constraint \eqref{eq:prismatic_trans} is linear in $\mY_{\tau i}$ under additional conditions.
We define constraints 
    \begin{equation}\label{eq:linearconstraintsYtau}
        \begin{aligned}
            \mA_{p0,eq}\vec(\mY_{\tau}) = \vb_{p0,eq}\\
        \mA_{p,ieq}\vec(\mY_{\tau}) \leq \vb_{p,ieq}\\
        \mA_{p1,eq}\vec(\mY)+\mA_{p2,eq}\vec(\mY_{\tau}) = \vb_{p1,eq}
        \end{aligned}
    \end{equation}
to restrict the following linear relations of $\mY_{\tau i}$ and $\mY_i$ entries: for $(i,j)\in\ave_p$,
\begin{enumerate}
    \item the trace of $\mY_{\tau i}$ equals 2;\label{itm:ytau1}
    \item $\trace(\mY_{\tau i}(1:3,1:3))=\mY_{\tau i}(7,7)$ and $\trace(\mY_{\tau i}(4:6,4:6))=\mY_{\tau i}(8,8)$;\label{itm:ytau2}
      \item $\mY_{\tau i}(4:6,7)=\mY_{\tau i}(1:3,8)$;\label{itm:ytau3}
      \item $\trace(\mY_{\tau i}(1:3,4:6))=\mY_{\tau i}(7,8)$;\label{itm:ytau4}
      \item $\mY_{\tau i}(7,7)\in [0,1]$;\label{itm:ytau5}
      \item $\mY_{\tau i}(7,8)\geq 0$;\label{itm:ytau6}
      \item $\mY_{\tau i}(1:3,7)+\mY_{\tau i}(4:6,8)=\text{r.h.s. of } \eqref{eq:cross product linear fun}$.\label{itm:ytau7}
\end{enumerate}
The relations \ref{itm:ytau1}-\ref{itm:ytau4} encode the algebraic identities satisfied by the construction in \eqref{eq:Y_itau definition} when $\norm{\mR^{(3)}}=1$, including
\begin{equation}
    \begin{aligned}
        \norm{\sqrt{\tau_i}\mR_{i}^{(3)}}_2^2+\norm{\sqrt{1-\tau_i}\mR_{i}^{(3)}}_2^2&+\norm{\sqrt{\tau_i}}_2^2+\norm{\sqrt{1-\tau_i}}_2^2=2\\
        \norm{\sqrt{\tau_i}\mR_{i}^{(3)}}_2^2&=\norm{\sqrt{\tau_i}}_2^2\\
        \norm{\sqrt{1-\tau_i}\mR_{i}^{(3)}}_2^2&=\norm{\sqrt{1-\tau_i}}_2^2\\
        \sqrt{1-\tau_i}\mR_i^{(3)}\cdot\sqrt{\tau_i}&=\sqrt{\tau_i}\mR_i^{(3)}\cdot\sqrt{1-\tau_i}\\
        \sqrt{\tau_i(1-\tau_i)}\trace(\mR_i^{(3)}\mR_i^{(3)\trans})&=\sqrt{\tau_i(1-\tau_i)}
    \end{aligned}
\end{equation}
Constraint \ref{itm:ytau5} enforces that $\tau_i\in [0,1]$. 

Constraints \ref{itm:ytau4} and \ref{itm:ytau6} are to make sure that the part $\mY_{\tau i}(1:3,4:6)$ is restricted by a constraint such that the SDP solver does not assign all-zeros to these entries.

Constraint \ref{itm:ytau7} is a linear constraint on $\mY_{\tau i}$ and $\mY_{i}$ where the left-hand side is an analogue of $\tau_i\mR_i^{(3)}+(1-\tau_i)\mR_i^{(3)}$ while the right-hand side is a linear function of $\mY_{i}$ as discussed in Definition \ref{def:g definition}. This constraint is to yield that the $\mR^{(3)}_i$ extracted from $\mY_{\tau i}$ is exactly $\mR_i^{(1)}\times\mR_i^{(2)}$ from $\mY_{i}$.

Paralleling Definition \ref{def:g definition}, where the columns of rotations $\vu$ are extracted linearly from $\mY$, we define the following linear function to extract $\capgreek{\tau}$ from $\mY_\tau$.
\begin{definition}
  \begin{equation}\label{eq:gtau def}
    g_\tau(\mY_\tau) := \{\mY_{\tau i}(7,7)\}_{i\in\cV_p}
  \end{equation}
\end{definition}

With the above definitions, we can develop the following proposition, the proof of which is in the appendix.

\begin{proposition}\label{prop:prismatic trans Y_tau}
    Equations \eqref{eq:prismatic_trans} hold if and only if
    \begin{equation}\label{eq:prismatic_trans_Y_tau}
        \mT_j = \mT_i+\tau_l\mR^{(3)}_i+(\tau_u-\tau_l)\mY_{\tau i}(1:3,7),
    \end{equation}
    $(\mY_i,\mY_{\tau i})$ satisfies \eqref{eq:linearconstraintsYtau}, $\mY_i, \mY_j, \mY_{\tau i}\in\IS_+$, $\mY_i,\mY_j$ each satisfies \eqref{eq:A structure}, $\mY_i=\mY_j$, and $\rank(\mY_{i})=\rank(\mY_j)=\rank(\mY_{\tau i}) = 1$.
\end{proposition}
\begin{proof}
    Please find the proof in \ref{appendix:proof_prismatic_trans}.
\end{proof}

In \eqref{eq:ee trans forward kinematics}, the contribution of each prismatic joint to the end-effector position can be expressed linearly in $\mY_\tau$ using Proposition \ref{prop:prismatic trans Y_tau} as
\begin{equation}\label{eq:prismatic replace bilinear}
\begin{aligned}
(\tau_l+\tau_i(\tau_u-\tau_l))\mR_i^{(3)}
  = \tau_l\mR_i^{(3)} + (\tau_u-\tau_l)\mY_{\tau i}(1\!:\!3,7).
\end{aligned}
\end{equation}

\begin{remark}
  The number of variables in the relaxation \eqref{eq:Y_itau definition} can be further reduced by exploiting the problem-specific sparsity of $\mY_{\tau i}$. For example, \cite{wu2025certifiably} uses three $4\times 4$ submatrices, reducing the number of variables to 48. In this paper, we retain the form in \eqref{eq:Y_itau definition} for simplicity of exposition.
\end{remark}

\subsection{Relaxation of the feasible set and its properties}
We have expressed the kinematic constraints in terms of the rotation vectorization $\vu$ and the prismatic joint extensions $\capgreek{\tau}$, and introduced the lifted variables $\mY$ and $\mY_\tau$. Under the rank-1 condition, these constraints are linear in $\mY$ and $\mY_\tau$. We now formalize this relaxation of the robot kinematics and study its properties.

We begin by collecting the kinematic constraints on $\vu$ into the set $\cU_{\textrm{kine}}$:
\begin{equation}\label{eq:u kine def}
  \cU_{\textrm{kine}} := \left\{
  \vu \;\middle|\;
  \begin{array}{l}
    \mA_{\textrm{axis}}\vu = \vb_{\textrm{axis}},\\
    \mA_{\textrm{angle}}\vu \le \vb_{\textrm{angle}},\\
    \mA_{\textrm{parallel}}\vu = \vb_{\textrm{parallel}}
  \end{array}
  \right\}.
\end{equation}
Here, the constraint $\mA_{\textrm{parallel}}\vu = \vb_{\textrm{parallel}}$ enforces $\mR_i = \mR_j$ for $(i,j)\in \ave_p$. The definition in \eqref{eq:u kine def} can be readily extended to incorporate other constraints on the rotation matrices, such as fixed or convexly constrained positions or orientations of arbitrary links.

\begin{definition}\label{def:kine feas set relax}
  Let the \textbf{original feasible set} $\cU$ denote the set of pairs $\{(\vu,\capgreek{\tau})\mid \vu= \stack(\{\vec(\mR_{i})\mid i\in\cV_r\}), \text{\boldmath{$\tau$}}:=\{\tau_i\mid i\in\cV_p\}\}$ satisfying both kinematics and manifold constraints:
  \begin{subequations}\label{eq:def ik feas set}
    \begin{align}
      &\text{[Kinematics, linear]}&\vu \in \cU_{\textrm{kine}} \label{eq:kine}\\[3pt]
      &\multirow{2}{*}{\text{[Manifold, nonlinear]}}
    & \mR_i \in \SO{3}, \forall i\in\cV_r \label{eq:ik manifold}\\
    & & \tau_i \in [0,1], \forall i\in\cV_p .
    \end{align}
  \end{subequations}
  Let $\cY:=\{(\mY,\mY_\tau)\}$ denote the corresponding \textbf{lifted feasible set} induced by
  \begin{equation}
      \begin{aligned}
          \mY &:= \{\mY_i\mid\mY_i=y(\vu),\quad (\vu,\capgreek{\tau})\in\cU\}_{i\in\cV_r} \text{ and }\\ 
          \mY_{\tau} &:= \{\mY_{\tau i}\mid\mY_{\tau i}=y_{\tau}(\vu,\capgreek{\tau}),\quad(\vu,\capgreek{\tau})\in\cU\}_{i\in\cV_p},
      \end{aligned}
  \end{equation}
  where $y$ and $y_\tau$ are functions of $\vu$ and $\capgreek{\tau}$ that formulate the matrices in \eqref{eq:Ystructure} and \eqref{eq:Y_itau definition}, respectively.

  Let $\bar{\cY}$ denote the \textbf{relaxed lifted set} defined by 
  $\bar{\cY}:=(\bar{\mY},\bar{\mY}_\tau)$, where $\bar{\mY}:=\{\bar{\mY}_i\}_{i\in\cV_r}, \bar{\mY}_{\tau}:=\{\bar{\mY}_{\tau i}\}_{i\in\cV_p}$, and
  \begin{subequations}\label{eq:defsdprelaxset}
    \begin{align}
      &\text{[Kinematics, linear]} &g(\bar{\mY})\in\cU_{\textrm{kine}}\label{eq:ubar kine}\\[3pt]
      &\multirow{2}{*}{\text{[Manifold, PSD]}}  &\bar{\mY} = \{\bar{\mY}_i\in\IS^{7}_+\mid i\in\cV_r\},\label{eq:ubar psd}\\
      & &\bar{\mY}_{\tau} = \{\bar{\mY}_{\tau i}\in\IS^{8}_+\mid i\in\cV_p\},\label{eq:ubar psd ytau}\\[3pt]
      &\multirow{2}{*}{\text{{[Manifold, linear]}}}  &\mA_{\textrm{structure}}\vec(\bar{\mY})=\vb_{\textrm{structure}},\label{eq:ubar structure}\\
      & &\bar{\mY},\bar{\mY}_{\tau}\text{ satisfy } \eqref{eq:linearconstraintsYtau}.\label{eq:ubar yytau}
    \end{align}
  \end{subequations}

  Its projection onto the original variables, the \textbf{projected relaxed set}, is
  \begin{equation}
  \bar{\cU}:=\{g(\bar{\mY}),g_\tau(\bar{\mY}_\tau)\mid \bar{\mY},\bar{\mY}_\tau\in\bar{\cY}\}.   
  \end{equation}
  
  Finally, let $\cR_1$ be the \textbf{rank-1 subset}: $\cR_1:=\{(\mY,\mY_\tau)\mid \rank(\mY_i)=\rank(\mY_{\tau j})=1,\forall i\in\cV_r, j\in\cV_p\}$.
\end{definition}
Intuitively, $\cU$ is the set of kinematically feasible solutions, defined by linear constraints on $\vu$ and $\capgreek{\tau}$ together with the nonlinear requirement that the rotations lie on $\SO{3}$. Lifting these variables yields the set $\cY$, whose elements are implicitly rank-1, that is, $\mY,\mY_\tau \in \cR_1$. Dropping the rank-1 constraints gives the relaxed lifted set $\bar{\cY}$, which can then be projected back to the original variable space through the linear maps $g$ and $g_\tau$ to obtain $\bar{\cU}$.
We show some useful results about these sets.

\begin{proposition}
  The set $\bar{\cY}$ is bounded.
\end{proposition}
\begin{proof}
  For each \(\mY_i\succeq 0\), the structural constraints imply $\trace(\mY_i)=3$.
    Hence all eigenvalues of \(\mY_i\) are nonnegative and bounded by \(3\).
    Therefore \(\|\mY_i\|_F\leq 3\). Similarly, (24) imposes
    \(\trace(\mY_{\tau i})=2\), so \(\|\mY_{\tau i}\|_F\leq 2\).
    Since there are finitely many blocks, \(\bar{\mathcal Y}\) is bounded.
\end{proof}

\begin{proposition}\label{prop:u subset ubar}
  The set $\bar{\cU}$ contains every element of $\cU$, i.e., $\cU\subseteq \bar{\cU}$.
\end{proposition}
\begin{proof}
This is a consequence of the fact that we can build $\mY,\mY_\tau\in\cY$ that satisfy all the constraints of $\bar{\cY}$ for any $(\vu,\capgreek{\tau})\in\cU$ using \eqref{eq:Ystructure} and \eqref{eq:Y_itau definition}.
\end{proof}
\begin{proposition}\label{prop:intersect u_bar rank1}
  The set $\bar{\cU}$ exactly matches the set $\cU$ if $\bar{\cU}$ is not only an image of $\bar{\cY}$, but also an image of $\bar{\cY}\intersect\cR_1$, i.e., 
  $\bar{\cU}=\{g(\bar{\mY}),g_\tau(\bar{\mY}_\tau)\mid \bar{\mY},\bar{\mY}_\tau\in\bar{\cY}\intersect\cR_1\}$.
\end{proposition}
\begin{proof}
    Using Proposition \ref{prop:SO3} we know that for any $\vu=\stack(\{\vec(\mR_{i})\mid i\in\cV_r\})\in\cU$, we have $\mR_i\in\SO{3},\forall i\in\cV_r$ if and only if every corresponding $\mY_i$ as a function of $\mR_i$ through \eqref{eq:Ystructure} satisfies \eqref{eq:ubar structure}, $\mY_i\succeq0$, and $\rank(\mY_i)=1$. Therefore the equivalence holds between \eqref{eq:ik manifold} and $\{$\eqref{eq:ubar psd}, \eqref{eq:ubar psd ytau}, \eqref{eq:ubar structure}$\}\intersect \cR_1$. According to Proposition \ref{prop:prismatic trans Y_tau}, for every $(i,j)\in\ave_p$, $\mR_i=\mR_j\in\SO{3}$, $\tau_i\in[0,1]$ hold if and only if $\mY,\mY_\tau$ satisfy \eqref{eq:linearconstraintsYtau} (enforced by \eqref{eq:ubar yytau}), $\mR_i=\mR_j$ (enforced by \eqref{eq:ubar kine}), and $\rank(\mY_{\tau i})=1$.
    The equivalence holds between the constraints on revolute joints, i.e., $\eqref{eq:kine}$ and $\eqref{eq:ubar kine}\intersect \cR_1$ and all the constraints in $\bar{\cY}$ and $\bar{\cU}$ are covered. 
\end{proof}
\begin{proposition}\label{prop:u subset boundary}
  The set $\cY$ is a subset of the boundary of $\bar{\cY}$, i.e., $\cY\subseteq\partial \bar{\cY}$.
\end{proposition}
\begin{proof}
  For any $(\mY,\mY_\tau)\in\cY$, each $\mY_{i}$ or $\mY_{\tau i}$ is rank-1 and has a zero eigenvalue. Therefore, there exists a matrix $\mM$ such that $\tilde{y}_i(t)=\mY_{i}+t\mM\notin \IS_+$ 
  for any $t>0$. Thus $(\mY,\mY_\tau)$ is on the boundary of $\bar{\cY}$.
\end{proof}

These propositions show that the relaxation is bounded and contains the original lifted feasible set on its boundary. Moreover, when the projected relaxed solution is rank-1, it recovers the original feasible set exactly. These results provide the foundation for formulating IK as optimization problems.

\section{Optimization}\label{sec:opt}
This section first introduces how to build an inverse kinematics optimization problem, then discusses how to relax this problem using the relaxed set defined in Section \ref{sec:relaxation}, and finally introduces a rank minimization algorithm to find low-rank solutions, with local convergence guarantees toward the exact solutions of the original IK problem.
\subsection{Forward kinematics as a linear function}\label{sec:fk_linear}
Forward kinematics maps the robot configuration to the end-effector pose. In our rotation-based parameterization, the end-effector orientation is directly represented by the rotation matrix $\mR_{ee}$. Equation \eqref{eq:ee trans forward kinematics} expresses the end-effector position as a linear function of $\vu$ and the bilinear terms $\{\tau_i\mR_i^{(3)}\mid (i,j)\in\ave_{fk}\cap\ave_{p}\}$. Since $\vu=g(\mY)$ is linear in $\mY$, and the bilinear terms can be represented linearly in $(\mY,\mY_\tau)$ via \eqref{eq:prismatic replace bilinear}, the forward kinematics becomes linear in $\mY$ and $\mY_\tau$:
\begin{equation}\label{eq:forward kine linear expression}
\begin{aligned}
(\vec(\mR_{ee}),\mT_{ee})
&:= \left(\mE_{ee}\vu,\hat{l}_{ee}(\vu,\capgreek{\tau})\right)\\
&= \left(\mE_{ee}g(\mY),l_{ee}(\mY,\mY_\tau)\right).
\end{aligned}
\end{equation}
Here, $\mE_{ee}\in\IR^{9\times 9n_r}$ is a selection matrix such that $\vec(\mR_{ee})=\mE_{ee}\vu$, and
\begin{equation}\label{eq:fk linear terms}
\begin{aligned}
\hat{l}_{ee}(\vu,\capgreek{\tau})
&= \text{r.h.s of \eqref{eq:ee trans forward kinematics}} \\
l_{ee}(\mY,\mY_\tau)
&= \mT_{base} + \mA_t\vec(\mY)+\mB_t\vec(\mY_\tau).
\end{aligned}
\end{equation}
Thus, the end-effector translation $l_{ee}$ is linear in $(\mY,\mY_\tau)$. The full expressions of $\mA_t$ and $\mB_t$ are given in \ref{sec:full f}.
\begin{remark}\label{rem:fk_any_point}
  By \eqref{eq:forward kine linear expression}, the pose of any point rigidly attached to the robot is a linear function of $(\mY,\mY_\tau)$, provided there is a path $\cP_{fk}$ from the base to that point.
\end{remark}

\subsection{Inverse kinematics problem}\label{sec:probform}
The inverse kinematics problem aims to find the optimal and feasible $\vx^*$ such that the end-effector $ee$, reaches a desired location $\mT_{goal}$ and orientation $\mR_{goal}$. This objective can be encoded as 
\begin{equation}
\begin{aligned}
    f&:=\norm{f_{r,ee}}_2^2+\norm{f_{t,ee}}^2_2\text{, where}\label{eq:eecost}\\
    f_{r,ee}&:=\vec(\mR_{ee})-\vec(\mR_{goal})\text{, and}\\
    f_{t,ee}&:=\mT_{ee}-\mT_{goal}.
\end{aligned}
\end{equation}
We then substitute \eqref{eq:forward kine linear expression} into the cost to get
\begin{subequations}
    \begin{align}
        f_{r,ee}(\mY)&=\mE_{ee}g(\mY)-\vec(\mR_{goal})\\
f_{t,ee}(\mY,\mY_\tau)&=l_{ee}(\mY,\mY_\tau)-\mT_{goal}\\
  f(\mY,\mY_\tau) &= \norm{f_{r,ee}(\mY)}_2^2+\norm{f_{t,ee}(\mY,\mY_\tau)}_2^2.\label{eq:ee2normcostfk}
    \end{align}
\end{subequations}
We note that the cost $f$ is quadratic in $(\mY,\mY_\tau)$. See \ref{sec:full f} for a full expression of $f$.
Our inverse kinematics problem is then defined as
\begin{prob}[1][Inverse kinematics]\label{prob:IK1}
  \begin{subequations}\label{eq:ikinitialproblem}
    \begin{align}
      &\min_{\mY,\mY_{\tau}} && f(\mY,\mY_\tau)\\
      &\subjectto && \mY,\mY_\tau\in\cY \label{eq:IK1 u}
    \end{align}
  \end{subequations}
\end{prob}
As defined in Definition \ref{def:kine feas set relax}, $\cY$ is the exact lifted feasible set of the kinematic constraints. Therefore, any $(\mY,\mY_\tau)\in\cY$ satisfying $f(\mY,\mY_\tau)=0$ yields $\vu=g(\mY)$ and $\capgreek{\tau}=g_\tau(\mY_\tau)$ such that $(\mR_{ee},\mT_{ee})=(\mR_{goal},\mT_{goal})$, that is, a configuration whose end-effector pose exactly matches the target.

\subsection{Rank constrained problem and relaxations}
Using Proposition \ref{prop:intersect u_bar rank1}, we can equivalently write the inverse kinematics problem as
\begin{prob}[2a][Rank constrained inverse kinematics]\label{prob:IK2a}
  \begin{subequations}\label{eq:ikrcproblem}
    \begin{align}
      &\min_{\mY,\mY_{\tau}} && f(\mY,\mY_\tau)\\
      &\subjectto && \mY,\mY_\tau\in\bar{\cY} \label{eq:IK2 Ytau}\\
      & &&\rank(\mY_{i}) = 1,\ i \in \cV_r \label{eq:IK2 Y rank}\\
      & &&\rank(\mY_{\tau i}) = 1,\ i \in \cV_p \label{eq:IK2 Ytau rank}
    \end{align}
  \end{subequations}
\end{prob}
This problem has a quadratic objective function and convex constraints except for \eqref{eq:IK2 Y rank} and \eqref{eq:IK2 Ytau rank}, which are nonconvex. We define the following relaxed problem, which is obtained from Problem \ref{prob:IK2a} with the omission of the rank constraints.
\begin{prob}[2b][Relaxed inverse kinematics]\label{prob:IK2b}
  \begin{subequations}\label{eq:ikrcproblem_relax}
    \begin{align}
      &\min_{\mY,\mY_{\tau}} && f(\mY,\mY_\tau)\\
      &\subjectto && \mY,\mY_\tau\in\bar{\cY} \label{eq:IK2b u}
    \end{align}
  \end{subequations}
\end{prob}
\begin{remark}
  The cost $f(\mY,\mY_\tau)$ is a sum of squared Euclidean norms of affine functions and is therefore convex. Hence, Problem \ref{prob:IK2b} is a convex optimization problem with a convex objective and linear and semidefinite constraints. Although it is not written in standard SDP form, it can be reformulated as an SDP with a linear objective, a transformation handled automatically by modeling systems such as CVX \cite{cvx,gb08}.
\end{remark}
\begin{remark}\label{rem:infeasibility certify}
If Problem~\ref{prob:IK2b} is infeasible, then
Problem~\ref{prob:IK2a} is also infeasible. Moreover, if an optimal
solution \((\mY^*,\mY_\tau^*)\) of Problem~\ref{prob:IK2b} satisfies
\((\mY^*,\mY_\tau^*)\in\cY\), then it is also optimal for
Problem~\ref{prob:IK2a}. 

For a prescribed target pose, the infeasibility of exact pose matching can be certified by solving the relaxed feasibility problem obtained by adding the linear forward-kinematics constraints
\[
    f_{r,ee}(\mY,\mY_\tau)=\mathbf{0}, \qquad
    f_{t,ee}(\mY,\mY_\tau)=\mathbf{0}
\]
to the relaxed feasible set \(\bar{\cY}\). If this relaxed feasibility
problem is infeasible, then no feasible solution of the original
rank-constrained IK problem can exactly reach the target pose.
\end{remark}
Remark \ref{rem:infeasibility certify} provides us a way to certify the infeasibility of the IK problem. If 
Problem \ref{prob:IK2b} is infeasible, we are certain that there exists no feasible solution in the original feasible set $\cU$. Conversely, if a solution to the non-relaxed problem matches the optimal objective of the relaxed problem, then it is globally optimal; however, in the case of IK, this can be more easily detected by having cost equal 0.

\subsection{Rank minimization via eigenvalue maximization}\label{sec:rankmin}
This paper proposes two ways to project solutions of the relaxed problem to the set of rank-1 matrices by manipulating their eigenvalues. The first can result in faster convergence, but only works under the assumption that optimal objective of the IK problem matches the optimal objective of the relaxation (Assumption~\ref{as:optimalilty in feas set}). The second is slower in practice but requires no assumption. We introduce the first method in this section.

 For the rank minimization, we will make use of the gradient of an eigenvalue with respect to the entries of the corresponding matrix. Specifically, consider a matrix $\mA\in\IS^{n}$ that has multiplicity 1, and let the eigenvalues of $\mA$ be $\lambda_1>\dots>\lambda_n$. We are interested in finding $\frac{\partial \lambda_l}{\partial \vec(\mA)}$, for the $l$-th largest eigenvalue $\lambda_l$. Lemma \ref{lem:eigendiff} summarizes a result from \cite{magnus1985} (see that paper for a proof).
\begin{lemma}[Gradient of eigenvalues]\label{lem:eigendiff}
  Given $\mX_0\in\IS^{n}$, let $\vv_l$ and $\lambda_l$ be a pair of \emph{normalized} eigenvector and eigenvalue of $\mX_0$. For functions $\lambda(\mX)=\lambda_l$ and $v(\mX)=\vv_l$ defined on neighborhood $N(\mX_0)\subset \IR^{n\times n}$ of $\mX_0$, the gradient of $\lambda(\mX)$ at $\mX_0$ is given by
  \aleq{
    \frac{\partial \lambda}{\partial \vec(\mX)}=\vv_l\otimes\vv_l.\label{eq:grad_eigenvalue}
  }
\end{lemma}

\begin{lemma}\label{lem:lambda1convex}
The largest eigenvalue as a function $\lambda_{1}(\mX)$ is convex in $\mX\in\IS^n$.
\end{lemma}
\begin{proof}
     It is easy to see that for every $\vv$ the function $f(\mX)=\vv^\trans\mX\vv$ is convex in $\mX$ and therefore $\lambda_{1}(\mX) = \sup_{\norm{\vv}_2=1} \vv^\trans \mX \vv$ is convex in $\mX$.
\end{proof}
We have shown that the largest eigenvalue $\lambda_1$ is convex. When it is simple, its gradient is given by Lemma 1. In addition, if the matrix is PSD and has a fixed trace, we can develop a condition for the matrix to be rank-1 using the following Proposition.
\begin{proposition}\label{prop:rank_trace}
  Consider PSD matrix $\mM\in\IS^m_+$, with eigenvalues $\lambda_1\geq\dots\geq\lambda_m$. Consider the function $\lambda_1(\mM)=\lambda_1$ 
  subject to the constraint $\trace(\mM)=c$ and $c>0$, $\hat{\mM}$ is a rank-1 matrix if and only if $\hat{\mM}\in\argmax(\lambda_1(\mM))$.
\end{proposition}
\begin{proof}
  The trace of a matrix is also the sum of all of its eigenvalues, which are all non-negative when the matrix is positive semidefinite. For the ``if'' direction, when the constant $\trace(\mM)=c$, we have $\lambda_i\leq c$, and the condition $\hat{\mM} \in \argmax(\lambda_1(\mM))$ is achieved when $\lambda_1(\hat{\mM})=c$, implying $\lambda_{2}(\hat{\mM}),\dots,\lambda_m(\hat{\mM})=0$, and hence $\rank(\hat{\mM})=1$. For the ``only if'' direction, under the constant $\trace(\mM)=c$, $\rank(\hat{\mM})=1$ implies that the only positive eigenvalue $\lambda_1(\hat{\mM})$ equals $c$, and since $0\leq\lambda_i(\mM)\leq c$ we have $\hat{\mM}\in\argmax(\lambda_1(\mM))$.
\end{proof}

\begin{assumption}\label{as:optimalilty in feas set}
    The optimal value of the relaxed Problem \ref{prob:IK2b} is zero, and equals that of the rank-constrained Problem \ref{prob:IK2a}: 
    \begin{equation}
        \min_{(\mY,\mY_\tau)\in\cY}f(\mY,\mY_\tau)=\min_{(\mY,\mY_\tau)\in\bar{\cY}}f(\mY,\mY_\tau)=0
    \end{equation}
    and there exists an optimal solution of Problem \ref{prob:IK2b} satisfying the rank-1 constraints.
\end{assumption}

Recall that the structural constraints \eqref{eq:A structure} and \eqref{eq:linearconstraintsYtau} enforce that the traces of $\mY$ and $\mY_\tau$ are constants.
Under Assumption \ref{as:optimalilty in feas set} and using Proposition \ref{prop:rank_trace} we can rewrite Problem \ref{prob:IK2a} as the following problem:
\begin{prob}[2c][Eigenvalue maximization]\label{prob:IK2c}
  \begin{subequations}\label{eq:ikrcproblem_max_lambda1}
    \begin{align}
      &\max_{\mY,\mY_{\tau}} && \sum_{i\in\cV_r}\lambda_{1}(\mY_{i})+\sum_{j\in\cV_p}\lambda_{1}(\mY_{\tau j})\label{eq:ik2c_obj}\\
      &\subjectto && (\mY,\mY_\tau) \in \argmin(f(\mY,\mY_\tau))\label{eq:ik2cargmin}\\
      & &&\mY,\mY_\tau\in \bar{\cY} \label{eq:ik2c_u}
    \end{align}
  \end{subequations}
\end{prob}
Notice that it is assumed in the above problem that the optimal solution exists in the relaxed set, meaning that there exists an end-effector pose that reaches the desired pose. 

We propose a gradient-based approach to solve Problem \ref{prob:IK2c}. The idea is to increase the largest eigenvalue, and the other eigenvalues will decrease because the traces of our variables are fixed. To begin with, we define the following operator to simplify the formulation.
\begin{equation}
    Z(\mA,\vv) = \vec(\mA)^\trans (\vv\otimes\vv)=\vv\transpose \mA \vv=\trace(\vv\vv\transpose \mA)
\end{equation}

To find a solution for Problem \ref{prob:IK2c}, we first find a solution to Problem \ref{prob:IK2b} and then minimize the rank iteratively. In each step, the following problem is solved.
\begin{prob}[3][Eigenvalue increase update]\label{prob:IK3}
  \begin{subequations}\label{eq:ikrcproblem_eigen}
    \begin{align}     
       &\max_{\mU^k,\mU^{k}_{\tau}} &&\sum_{i\in\cV_r} Z(\mU^k_{i},\mV^{k-1,(1)}_{i}){+}\sum_{j\in\cV_p} Z(\mU^{k}_{\tau j},\mV^{k-1,(1)}_{\tau j}) \label{eq:ik3obj}\\
      &\textrm{s.t.} && f_{r,ee}(\mY^{k-1}+\mU^k,\mY_\tau^{k-1}+\mU_\tau^k)=\omat,\notag\\ 
      & &&f_{t,ee}(\mY^{k-1}+\mU^k,\mY_\tau^{k-1}+\mU_\tau^k)=\omat,\label{eq:ik3grad}\\
      & &&\mY^{k-1}+\mU^k,\mY^{k-1}_{\tau}+\mU^{k}_{\tau}\in\bar{\cY}                                     \label{eq:ik3 u}
    \end{align}
  \end{subequations}
\end{prob}

The variable $\mU^k$ denotes the update to $\mY^{k-1}$, namely, $\mY_i^k=\mY_i^{k-1}+\mU_i^k$, and $\mU_\tau^k$ is defined analogously for $\mY_\tau^{k-1}$. The objective \eqref{eq:ik3obj} maximizes the sum of inner products between the updates and the gradients of the largest eigenvalues. In particular, by Lemma \ref{lem:eigendiff},
\begin{equation}
  \begin{aligned}
  \sum_{i\in\cV_r}\vec&(\mU_i^k)^\trans
  \frac{\partial \lambda_1(\mY_i^{k-1})}{\partial \vec(\mY_i^{k-1})}\\
  &{=}
  \sum_{i\in\cV_r}\vec(\mU_i^k)^\trans
  (\mV_i^{k-1,(1)}\otimes \mV_i^{k-1,(1)}),
  \end{aligned}
\end{equation}
where $\mV_i^{k-1,(1)}$ is the normalized eigenvector of $\mY_i^{k-1}$ associated with $\lambda_1$. Thus, \eqref{eq:ik3obj} drives each update in the direction that most increases the sum of the largest eigenvalues of $\mY_i^k$ and $\mY_{\tau j}^k$. Since the traces of these matrices are fixed, increasing their largest eigenvalues decreases the remaining eigenvalues, thereby driving the matrices toward rank-1.
The constraint \eqref{eq:ik3grad} enforces that the updates lie in the zero level set of $f$ at $(\mY^0,\mY_\tau^0)$, where $(\mY^0,\mY_\tau^0)$ is obtained by solving Problem \ref{prob:IK2b}. By Assumption \ref{as:optimalilty in feas set}, and since Problem \ref{prob:IK2b} is convex, we have $(\mY^0,\mY_\tau^0)\in\argmin f(\mY,\mY_\tau)$. Hence, the updated pair $(\mY^k,\mY_\tau^k)$ remains a minimizer of $f(\mY,\mY_\tau)$. Finally, \eqref{eq:ik3 u} ensures that $(\mY^k,\mY_\tau^k)$ stays in the feasible set $\bar{\cY}$.

\subsection{Adaptive rank minimization}\label{sec:alternative rank min}
The rank minimization process in the previous section relies on Assumption \ref{as:optimalilty in feas set}. In practice, there exist IK problems where we don't know if there are feasible solutions that result in the optimal cost of the relaxed problem. For example, in some tracking problems, we want to find the joint configurations that lead the end-effector to be \emph{as close as possible} to a target that is too far for the robot to reach. In this case, the optimal cost of the original IK problem is non-zero and needs to be determined. To be able to account for such situations, we introduce the following adaptive rank minimization algorithm.

We define the following function of $\mY$
\begin{equation}\label{eq:function V}
    \begin{aligned}
        w(\mY) = \sum_{i}^{n_r} 3-\lambda_1(\mY_{i})
    \end{aligned}
\end{equation}

\begin{lemma}\label{lem:V concavity}
For any two $\mY^k,\mY^{k+1}$ it holds that
\begin{equation}\label{eq:V convavity}
w(\mY^{k+1})\leq w(\mY^k) -\sum_i^{n_r}\langle\nabla \lambda_{1}(\mY^k_{i}),\mY^{k+1}_i-\mY^k_{i}\rangle,
\end{equation}
where $\nabla\lambda_1(\mY^k_{i})=\partial \lambda_{1}(\mY^{k}_{i})/\partial \vec(\mY^{k}_{i})$.
\end{lemma}

\begin{proof}
    From Lemma \ref{lem:lambda1convex}, we have that $\lambda_{1}(\mY_i)$ is a convex function of the entries of $\mY$, hence $3-\lambda_{1}(\mY_i)$ is a concave function, and the sum of concave functions is still concave. As a consequence of the properties of concave functions, we have
    \begin{equation}
w(\mY^{k+1})\leq w(\mY^k) + \langle\nabla_{\mY^k} w(\mY^{k}),\mY^{k+1}-\mY^{k}\rangle.
\end{equation}
Figure \ref{fig:concave} visualizes this relation. By substituting \eqref{eq:function V}, we get \eqref{eq:V convavity}.
\end{proof}
\begin{figure}[H]
  \centering
  \includegraphics[width=0.98\linewidth]{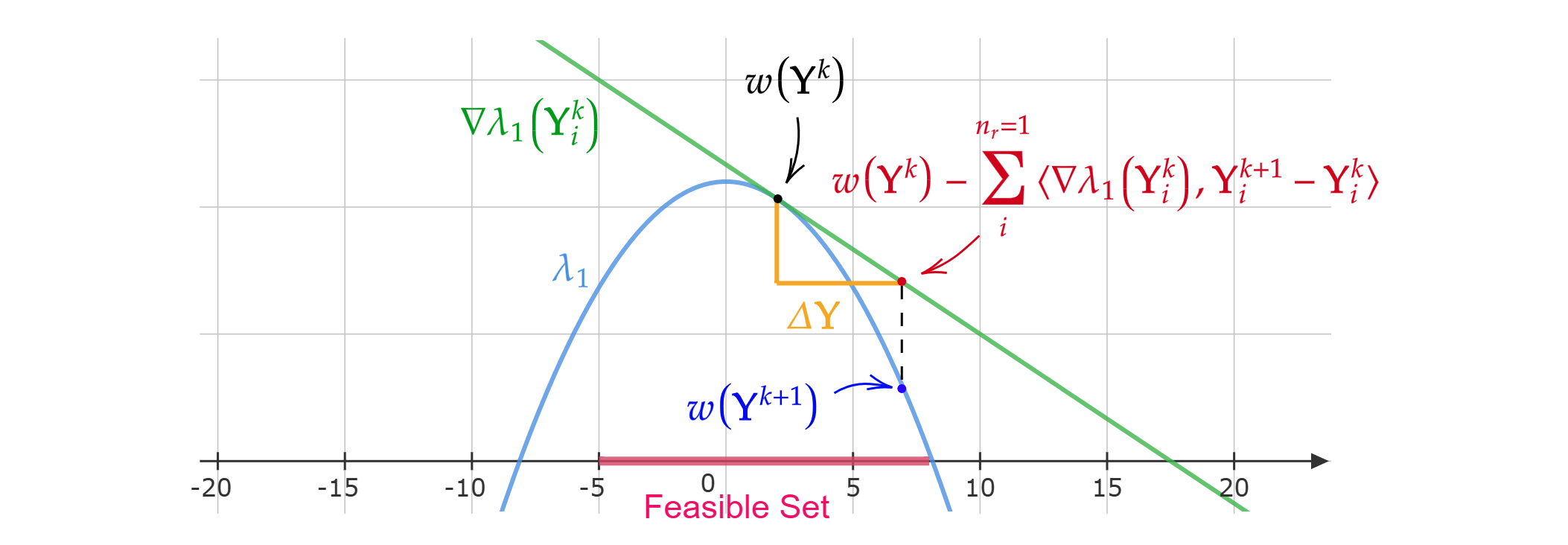}
  \caption{Suppose $n_r = 1$ and $3-\lambda_1$ is a concave quadratic function. This figure shows the relation in Lemma \ref{lem:V concavity} resulting from the concavity of the function. }
  \label{fig:concave}
\end{figure}

\begin{proposition}\label{prop:W geo decrease}
Given a sequence $\{\mY^k\}$ and a sequence of constants $\{c^{(k)}\}$ (the superscript denotes the sequence index), if the condition 
\begin{equation}\label{eq:lin approx c Vk}
    w(\mY^k) -\sum_i^{n_r}\langle\nabla \lambda_{1}(\mY^k_{i}),\mY^{k+1}_i-\mY^k_{i}\rangle\leq c^{(k)} w(\mY^k)
\end{equation}
is satisfied for every $k$, then 
\begin{equation}\label{eq:min c lin}
    w(\mY^k)\leq (\prod_{\ell=1}^k c^{(\ell)}) w(\mY^0)
\end{equation}
for every $k$.
\end{proposition}
\begin{proof}
Eq. \eqref{eq:lin approx c Vk} combined with Lemma \ref{lem:V concavity} implies that $w(\mY^{k+1})\leq c^{(k)}w(\mY^k)$. The claim then follows by induction.
\end{proof}

\begin{theorem}\label{thm:V convergence}
Assume there exists \(\bar{c}<1\).
If, for every $k$, \eqref{eq:lin approx c Vk} is satisfied and $c^{(k)}\in[0,\bar{c}]$, then $\lim_{k\to\infty} w(\mY^k)=0$.
\end{theorem}
\begin{proof}
    Since $\mY^k_{i}\succeq 0$ and $\trace(\mY^k_{i})=3$, we have $\lambda_{1}(\mY^k_{i})\leq 3$, and $w(\mY^k)\geq 0$. Combined with Proposition \ref{prop:W geo decrease} we know that $0\leq w(\mY^k)\leq 3n_r (\Pi_{1}^k c^{(k)})$. The claim is then followed by the squeeze theorem.
\end{proof}
Equivalently, if we replace $\mY_i$ with $\mY_{\tau i}$ and change the number 3 into 2 in \eqref{eq:function V}, Theorem \ref{thm:V convergence} also holds because the trace of $\mY_{\tau i}$ is always 2. We develop the following problem to be solved iteratively in the rank minimization process.

\begin{prob}[4][Adaptive rank-reduction update]\label{prob:IK4}
  \begin{subequations}
    \begin{align}
      &\min_{\mU^k,\mU^{k}_{\tau}} f(\mY^{k-1}+\mU^{k},\mY^{k-1}_{\tau}+\mU^{k}_{\tau})\label{eq:adaptive rank min obj}\\
          &\subjectto \notag\\
          &\sum_{i\in\cV_r}\vec(\mU^k_{i})^\trans\nabla\lambda_{1}(\mY^{k-1}_i) {\geq} \sum_{i\in\cV_r}(c^{(k)}{-}1)(\lambda_{1}(\mY^{k-1}_i){-}3) \label{eq:lin approx c Y}\\
          &\sum_{i\in\cV_p}\vec(\mU^{k}_{\tau i})^\trans\nabla\lambda_{1}(\mY^{k-1}_{\tau i}) {\geq} \sum_{i\in\cV_p}(c^{(k)}{-}1)(\lambda_{1}(\mY^{k-1}_{\tau i}){-}2)\label{eq:lin approx c Ytau}\\
          &\mY^{k-1}+\mU^k, \mY^{k-1}_{\tau}+\mU_\tau^{k}\in \bar{\cY}\label{eq:IK4 Y}
    \end{align}
  \end{subequations}
\end{prob}

In the above problem, \eqref{eq:lin approx c Y} and \eqref{eq:lin approx c Ytau} are derived by substituting $w(\mY)$ and $w(\mY_\tau)$ into \eqref{eq:lin approx c Vk}, respectively. By enforcing \eqref{eq:lin approx c Y} and \eqref{eq:lin approx c Ytau} and according to Theorem \ref{thm:V convergence}, we have $\lim_{k\rightarrow\infty}\lambda_1(\mY^k_i)=3,\forall i \in \cV_r$ and $\lim_{k\rightarrow\infty}\lambda_1(\mY^k_{\tau i})=2,\forall i\in \cV_p$. By updating $\mY$ and $\mY_\tau$ iteratively using Problem \ref{prob:IK4}, we can move $\mY$ and $\mY_\tau$ toward rank-1 matrices while allowing the cost to increase, thus enabling us to solve for problems where Assumption 1 does not hold. However, in Problem \ref{prob:IK4} the factor $c^{(k)}$ has to be chosen properly. If $c^{(k)}$ becomes too small, the process becomes too ``aggressive'' and there might not exist updates $\mU^k$ and $\mU_\tau^k$ such that $\mY^{k-1}+\mU^k, \mY^{k-1}_{\tau}+\mU_\tau^{k}\in \bar{\cY}$. To account for that, we can tune $c^{(k)}$ automatically and adaptively during the rank minimization process. We initialize with $c^{(k)}\leftarrow c_0$ and a counter $p=0$, and solve Problem \ref{prob:IK4} to find a feasible solution. If the problem is infeasible, meaning that $c^{(k)}$ is too aggressive, then we update $c^{(k)}\leftarrow c_p$ with the following rule until we find a feasible solution of Problem \ref{prob:IK4}:
\begin{equation}\label{eq:adaptive c}
    \begin{aligned}
        p&\gets p+1,\\
    c_p&= 1-(1-c_0)^{p+1}.
    \end{aligned}
\end{equation}
The purpose of using this adaptive update is to find a $c^{(k)}$ that is neither too aggressive nor too conservative in terms of increasing $\lambda_1$. In Section \ref{sec:convergence}, we show that, if $p\rightarrow \infty$, then the corresponding $\mY^k,\mY^k_\tau$ is a local maximizer of Problem \ref{prob:IK2c}. Practically, we can set a large number as an upper-bound for $p$, and if this upper-bound is reached without finding a feasible solution, then the rank minimization has moved the matrices to a point close to a local maximizer of Problem \ref{prob:IK2c}.
\subsection{Algorithm}

The complete algorithm that solves the inverse kinematics Problem \ref{prob:IK1} is summarized in Algorithm \ref{alg:ik}.
\begin{algorithm}[h]
  \caption{Iterative SDP Inverse Kinematics Solver}\label{alg:ik}
  \hspace*{\algorithmicindent} \textbf{Input} $\mT_{goal}$, $\mR_{goal},\epsilon_1,\epsilon_2,k_{max}$\\
  \hspace*{\algorithmicindent} \textbf{Output} $\vx^*$
  \begin{algorithmic}[1]
    \State Solve the convex relaxation (Problem \ref{prob:IK2b}) to get an initial solution $\mY^0,\mY_{\tau}^0$, initialize $(\mU^0,\mU_\tau^0)=+\infty$, and set $k=1$.\label{step:init}
    \While{($\exists \lambda_{1}(\mY_i)\leq3-\epsilon_1$ or $\exists\lambda_{1}(\mY_{\tau i})\leq2-\epsilon_1$) and $\norm{(\mU^{k-1},\mU_\tau^{k-1})}_{F}\geq \epsilon_2$ and $k\leq k_{max}$}\label{step:startwhile}
    \State \parbox[t]{200pt}{For each $\mY^{k-1}_{i}$, compute the largest eigenvalue $\lambda_{1}(\mY_i)$ and the corresponding normalized eigenvector $\mV^{k-1,(1)}_{i}$. Repeat the same for $\mY^{k-1}_{\tau j}$ to get $\lambda_{1}(\mY_{\tau j})$ and $\mV^{k-1,(1)}_{\tau j}$.}
    \State \parbox[t]{200pt}{Compute updates $\mU^{k}$ and $\mU^{k}_{\tau}$ by solving Problem~\ref{prob:IK3} or Problem \ref{prob:IK4}.}\label{step:update}
    \State \parbox[t]{200pt}{Update $\mY^k_{i} = \mY^{k-1}_{i}+\mU^k_{i}$ for all $i\in\cV_r$ and update $\mY^k_{\tau j} = \mY^{k-1}_{\tau j}+\mU^k_{\tau j}$ for all $j\in\cV_p$ and set $k = k+1$.}
    \EndWhile\label{step:endwhile}
    \State Extract the rotations $\{\mR_{i}\}$ from $g(\mY^{k-1})$.\label{step:getrotations}
    \State Extract $\{\tau_i\mid i\in \cV_p\}$ from $g_\tau(\mY^{k-1}_{\tau})$.\label{step:gettau}
    \State Recover the translations $\{\mT_{i}\}$ using \eqref{eq:translation relation} and \eqref{eq:prismatic_trans}.
    \State \Return $\vx^*$ defined in \eqref{eq:xdefine}.
  \end{algorithmic}
\end{algorithm}

\subsection{Convergence analysis}\label{sec:convergence}
We next analyze the convergence of the proposed algorithm and establish several properties of the underlying optimization problems.
\begin{proposition}\label{prop:2bglobal}
  When Assumption \ref{as:optimalilty in feas set} holds, the optimal solution $\mY^*,\mY^*_\tau$ of the IK Problem \ref{prob:IK1} is a global minimizer of the relaxed Problem \ref{prob:IK2b}.
\end{proposition}
\begin{proof}
  Under Assumption 1, the optimal value of Problem \ref{prob:IK2b} is zero.
Since $f(\mY,\mY_\tau)\ge 0$ for all feasible $(\mY,\mY_\tau)$, any feasible
point attaining $f=0$ is a global minimizer of Problem \ref{prob:IK2b}.
\end{proof}
For a convex problem like Problem \ref{prob:IK2b}, the set of its optimal solutions is convex, hence connected. This enables the move from $(\mY^0,\mY_\tau^0)$ to $(\mY^*,\mY_\tau^*)$ within $\bar{\cY}$.

The next two propositions analyze the convergence of Algorithm \ref{alg:ik}; their proofs are given in the appendix.
\begin{proposition}\label{prop:local converge eig max}
  When Assumption \ref{as:optimalilty in feas set} holds and $\{\mY^k,\mY^k_\tau\}$ is updated using Algorithm \ref{alg:ik} with Problem \ref{prob:IK3} solved in step \ref{step:update}, it holds that $\{\mY^k,\mY^k_\tau\}\rightarrow \{\tilde{\mY}^*,\tilde{\mY}^*_\tau\}$ as $k\rightarrow +\infty$, 
  where $\tilde{\mY}^*,\tilde{\mY}^*_\tau$ is a local maximizer of Problem \ref{prob:IK2c}.
\end{proposition}

\begin{proposition}\label{prop:local converge eig max adaptive}
    When Problem \ref{prob:IK1} is feasible and $\{\mY^k,\mY^k_\tau\}$ is updated using Algorithm \ref{alg:ik} with Problem \ref{prob:IK4} solved in step \ref{step:update}, as $k$ increases, there exist two situations: either $k\to\infty$ and $p$ is bounded (we find a $c_p$ that works for all the iterations), or $k$ is bounded and $p\to\infty$ (we get stuck on a $k$ because we cannot find a ``stepsize" small enough). In either way, it holds that
    \begin{itemize}
        \item $\{\mY^k,\mY^k_\tau\}\rightarrow\{ \tilde{\mY}^*,\tilde{\mY}^*_\tau\}$;
        \item $\tilde{\mY}^*,\tilde{\mY}^*_\tau\in\bar{\cY}$ is a local maximizer of the sum of the largest eigenvalues \eqref{eq:ik2c_obj}.
    \end{itemize}
\end{proposition}

Finally, proposition \ref{prop:2c optimal is on boundary} connects the optimal solution of the eigenvalue optimization to the optimal solutions of the original IK problem.
\begin{proposition}\label{prop:2c optimal is on boundary}
  Under Assumption \ref{as:optimalilty in feas set}, every globally optimal solution $\mY^*,\mY^*_\tau$ of Problem \ref{prob:IK2c} is also an optimal solution of Problem \ref{prob:IK2a}, and $(\vu^* = g(\mY^*),\capgreek{\tau}^*=g_\tau(\mY^*_{\tau})) \in \cU$.
\end{proposition}
\begin{proof}
  By Assumption~\ref{as:optimalilty in feas set}, the feasible set of Problem~\ref{prob:IK2c} contains
at least one point satisfying the rank-1 constraints. At such a point,
the objective in \eqref{eq:ik2c_obj} equals the sum of the fixed traces. Since no
feasible point can exceed this value by Proposition~\ref{prop:rank_trace}, every global
maximizer must also attain this value. Hence $(\mY^*,\mY_\tau^*)\in\cR_1$.
  Since $(\mY^*,\mY^*_\tau)\in\argmin(f(\mY,\mY_\tau))$, 
  $(\mY^*,\mY^*_\tau)$ is an optimal solution to Problem \ref{prob:IK2a}. Using Proposition \ref{prop:intersect u_bar rank1} we have $(\vu^*,\capgreek{\tau}^*)\in \textrm{image}(\bar{\cY}\intersect\cR_1)=\cU$.
\end{proof}

\section{Obstacle Avoidance}\label{sec:obs_avoid}
In this section, we extend the proposed IK formulation to handle obstacle avoidance by representing the obstacle-free workspace as a union of convex polyhedra and the robot as a collection of spherical collision bodies. We then derive a relaxed convex formulation based on perspective relaxation and show how it can be integrated with the rank-minimization procedure. 

\subsection{Representing the free region using a set of convex polyhedra}
We model the obstacle-free workspace as a union of convex polyhedra 
\begin{equation}
    \cC_c:=\{\cC_i\mid \cC_i:=\{\vp\mid \mA_i\vp+\vb_i\geq0\},\forall i=1,\dots,n_c\}. 
\end{equation}
Such a collection can be generated by a 3-D decomposition algorithm, such as IRIS \cite{deits2015computing}, which inflates polyhedra from seed points. 
We model the robot as a set of spherical collision bodies $\cS_b:=\{\cS_j \mid \cS_j:=\cS(\vc_j,r_j),\ \forall j=1,\dots,n_b\}$ attached to the robot, with center positions $\{\vc_j\}$ and radii $\{r_j\}$. 
Let $\cI_c$ and $\cI_b$ denote the index sets $\{1,\dots,n_c\}$ and $\{1,\dots,n_b\}$, respectively.
The obstacle-avoidance constraint can then be stated as follows:

\begin{equation}\label{eq:obstacl_explicit}
\begin{aligned}
    \text{For every } \cS_j\in\cS_b&\text{, there exists a }\cC_i\in\cC_c \text{ such that}\\
    &\cS_j\subseteq\cC_i,
\end{aligned}
\end{equation}
i.e., each spherical collision body needs to completely fit in one of the collision-free polyhedra.

We define the following sets.
\begin{definition}[Free space of centers]\label{def:Pij}
For each obstacle-free polyhedron $\cC_i$ and collision body $\cS_j=\cS(\vc_j,r_j)$, define
\begin{equation}
  \cP_{c,ij}:=\{\vp\in\real{3}\mid \cS(\vp,r_j)\subseteq \cC_i\}.
\end{equation}

Define the corresponding sets of collections of continuous and binary variables
\begin{equation}
{
  \begin{aligned}
  \cP_c&\defeq\{\{\vp_{ij}\}\in\real{3\times n_c\times n_b}\mid\vp_{ij}\in\cP_{c,ij},\ \forall i\in\cI_c, j\in\cI_b\},\\
\capgreek{\Delta}_{c1}&\defeq\{\{\delta_{ij}\}\in\real{n_c\times n_b}\mid \delta_{ij}\in\{0,1\}\ \forall i\in\cI_c, \forall j\in\cI_b\},\\
\capgreek{\Delta}_{c2}&\defeq\{\{\delta_{ij}\}\in\real{n_c\times n_b}\mid \sum_{i\in\cI_c}\delta_{ij}=1, \forall j\in\cI_b\}.
\end{aligned}
}
\end{equation}
\end{definition}

\begin{lemma}
    For every $i\in\cI_c, j\in\cI_b$, $\cP_{c,ij}$ is a polyhedron defined by
    \begin{equation}
    \begin{aligned}
        \cP_{c,ij}&:=\{\vp\in\real{3}\mid \mA_i\vp+\vb_{ij}\geq0\}\text{, and}\\
        \vb_{ij}&=(\vb_i-\va_i r_j),
    \end{aligned}
    \end{equation}
where $\va_i(l)=\norm{\mA_i(l,:)}$ is the norm of the $l$-th row of $\mA_i$.
\end{lemma}

Using these sets we can equivalently restate \eqref{eq:obstacl_explicit} as
\begin{equation}\label{eq:obstacl_explicit_2}
\begin{aligned}
\text{There exist }\vp_{c}=\{\vp_{ij}\}\in\cP_{c},\ \capgreek{\delta}_c=\{\delta_{ij}\}\in \capgreek{\Delta}_{c1}\intersect\capgreek{\Delta}_{c2} \\
\textrm{ such that }
    \vc_j=\sum_{i\in\cI_c} \vp_{ij}\delta_{ij},\ \forall j\in\cI_b.
\end{aligned}
\end{equation}

Intuitively, the variables $\{\delta_{ij}\}$ are binary \textit{switch} variables. For every collision body $j$, $\delta_{ij}=1$ activates the constraint $\cS(\vp_{ij},r_j)\subseteq\cC_i$ for only \emph{one} convex polyhedron $i$, ignoring the constraints $\cS(\vp_{i'j},r_j)\subseteq\cC_{i'}$ for other polyhedra.



We then define the set of variables $\vz_{c}=\{\vz_{ij}\}$ and consider the \emph{mixed-integer embedding} of the variables $\vp_c,\capgreek{\delta}_c$ in the set
\begin{equation}\label{eq:collision_bilinear}
\begin{aligned}
    \cW_{c}{:=}\big{\{}(\vp_c,\capgreek{\delta}_c,\vz_{c})\mid  \vp_{c}&\in\cP_{c},\capgreek{\delta}_c\in\capgreek{\Delta}_{c1},  \\
    \vz_{ij}&=\vp_{ij}\delta_{ij},\forall i\in\cI_c, j\in\cI_b\big{\}}.
\end{aligned}
\end{equation}

As noted in Remark \ref{rem:fk_any_point}, the translations of the attached collision bodies can be expressed using linear functions $\{l_j(\mY,\mY_\tau)\mid \forall j\in\cI_b\}$:
\begin{equation}
    \vc_j=l_j(\mY,\mY_\tau),\forall j\in\cI_b.
\end{equation}
Let $\mW{:=}\{(\vp_c,\capgreek{\delta}_c,\vz_c)\mid \vp_c,\vz_c\in\real{3\times n_c\times n_b}, \capgreek{\delta}_c\in\real{n_c\times n_b}\}$. We can then formulate the collision-aware IK problem as the following nonconvex optimization problem.

\begin{prob}[5][Obstacle-aware IK]\label{prob:ik_obs}
\begin{subequations}\label{eq:ik_obs}
    \begin{align}
      &\min_{\mY,\mY_\tau,\mW} && f(\mY,\mY_\tau)\label{eq:ik_obs_target_cr}\\
      &\subjectto && \mY,\mY_\tau\in\cY \label{eq:ik obs Y set}\\
      & &&\mW\in\cW_c\label{eq:ik_obs_set}\\
      & && \capgreek{\delta}_c\in\capgreek{\Delta}_{c2}\label{eq:ik_obs_y_sum}\\
      & && l_j(\mY,\mY_\tau)=\sum_i^{n_c}\vz_{ij},\forall j\in\cI_b\label{eq:ik_obs_link_cb_cr}
    \end{align}
  \end{subequations}
\end{prob}
The objective function \eqref{eq:ik_obs_target_cr} minimizes the mismatch between the end-effector and a target pose $(\mR_{ee},\vt_{ee})$. The constraint \eqref{eq:ik obs Y set} encodes the robot kinematic constraints in \eqref{eq:IK1 u}. Constraint \eqref{eq:ik_obs_set} and \eqref{eq:ik_obs_y_sum} enforce that each collision body is assigned to (and thus constrained within) exactly one polyhedron, namely $\vp_{c}\in\cP_{c}$ and $\capgreek{\delta}_c\in \capgreek{\Delta}_{c1}\intersect\capgreek{\Delta}_{c2}$. Finally, constraint \eqref{eq:ik_obs_link_cb_cr} attaches the collision bodies to the robot via forward kinematics.

\subsection{Relaxation of bilinear variables}
\begin{figure}[htb]
    \centering
    \includegraphics[width=0.98\linewidth]{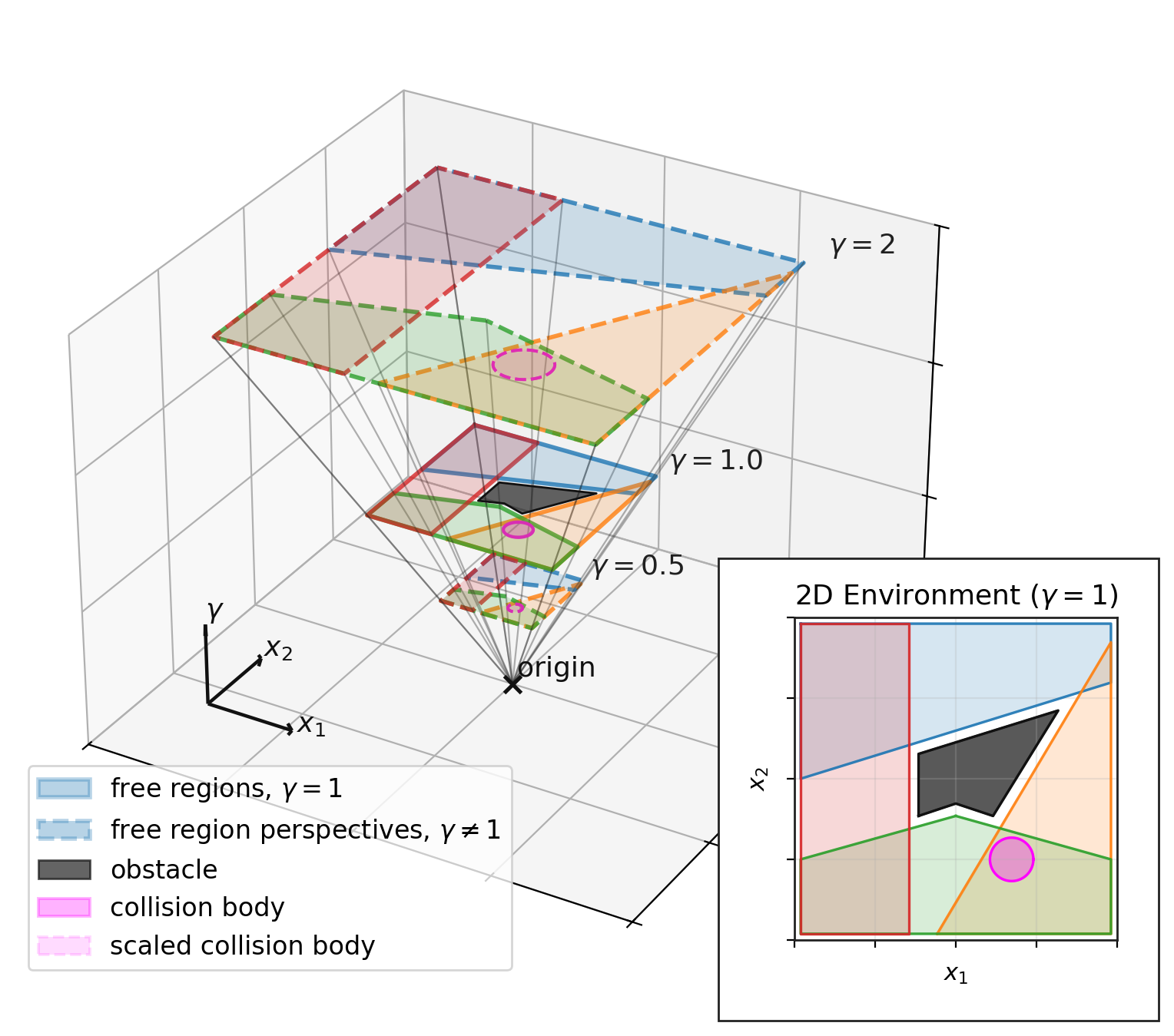}
    \caption{Perspective formulation in 2-D: a collision body (magenta disk) lies in one colored convex polygon $\mA_i\vp+\vb_i\geq0$; lifting with $\gamma$ yields a scaled polytope constraint $\mA_i\vp+\vb_i\gamma\geq0$ that is used to relax the bilinear variables brought by the collision avoidance constraint.}
    \label{fig:perspective vis}
\end{figure}

The constraints in \eqref{eq:collision_bilinear} include integers $\{\delta_{ij}\}$ and the nonlinear $\{\vz_{ij}\}$. To account for this, we introduce a convex relaxation by first relaxing the binary variable $\delta_{ij}$ to a continuous variable and then applying a perspective relaxation technique from \cite{marcucci2024shortest} to convexify the bilinear constraint $\vz_{ij}=\vp_{ij}\delta_{ij}$.
First, we relax $\cW_c$ to a continuous set $\tilde{\cW}_c$, given by
\begin{equation}\label{eq:collision_bilinear_y_cont}
\begin{aligned}
    \tilde{\cW}_{c}{:=}\big{\{}(\vp_c,\capgreek{\delta}_c,\vz_{c})\mid & \vp_{ij}\in\cP_{c,ij},\\
    &\delta_{ij}\in\capgreek{\Delta}_{c1,\textrm{cont}},  \\
    &\vz_{ij}=\vp_{ij}\delta_{ij},\forall i\in\cI_c, j\in\cI_b\big{\}},
\end{aligned}
\end{equation}
where $\capgreek{\Delta}_{c1,\textrm{cont}}$ is defined as
\begin{equation}\label{eq:y_c1_cont}
    \begin{aligned}
        \capgreek{\Delta}_{c1,\textrm{cont}}:=\{\delta\mid \vc\transpose \delta +\vd\geq0\}, \text{ where}\\
        \vc=\bmat{1&-1},\vd=\bmat{0&1}\transpose.
    \end{aligned}
\end{equation}
This is equivalent to enforcing $\delta_{ij}\in[0,1]$, but we keep the half-space form so that the perspective relaxation applies directly. Compared to \eqref{eq:collision_bilinear}, \eqref{eq:collision_bilinear_y_cont} allows every switch variable $\delta_{ij}$ to reside on the continuous interval $[0,1]$.

Next, we define the \emph{perspective} of $\cP_{c,ij}$ (for each $i,j$) as
{
\begin{equation}
    \tilde{\cP}_{c,ij}:=\{(\vp,\gamma)\mid \gamma\geq 0,\ \mA_i\vp+\vb_{ij} \gamma\geq 0\}.
\end{equation}
}
The perspective is a convex set in the variables $(\vp,\gamma)$, and it reduces to the original set $\cP_{c,ij}$ when $\gamma=1$, as visualized in Figure \ref{fig:perspective vis}. 

We use a set-based relaxation of a bilinear constraint, following \cite[section 7.1]{marcucci2024shortest}, which we summarize below.
\begin{lemma}\label{lem:bilinear relax}
    Consider a nonconvex set of the form \begin{equation}
        \cW:=\{(\vp,\capgreek{\delta},\mZ)\mid \vp\in\cP,\ \capgreek{\delta}\in\capgreek{\Delta},\ \mZ=\vp\capgreek{\delta}\transpose\},
    \end{equation}
    where $\cP:=\{\vp\mid \mA\vp+\vb\geq0\},\quad\capgreek{\Delta}:=\{\capgreek{\delta}\mid \vc_l\transpose \capgreek{\delta}+d_l\geq 0,\ \forall l\in\cI\}$ are closed convex sets, and $\cI$ is an index set for the rows of $\capgreek{\delta}\in\capgreek{\Delta}$. Define the perspective $\tilde{\cP}:=\{(\vp,\gamma)\mid \gamma\geq0,\mA\vp+\vb\gamma\geq0\}$. We have 
    \begin{multline}
    \label{eq:bi_convex_relax_lemma}
        \cW\subseteq\bar{\cW}:=\{(\vp,\capgreek{\delta},\mZ)\mid (\mZ\vc_l+d_l\vp,\ \vc\transpose_l\capgreek{\delta}+d_l)\in\tilde{\cP},\\
        \forall l\in\cI\}.
    \end{multline}
\end{lemma}

Let $\cW_{ij}$ denote the nonconvex set 
\begin{equation}
    \cW_{ij}:=\{(\vp,\delta,\vz)\mid \vp\in\cP_{c,ij},\ \delta\in\capgreek{\Delta}_{c1,\textrm{cont}},\ \vz = \vp \delta\}.
\end{equation}
Then, we apply the perspective reformulation by substituting \eqref{eq:y_c1_cont} into \eqref{eq:bi_convex_relax_lemma} of Lemma \ref{lem:bilinear relax} to get $\cW_{ij}\subseteq \bar{\cW}_{ij}$, where
\begin{equation}
\begin{aligned}
    \bar{\cW}_{ij}{:=}\big\{(\vp,\delta,\vz)\mid &(\vz,\delta)\in\tilde{\cP}_{c,ij},\\&(\vp-\vz,1-\delta)\in\tilde{\cP}_{c,ij}\big{\}}.
\end{aligned}
\end{equation}
This yields the relaxed set $\bar{\cW}_c\supseteq \tilde{\cW}_c\supseteq \cW_c$:
\begin{equation}\label{eq:bilinear_relax}
\begin{aligned}
    \bar{\cW}_c{:=}\big{\{}&(\vp_{c},\capgreek{\delta}_{c},\vz_{c})\mid(\vp_{ij},\delta_{ij},\vz_{ij})\in\bar{\cW}_{ij}\ \forall i\in\cI_c, \forall j\in\cI_b\big{\}}.
\end{aligned}
\end{equation}

Observe that $\bar{\cW}_c$ is linear in $\mW$. The relaxed obstacle-avoiding IK problem can be formulated as
\begin{prob}[6][Relaxed obstacle-aware IK]\label{prob:ik_obs_init}
\begin{subequations}\label{eq:ik_obs_relax}
    \begin{align}
      &\min_{\mY,\mY_\tau,\mW} && f(\mY,\mY_\tau)\label{eq:ik_obs_target}\\
      &\subjectto && \mY,\mY_\tau\in\bar{\cY} \label{eq:ik obs Y}\\
      & &&\mW\in\bar{\cW}_c\label{eq:ik_obs_bilinear}\\
      & && \capgreek{\delta}_c\in\capgreek{\Delta}_{c2}\label{eq:ik_obs_relax_y_sum}\\
      & && l_j(\mY,\mY_\tau)=\sum_i^{n_c}\vz_{ij},\forall j\in\cI_b\label{eq:ik_obs_link_cb}
    \end{align}
  \end{subequations}
\end{prob}
Compared to Problem \ref{prob:ik_obs}, Problem \ref{prob:ik_obs_init} relaxes the nonconvex set $\cY$ to $\bar{\cY}$ and relaxes the nonconvex set $\cW_c$ to $\bar{\cW}_c$, 
making Problem \ref{prob:ik_obs_init} convex. 
Relative to Problem \ref{prob:IK2b}, it introduces the additional decision variable $\mW$ and the constraints \eqref{eq:ik_obs_bilinear}--\eqref{eq:ik_obs_link_cb}. 
Solving Problem \ref{prob:ik_obs_init} yields a solution $(\mY^0,\mY_\tau^0,\mW^0)$, which must then be projected onto the nonconvex set defined by $\rank(\mY)=\rank(\mY_\tau)=1$ and $\{\delta_{ij}\}\in\capgreek{\Delta}_{c1}$, while remaining feasible in the constraints of Problem \ref{prob:ik_obs_init}. If the subsequent rank and integrality recovery succeeds, then an exact obstacle-free IK solution is obtained. 
In practice, this projection can be carried out by solving Problems \ref{prob:IK3} and \ref{prob:IK4} (within Algorithm \ref{alg:ik}) with the additional variable $\mW$ and the constraints \eqref{eq:ik_obs_bilinear}--\eqref{eq:ik_obs_link_cb}.

\begin{remark}
  The rank-minimization procedure in Section~\ref{sec:opt} drives $\mY$ and $\mY_\tau$ toward rank-1 solutions; however, it does not generally guarantee integrality of the assignment variables $\delta_{ij}$. To further encourage $\{\delta_{ij}\}$ to approach the binary set $\capgreek{\Delta}_{c1}$, one may additionally maximize the following term in the objectives in \eqref{eq:ik3obj} and \eqref{eq:adaptive rank min obj}
  \begin{equation}\label{eq:ymax term}
    \begin{aligned}
      \delta_{\mathrm{max}} &:= \sum_{j=1}^{n_b} \delta^k_{i_j^{\star}j}, \quad \text{where} \\
      i_j^{\star} &\in \arg\max_{i\in\cI_c} \delta^{k-1}_{ij}.
    \end{aligned}
  \end{equation}
  Maximizing $\delta_{\mathrm{max}}$ promotes the entries of $\{\delta_{ij}\}$ that were largest (among $i\in\cI_c$) in the previous iteration, thereby iteratively pushing the assignment variables toward binary values. Other solutions (e.g., the use of a concave regularizer) could be used. In practice, however, the results in Section~\ref{sec:sawyer collision} show that $\{\delta_{ij}\}$ already becomes nearly binary even without including $\delta_{\mathrm{max}}$. We nevertheless present this term for completeness.
\end{remark}

\section{Incremental Motion Generation}\label{sec:motion}
In this section, we present a local method for generating a continuous sequence of constraint-satisfying robot configurations that tracks a prescribed end-effector path. This is not a complete path-planning approach; rather, it addresses a multi-configuration inverse kinematics problem in which consecutive configurations are required to remain close to one another. We assume that the end-effector positions are given in advance and construct the motion incrementally, generating each configuration based on the preceding ones. A full path-planning or trajectory-optimization framework would require additional components, such as a sampling-based planner and a smoothing procedure, which are beyond the scope of this paper. Our more modest objective here is to show that continuity constraints can be incorporated into the IK problem.

For simplicity, we consider robots with only revolute joints in this section; the prismatic joints can be accounted for with additional variables and constraints.

Consider a three-dimensional path in the robot task space, represented by a sequence of target poses
\begin{equation}
\cT_{\mathrm{path}} := \{ (\mT^k_{ee},\mR^k_{ee}) \mid k = 1, \dots, n_{\mathrm{path}} \}.
\end{equation}

The corresponding joint-space motion is described by a sequence of rotation matrices
\begin{equation}
\cR_{\mathrm{path}} := \{ \mR_k \mid \mR_k=\Mat{\mR_1^k,\dots,\mR_{n_r}^k},k = 1, \dots, n_{\mathrm{path}} \},
\end{equation}
which defines a continuous path in the robot’s configuration space that tracks $\cT_{\mathrm{path}}$.

The path $\cR_{\mathrm{path}}$ is constructed incrementally by solving a sequence of inverse kinematics problems, each initialized from the previous solution. Continuity of the resulting motion is enforced by introducing explicit constraints on the incremental change between successive configurations. We begin by formalizing this continuity constraint.
\subsection{Continuity constraint}
Given a configuration $\mR_k$, the motion planning task requires computing a subsequent configuration $\mR_{k+1}$ that differs only slightly from $\mR_k$. For a prescribed scalar bound $r_b>0$, this continuity requirement can be expressed as
\begin{equation}\label{eq:continuity bound}
\norm{\mR_i^{k+1}-\mR_i^{k}}_F \le r_b,\quad \forall i \in \cV_r .
\end{equation}
Analogous to joint-angle limits, the constraints in \eqref{eq:continuity bound} can be encoded as linear constraints on $\mR_{k+1}$. To this end, we approximate the Frobenius-norm ball by a set of linear inequalities defined by a bounding polyhedron. This yields the linear continuity constraint
\begin{equation}\label{eq:continuity}
\mA_{\mathrm{continuity}}\vec(\mR) \le \vb_{\mathrm{continuity}} .
\end{equation}
\subsection{Incremental construction}
As an initial step, we compute a feasible configuration by solving the inverse kinematics problem for the first pose on the end-effector path $\cT_{\mathrm{path}}$. Subsequently, the configuration sequence $\cR_{\mathrm{path}}$ is generated incrementally by repeatedly solving the IK problem using Algorithm~\ref{alg:ik}, with the continuity constraint \eqref{eq:continuity} incorporated into the optimization problem at step~\ref{step:update}, where $\vec(\mR)$ is replaced by $g(\mY)$.

\section{Performance Improvements}
To improve computational performance, we introduce several additional modifications to the proposed method.

\subsection{Restart}\label{sec:restart}
Since the proposed rank-minimization scheme guarantees only local convergence, the iterates $(\mY^k,\mY_{\tau}^k)$ may converge to a point whose rank remains greater than one. To address this issue, we introduce a line search in Algorithm~\ref{alg:reproj}, which moves $(\mY^k,\mY_{\tau}^k)$ to a different point in the feasible set $\bar{\cY}$, from which the rank minimization procedure can then restart.

\begin{algorithm}[h]
  \caption{Restart}\label{alg:reproj}
  \hspace*{\algorithmicindent} \textbf{Input} $\mY,\mY_\tau$ a small step size $\eta>0$\\
  \hspace*{\algorithmicindent} \textbf{Output} $\mY',\mY'_{\tau}$
  \begin{algorithmic}[1]
    \State For some $t>0$, find matrices $\mM,\mM_\tau$ such that $\mY+t\mM,\mY_\tau+t\mM_\tau\in \bar{\cY}$
    \State $n\gets 0$
    \While{$\mY+(t+(n+1)\eta)\mM$, $\mY_\tau+(t+(n+1)\eta) \mM_\tau$ $\in \bar{\cY}$}
        \State $n\gets n+1$
    \EndWhile
    \State $\mY'\gets \mY+(t+n\eta)\mM$, $\mY'_\tau\gets\mY_\tau+(t+n\eta)\mM_\tau$
  \end{algorithmic}
\end{algorithm}

The boundary of the PSD cone is composed of singular matrices, thus the rank-1 solutions are on the boundary of $\bar{\cY}$.
Algorithm~\ref{alg:reproj} therefore perturbs $(\mY^k,\mY^k_\tau)$ along a feasible direction and continues advancing until the next step would leave $\bar{\cY}$. As a result, the final feasible point $(\mY',\mY'_\tau)$ lies close to the boundary of $\bar{\cY}$. Although this procedure does not guarantee movement toward a rank-1 point, it provides a new initialization within $\bar{\cY}$ from which Algorithm~\ref{alg:ik} can be restarted. In practice, we find that this restart strategy is effective for recovering optimal solutions. Additional details are provided in Section~\ref{sec:baxter}.

\begin{remark}
The $\mM$, $\mM_\tau$ in our current implementation are found by solving a feasibility problem. Other implementations could project random directions, or build directions based on the eigenvectors of the current solution. The current implementation has already provided significant improvements; further improvements are left as future work.
\end{remark}
\subsection{Reducing the number of variables using unit quaternions}\label{sec:quaternion}
In this subsection, we show that the formulation in Section~\ref{sec:so3relax} can be reduced in size by representing rotations with the unit quaternion $\vq=\Mat{q_r & q_x & q_y & q_z}^\trans$.
To this end, we introduce the decision variable
\begin{equation}
    \begin{aligned}
        \mQ = \vq\vq^\trans
        = \Mat{
        q_r^2 & q_rq_x & q_rq_y & q_rq_z\\
        \ast & q_x^2 & q_xq_y & q_xq_z\\
        \ast & \ast & q_y^2 & q_yq_z\\
        \ast & \ast & \ast & q_z^2
        }
        \in \IR^{4\times 4}.
    \end{aligned}
\end{equation}
By construction, $\mQ$ is positive semidefinite, has rank one, and satisfies $\trace(\mQ)=1$. Therefore, the same rank-minimization strategy used in Algorithm~\ref{alg:ik} can be applied to enforce recovery of the rank-1 structure.
A unit quaternion can be converted to a rotation matrix through the map \cite[Section~2.6]{siciliano2009modelling}
\begin{equation}
    \begin{aligned}
        \mR_q {=} \Mat{1-2(q_y^2+q_z^2) & 2(q_xq_y-q_zq_r) &2(q_xq_z+q_yq_r)\\
        2(q_xq_y+q_zq_r)&1-2(q_x^2+q_z^2) &2(q_yq_z-q_xq_r)\\
        2(q_xq_z-q_yq_r) &2(q_yq_z+q_xq_r) &1-2(q_x^2+q_y^2)}
    \end{aligned}
\end{equation}
Since each entry of $\mR_q$ is linear in the entries of $\mQ$, every rotation variable in the IK problem can be replaced by a linear function of $\mQ$. Accordingly, we may assign one matrix $\mQ_i$ to each rotation and reformulate Problem~\ref{prob:IK1} in terms of $(\mQ_i,\mY_\tau)$. This reduces the size of the matrix variable associated with each rotation from $7\times 7$ to $4\times 4$, leading to a significant reduction in the run time of off-the-shelf solvers. 

\subsection{Reducing the number of convex polyhedra for collision avoidance}\label{sec:reduce_convex_poly}
The number of candidate convex polyhedra for each collision body can be reduced through a preprocessing step. For a given collision body $\cS_j$, some of the options $\cC_i$ might be infeasible; if this information is available, we can consider the corresponding binary $\delta_{ij}$ to be zero, and remove it from the optimization problem. We then propose using the certification properties of our convex relaxation to identify some of these cases.

The constraint \eqref{eq:obstacl_explicit} need not be satisfiable for every convex polyhedron $i$. We therefore solve the following convex feasibility problem to identify the admissible polyhedra for each body. For each convex polyhedron $\cC_i$ and collision body $\cS(\vp_j,r_j)$, we solve the following feasibility problem.
\begin{prob}[7][Polyhedron--body feasibility check]\label{prob:ik_obs_feas_check}
\begin{subequations}\label{eq:ik_obs_feas}
    \begin{align}
      &\text{find } &&{\mY,\mY_\tau,\vp}\\
      &\subjectto && f_{t,ee}(\mY,\mY_\tau)=\omat\text{, }f_{r,ee}(\mY,\mY_\tau)=\omat\label{eq:ik_obs_feas_target}\\
      & &&\mY,\mY_\tau\in\bar{\cY} \label{eq:ik obs feas Y}\\
      & && l_j(\mY,\mY_\tau)=\vp_{j}\label{eq:ik_obs_feas_link_cb}\\
      & &&\cS(\vp_{j},r_j)\subseteq\cC_i
    \end{align}
  \end{subequations}
\end{prob}

A feasible solution to Problem~\ref{prob:ik_obs_feas_check} gives a robot configuration that reaches the target pose while placing collision body $\cS_j$ inside $\cC_i$. These constraints can be expressed in the same form as \eqref{eq:IK2b u}. If the problem is infeasible, then, by Remark~\ref{rem:infeasibility certify}, the corresponding IK constraints cannot be satisfied with body $j$ assigned to polyhedron $i$. In that case, the variables $\vp_{ij}$, $\delta_{ij}$, and $\vz_{ij}$, together with their associated constraints, can be removed to reduce the problem size.

In practice, many candidate polyhedra can be eliminated for each collision body, substantially improving the efficiency of the SDPs solved in Algorithm~\ref{alg:ik}. However, the feasibility check itself introduces additional computational overhead, although the pairwise checks can be parallelized. Consequently, its net effect on performance depends on the specific problem instance. 

\begin{remark}
\label{rem:infeasibility if no feasible for j body}
If Problem~\ref{prob:ik_obs_feas_check} is infeasible for every convex polyhedron $i$ for a given collision body $j$, then the IK problem itself is infeasible. Indeed, collision body $j$ cannot be placed in any candidate polyhedron, so \eqref{eq:obstacl_explicit} cannot be satisfied. In this case, either the target pose is unreachable or the chosen convex polyhedra do not provide sufficient free space for that body.
\end{remark}

\section{Results}
\subsection{Dual-Arm Baxter}\label{sec:baxter}
We first evaluate IKSPARK on the dual-arm humanoid robot Baxter by Rethink Robotics. The robot consists of two 7-DOF revolute arms mounted on a fixed torso, with joint angle limits on each arm. To model a collaborative box-carrying task, we represent the two arms as a single closed kinematic chain by rigidly constraining the grippers to lie on a common line at a fixed separation. 
The goal is then to compute robot configurations that realize prescribed end-effector poses.

In contrast to most traditional IK solvers, which are designed primarily for open kinematic chains, our method handles the closed chain directly without decomposing it into separate subtrees. This is achieved by adding the linear constraint described in Remark~\ref{rem:closedloop} to Problem~\ref{prob:IK2b}.
\begin{table}[htb]
    \centering
    \resizebox{0.48\textwidth}{!}{
    \begin{tabular}{ccc}
    \toprule
        Variable Type & Variable Size & Number of Rows (Equality/Inequality)\\
        \midrule
        Rotations&$\mY\in\IR^{7\times7\times n_r}$ & $137/6112$\\
        Quaternions&$\mQ\in \IR^{4\times 4\times n_r}$& $92/6112$\\
        \bottomrule
        \toprule
        Variable Type & $err(\mR_{ee})$ & $err(\mT_{ee})$\\
        \midrule
        Rotations& $0$ & $3.87\cdot 10^{-9}$\\
        Quaternions&$1.62\cdot 10^{-16}$& $9.60\cdot 10^{-9}$\\
        \bottomrule
        
    \end{tabular}
    }
    \caption{Problem sizes and results when solving the Dual-arm Baxter example using two different types of variables}
    \label{tab:problem compare}
\end{table}

\begin{figure}[htb]
  \centering
  \includegraphics[width=0.98\linewidth]{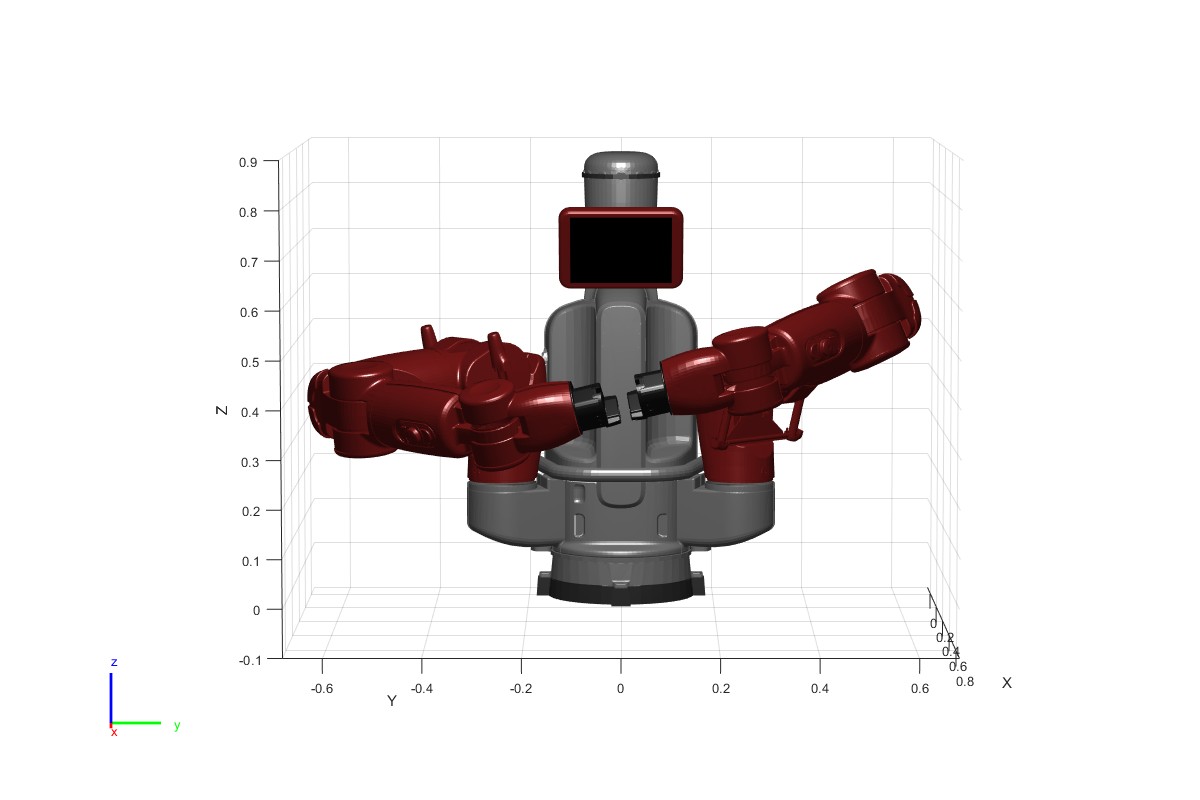}
  \caption{An example posture solved for Baxter, where the two arms are modeled as one closed kinematic chain.}
  \label{fig:baxter1pose}
\end{figure}

\begin{figure*}[htb]
  \subfloat[$\lambda_1$, rotation matrices]{
    \begin{minipage}[c][1\width]{
        0.23\textwidth}
      \centering
      \includegraphics[width=1.1\textwidth]{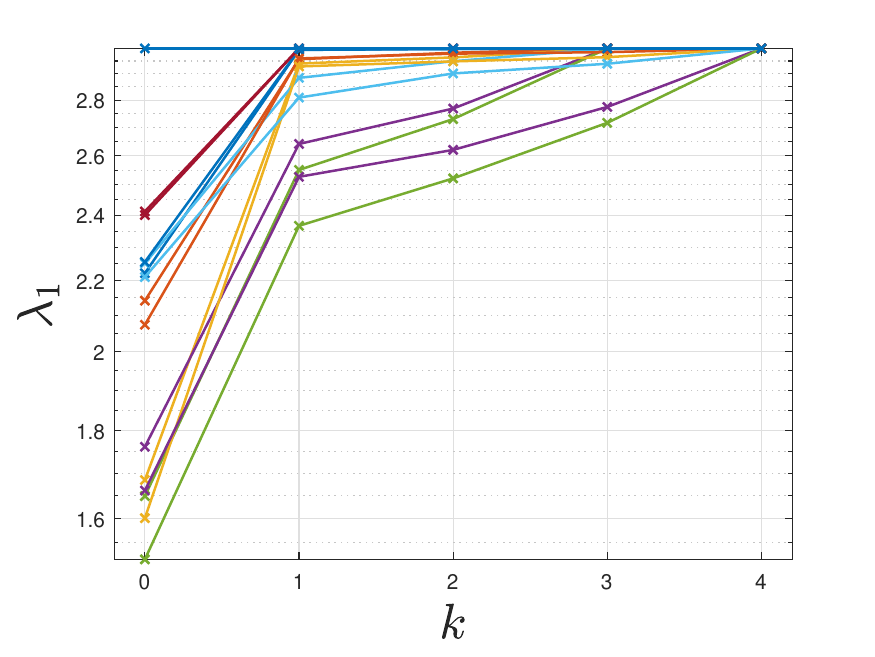}\label{fig:lambda1}
    \end{minipage}}
  \subfloat[$\lambda_1$, quaternions]{
    \begin{minipage}[c][1\width]{
        0.23\textwidth}
      \centering
      \includegraphics[width=1.1\textwidth]{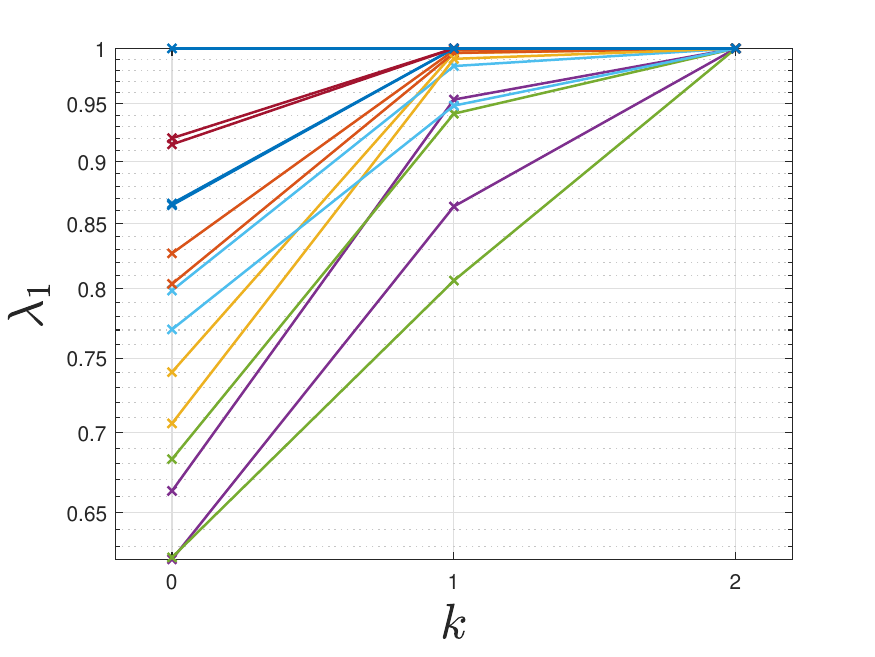}\label{fig:lambda1_quat}
    \end{minipage}}
  \subfloat[Eigenvalues, rotation matrices]{
    \begin{minipage}[c][1\width]{
        0.23\textwidth}
      \centering
      \includegraphics[width=1.1\textwidth]{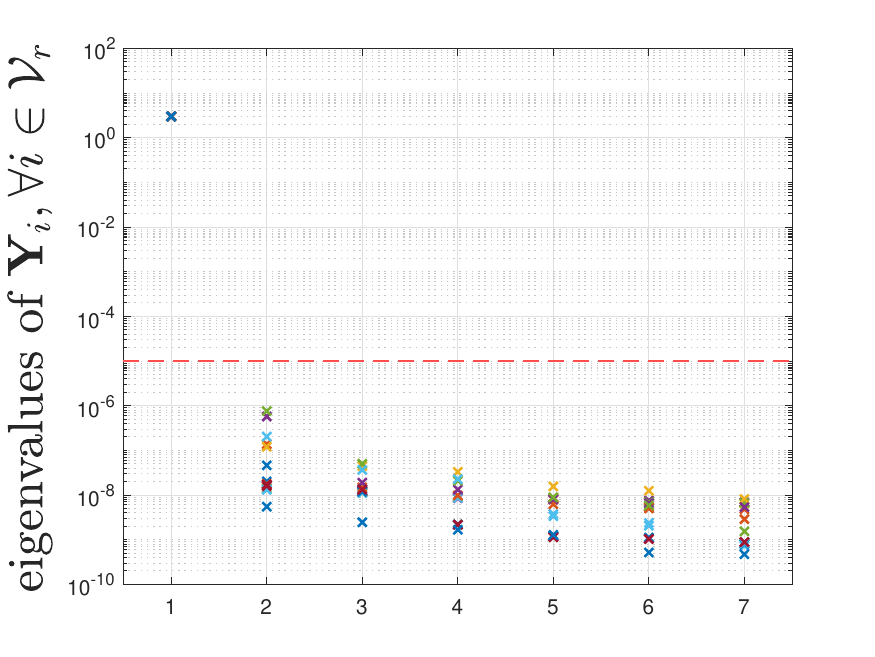}\label{fig:eigens}
    \end{minipage}}
    \subfloat[Eigenvalues, quaternions]{
    \begin{minipage}[c][1\width]{
        0.23\textwidth}
      \centering
      \includegraphics[width=1.1\textwidth]{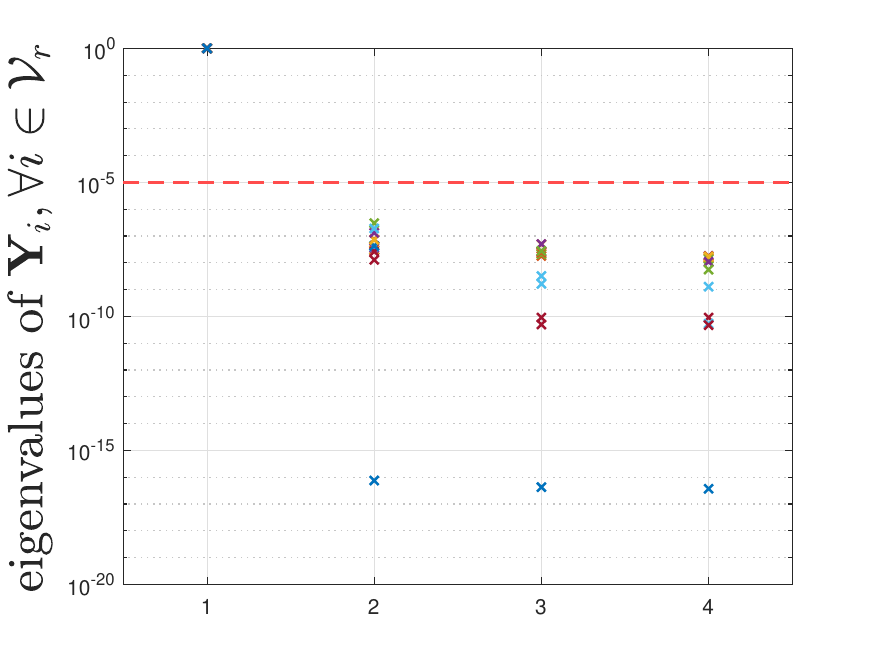}\label{fig:eigens_quat}
    \end{minipage}}
  \caption{Some computational results of the solutions for poses in Fig. \ref{fig:baxter1pose}. Figures \ref{fig:lambda1} and \ref{fig:lambda1_quat} show the changes of the largest eigenvalues $\lambda_1$ of each $\mY^k_{i}$ and $\mQ^k_{i}$ over iteration $k$ when different types of variables are used, where each line corresponds to one matrix. Figures \ref{fig:eigens} and \ref{fig:eigens_quat} compare the eigenvalues of each $\mY_i$ and $\mQ_i$ in the solution, where all eigenvalues except the largest one are below the tolerance $\epsilon_1$ (red dashed line).}
  \label{fig:result_1pose}
\end{figure*}

\begin{table*}[htb]
  \centering
    \resizebox{0.98\textwidth}{!}{
  \begin{tabular}{c cc cc cc}
    \toprule
    Method &Success ct., \%$^\dagger$ &Avg. time &$\overline{err}(\mR_{ee})$ & $\overline{err}(\mT_{ee})$ &$\max(\norm{\mR_{i}-P(\mR_{i})}_F)$ & maximal $e_2$\\ [0.5ex] 
    \midrule
    IKSPARK(rotations) & 376(+22)/431, 87.2\%(+5.1\%)=92.3\% & 0.6399 s& $1.7276\cdot 10^{-17}$&$2.6958\cdot 10^{-8}$&$9.0475\cdot 10^{-6}$ & $2.2130\cdot 10^{-6}$\\
    IKSPARK(quaternions) & 382(+18)/431, 88.6\%(+4.2\%)=92.8\% & 0.2889 s & $1.4879\cdot 10^{-16}$ & $6.8376\cdot 10^{-9}$& $2.1187\cdot 10^{-5}$ & $1.5729\cdot 10^{-8}$\\
    BFGS &388/431, 90.0\%&0.1966 s &$1.3452\cdot10^{-8}$&$5.9359\cdot 10^{-9}$& - &- \\
    Drake IK (10 attempts) & 362/431, 84.0\%&0.0448 s & $3.8046\cdot 10^{-5}$ & $2.9066\cdot 10^{-6}$ & - &- \\
    Drake IK (1 attempt) & 32/431, 7.4\%&0.0084 s & $2.2814\cdot 10^{-5}$ & $2.2391\cdot 10^{-6}$ & - &- \\
    \bottomrule
  \end{tabular}
  }
  \footnotesize{$^\dagger$The denominator is the total number of goals subtracted by the number of cases where infeasibility is certified. The additional success rate of IKSPARK with the restart process is shown in parentheses.}
\caption{Performance of IKSPARK on 500 IK goals without collision constraints, compared with two numerical solvers.}\label{tab:errors}
\end{table*}

\begin{figure}[htb]
    \centering
    \includegraphics[width=0.88\linewidth]{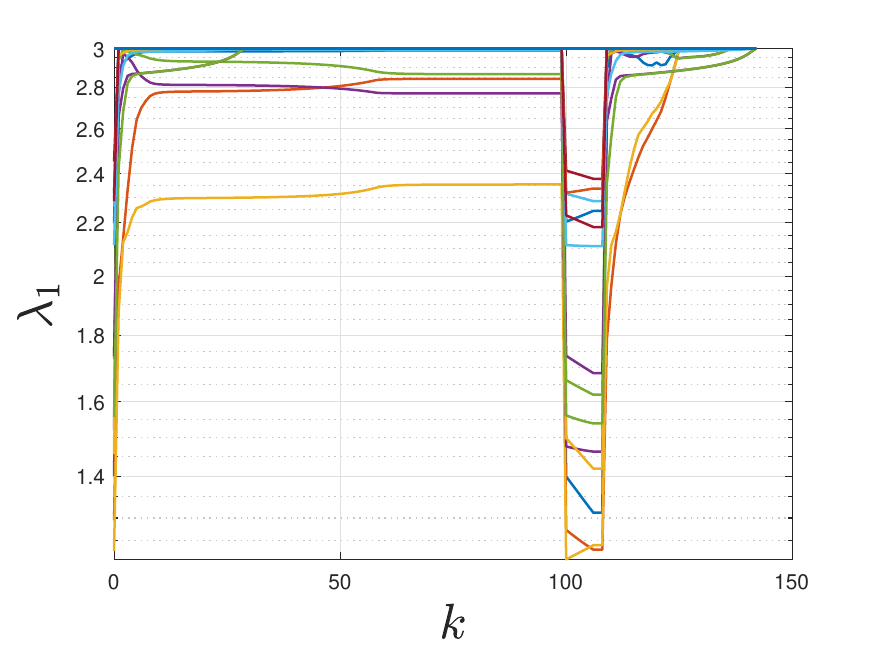}
    \caption{The change of the largest eigenvalues where a sub-optimal solution from Algorithm \ref{alg:ik} ($k=1,\dots,100$) is projected to another point in the relaxed set using Algorithm \ref{alg:reproj} ($k =101,\dots,108$) and restarted again with Algorithm \ref{alg:ik} ($k=109,\dots,142$).}
    \label{fig:restart}
\end{figure}

\begin{figure}[htb]
  \centering
  \includegraphics[width=0.78\linewidth, trim=0cm 3cm 0cm 3cm, clip]{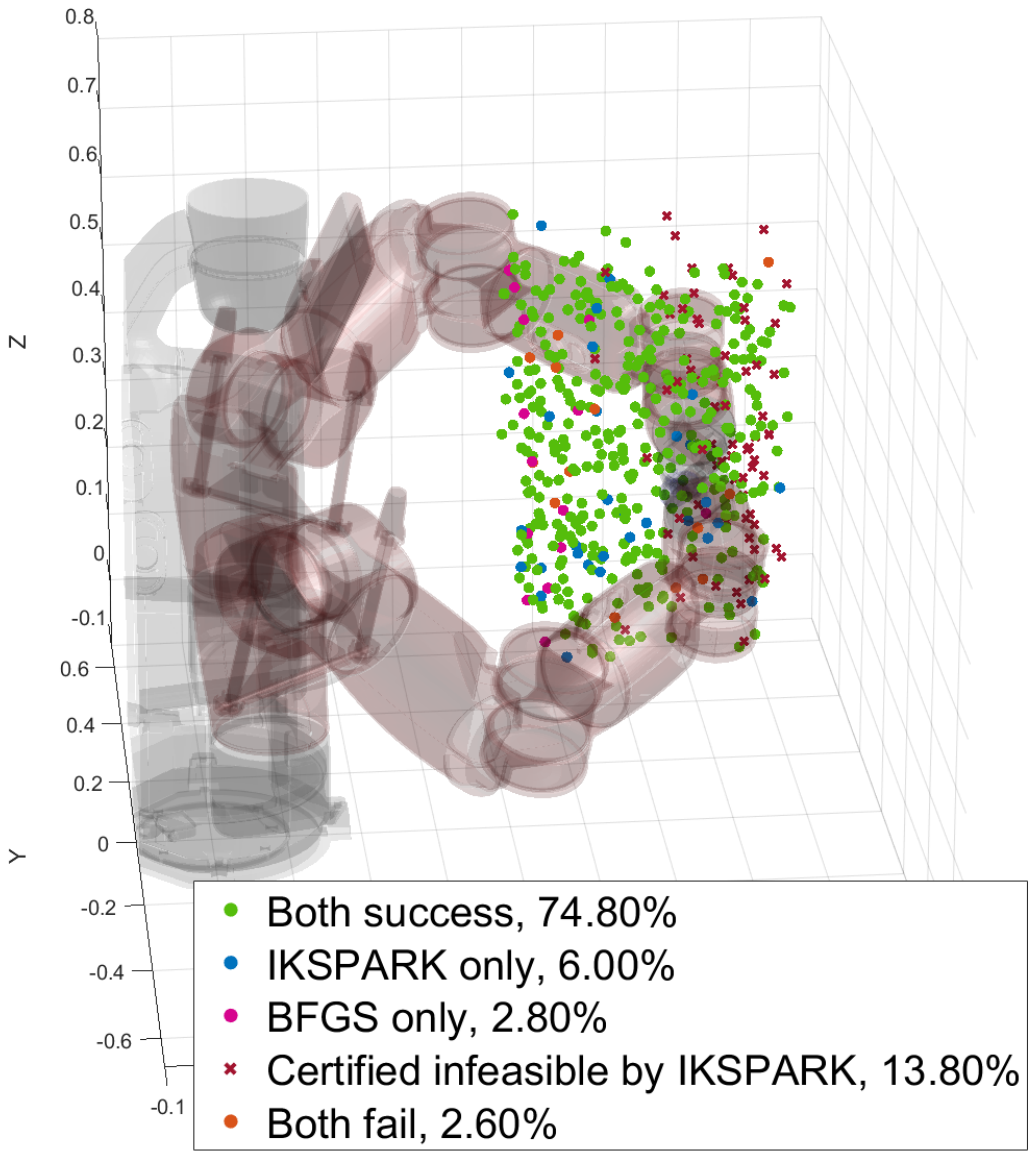}
  \caption{Performance of the proposed method, taking the better result across the two variable formulations, and a traditional BFGS gradient-projection method on 500 end-effector poses for the dual-arm Baxter robot.}
  \label{fig:compare3d}
\end{figure}

\begin{figure}[htb]
    \centering
    \includegraphics[width=0.98\linewidth]{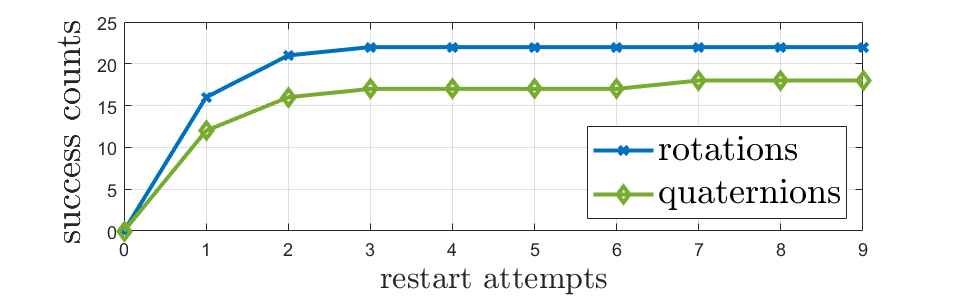}
    \caption{The restart (Algorithm \ref{alg:reproj}) is applied to the failed results in the 500 poses. The figure shows the cumulative number of successes after different numbers of attempts. The majority of the gains are obtained with three restarts or less.}
    \label{fig:restart success rate}
\end{figure}

We test IKSPARK through two simulations, one using the rotation formulation $\mY$ and one using the quaternion formulation $\mQ$. Figure \ref{fig:baxter1pose} shows a solution to a given $\mT_{goal}$ and $\mR_{goal}$ using IKSPARK. In this example, the total number of free links is $n_r=15$, which defines the size of the problem shown in Table \ref{tab:problem compare}. Problem \ref{prob:IK3} is solved in step \ref{step:update} of Algorithm \ref{alg:ik}. 
The end-effector is set as the midpoint of the two grippers and is treated as a link of the robot assigned with the reference frame $\{\mR_{ee},\mT_{ee}\}$. The errors of the end-effector pose, $err(\mR_{ee})=\norm{\vec(\mR_{ee})-\vec(\mR_{goal})}_2$ and $err(\mT_{ee})=\norm{\mT_{ee}-\mT_{goal}}_2$ are compared for the two variable selections in Table \ref{tab:problem compare}.
We verified that all the poses satisfy the imposed constraints in Problem \ref{prob:IK1} along with the translation relation \eqref{eq:translation relation}. Figure \ref{fig:result_1pose} shows some results regarding the computation process where \ref{fig:lambda1} and \ref{fig:lambda1_quat} show the change of the largest eigenvalue, $\lambda_1$, of each $\mY^k_{i}$ and $\mQ^k_{i}$ during the rank minimization process.
We observed that values of $\lambda_1$ increase iteratively, eventually reaching the maximum value given by the trace constraint. Figures \ref{fig:eigens} and 
\ref{fig:eigens_quat} present the eigenvalues of every $\mY_{i}$ in the final solution, where all eigenvalues except $\lambda_1$ are below the tolerance $\epsilon_1$. This shows that each $\mY_{i}$ and $\mQ_i$ in the solution is approximately a rank-1 matrix. These results show that in this example, the solver successfully solves the IK problem.

The proposed rank-minimization algorithm guarantees only local convergence. In practice, the solver may therefore fail to recover a rank-1 solution and instead stall at matrices of higher rank. In such cases, the restart procedure \ref{sec:restart} can be used to improve the result.
To illustrate this behavior, we take a suboptimal solution at which Algorithm~\ref{alg:ik} terminates because it cannot find a suitable update direction ($\norm{\mU^k}_F<\epsilon_2$), and use it as input to Algorithm~\ref{alg:reproj}, thereby obtaining a new point in the relaxed feasible set. We then reapply the rank-minimization procedure in Steps~\ref{step:startwhile}--\ref{step:endwhile} of Algorithm~\ref{alg:ik}. The evolution of the eigenvalues during this process is shown in Fig.~\ref{fig:restart}. Initially, rank minimization converges to a collection of higher-rank matrices. After the restart step, 
the solver is able to continue and recover a rank-1 solution.

To test the performance of IKSPARK on multiple different targets, we implement it on a set of random end-effector poses. We build this set by randomly sampling 500 points in a space $\mT_{goal} = \Mat{x,y,z}^\trans\in \cT_{goal}$, where $x\in [0.4,0.75]$, $y\in [-0.2,0.2]$, and $z\in [0.2,0.7]$. For each point, we assign a randomly generated orientation $\mR_{goal}=\mR_z(\alpha)\mR_y(\beta)\mR_x(\gamma)\in\cR_{goal}$, where $\alpha\in [0,\pi/2]$, $\beta\in [0,\pi]$, and $\gamma\in [-\pi/2,0]$. These poses are selected based on the mutual reachable space of the arms, but are not guaranteed to have feasible IK solutions. 
For comparison, we evaluated two widely used IK solvers on the same problem set: the BFGS-based IK solver provided by MATLAB’s \texttt{generalizedInverseKinematics} class and Drake’s \texttt{InverseKinematics} module in Python.
It is important to note that the BFGS solver can only handle open kinematic chains. As a result, for a shared end-effector pose of the two arms, it must be applied separately to each arm. In addition, both the BFGS and Drake solvers require an initial guess for every query, whereas our method does not. In our simulations, the initial guess was set to the zero joint configuration for both BFGS and Drake.
We use SNOPT to solve the nonlinear optimization problems in the Drake IK solver, and we allow 10 attempts to restart the solver with different initial guesses if it fails to find a solution.
The \cite{mosek} SDP solver is employed to solve the SDP problems within our method. 
To improve performance, the restart algorithm is applied to the cases where IKSPARK fails to find a solution; in such cases we give 10 restart attempts. For this test, we do not consider the collisions.

The results of IKSPARK and BFGS are visualized and compared in Fig.~\ref{fig:compare3d}, where the sampled target poses are colored according to which methods succeed. For our approach, a problem is counted as successfully solved if a solution is found using either of the two variable formulations.
Moreover, as noted in Remark~\ref{rem:infeasibility certify}, infeasibility of Problem~\ref{prob:IK2b} certifies infeasibility of the corresponding IK problem. Using this criterion, we certify that 69 of the 500 sampled problems are infeasible, which, as expected, none of the methods succeeds in such cases.
The solvers are compared for their performance in Table \ref{tab:errors}, including success rates and for successful solutions: the average time covering only the time consumed in the SDP solver and the average errors of the end-effector poses. On average, the rank minimization process requires 7.67 iterations when using the rotation variables and 3.64 iterations under the quaternion formulation. Some other results of our method are also listed. This includes maximal $\norm{\mR_{i}-P(\mR_{i})}_F$, which is the maximal value of all Frobenius norms of the difference between computed $\mR_{i}$ and its projection $P(\mR_{i})$ on $\SO{3}$ (see \cite{umeyama1991least}), for all $i\in\cV_r$ in the successful solutions. This shows how close to the $\SO{3}$ manifold the computed rotations are. 
Another result is the maximal value among all of the second-largest eigenvalues of every $\mY_{i}$ in the successful solutions. This shows how close to a rank-1 matrix each $\mY_{i}$ is. 
BFGS and Drake IK use minimal, non-convex parametrizations of $\SO{3}$, hence the last two metrics are not applicable.
Figure \ref{fig:restart success rate} shows how the restart algorithm improves the success rate with increasing number of restart attempts for failed cases.

Table~\ref{tab:errors} shows that IKSPARK achieves a higher success rate and comparable solution accuracy relative to those of the BFGS solver, although its runtime is generally higher. When the quaternion-based formulation is used, the reduced variable size leads to a noticeable speedup, making the runtime comparable to that of BFGS. The last two columns of Table~\ref{tab:errors} further show that our method consistently recovers rank-1 solutions and rotation matrices that lie on $\SO{3}$, thereby validating the proposed rank-minimization approach. Figure~\ref{fig:compare3d} also demonstrates that the proposed method can solve instances for which the benchmark solver fails.
In general, the failure cases across the compared methods happen at the edges of the feasible workspace, but we did not find any discernible pattern distinguishing IKSPARK and the benchmark methods.

We then test the solver on a different problem set obtained by translating all $\cT_{goal}$ poses by $2$ in the $x$-axis direction. By construction, all these poses are beyond the reach of the robot end-effector. We use our solver to test the infeasibility of these problems 
consisting of the relaxed constraints $(\mY,\mY_\tau)\in\bar{\cY}$ together with 
$f_{t,ee}=\omat$ and $f_{r,ee} = \omat$. If this relaxed feasibility problem is infeasible, then the original IK problem is infeasible for that target. As a result, the solver detects infeasibility for all of the $500$ target poses, matching our theoretical expectations.

We evaluate the adaptive rank-minimization scheme on the same set of infeasible poses
by applying Algorithm~\ref{alg:ik} with  Problem~\ref{prob:IK4} in Step~\ref{step:update}. The results are reported in Table~\ref{tab:alt results}. We first test several choices in which $c^{(k)}$ is kept constant across iterations. Smaller values of $c^{(k)}$ can lead to faster convergence, but they may also cause the solver to fail to find an iterate satisfying \eqref{eq:min c lin} at some step. Larger values of $c^{(k)}$, by contrast, generally require more iterations and tend to produce a larger increase in the cost. They also more frequently lead to termination with only limited improvement. Finally, we consider the adaptive update rule in \eqref{eq:adaptive c}. This strategy improves the success rate, although it typically results in a larger~$\Delta \bar{f}$, which is the average cost increase during the iterative rank minimization. Figure~\ref{fig:adaptive_solns} visualizes the found configuration of a target instance.

\begin{table}[htb]
  \centering
\resizebox{0.48\textwidth}{!}{
  \begin{tabular}{cccccc}
    \toprule
    Variable Type &$c^{(k)}$ &Success \% &Avg. time &Avg. iterations & $\Delta \bar{f}$\\ [0.5ex]
    \midrule
    Rotations & 0.2&29.6\%&0.6429(s) &7.07 & +0.0391\\
    \midrule
    Quaternions & 0.05&62.2\%&0.3623(s) &4.41 & +0.0342\\
    \bottomrule
    \toprule
    Rotations & 0.4&55.0\%&0.9548(s) &10.95 & +0.0809\\
    \midrule
    Quaternions & 0.1&80.20\%&0.4020(s) &5.07 & +0.0521\\
    \bottomrule
    \toprule
    Rotations & 0.8&28.8\%&3.6441(s) &47.36 & +0.1404\\
    \midrule
    Quaternions & 0.2&67.6\%&0.4886(s) &7.06 & +0.0607\\
    \bottomrule
    \toprule
    Rotations & Adaptive &99.6\%&1.6226(s) &7.83 & +0.1239\\
    \midrule
    Quaternions & Adaptive &97.0\%&0.5170(s) &5.33 & +0.0905\\
    \bottomrule
  \end{tabular}
  }
  \caption{Performance of IKSPARK solving closest matching configurations for 500 out-of-reach targets}\label{tab:alt results}
\end{table}

\begin{figure}[htb]
  \centering
  \includegraphics[width=0.98\linewidth]{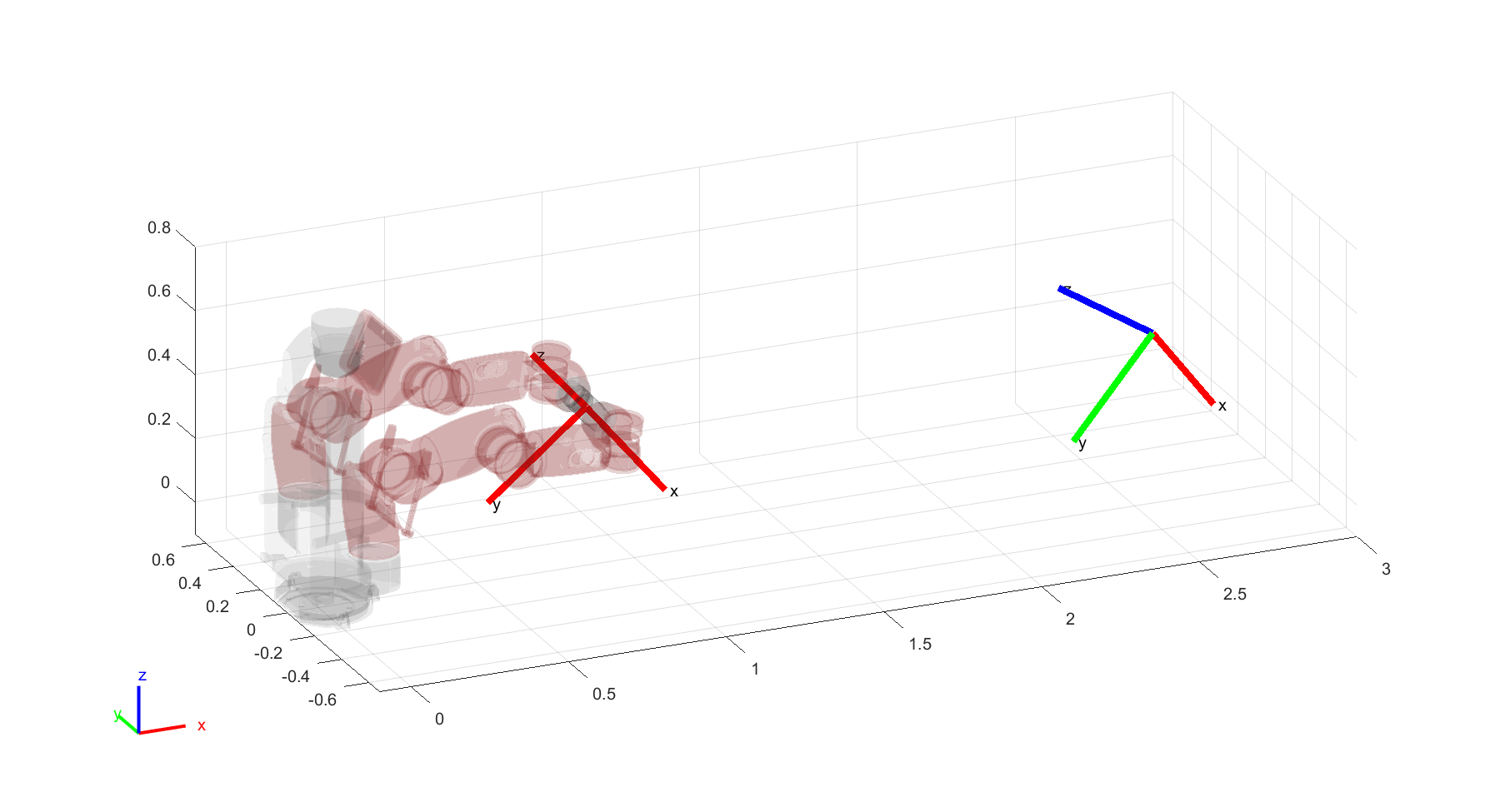}
  \caption{Visualization of one instance from Table \ref{tab:alt results}. The configuration is found under the settings of \(c^{(k)}=0.2\) and the variable type of rotations. The associated end-effector pose, shown in red, is displayed together with the target pose on the right. The resulting cost (squared offset between the end-effector and the target) is \(3.7044\).}
  \label{fig:adaptive_solns}
\end{figure}

\subsection{Stewart platform with prismatic joints}
We show in this subsection that our solver can find solutions for robots with complex closed kinematic chains and prismatic joints. As an example, we consider a classical Stewart platform \cite{stewart1965platform},  a parallel robot with two rigid bodies connected by 6 legs equipped with 6 actuated prismatic joints, as shown in Figure~\ref{fig:stewart platform}. The reference frame for the end-effector is attached rigidly in the center of the top surface, and the limits $\tau_{l}$ and $\tau_u$ are set as $0.0001$ and $1$. 
The legs are attached to the body through non-actuated spherical joints. 

\begin{figure}[htb]
   \centering
   \includegraphics[width=0.88\linewidth]{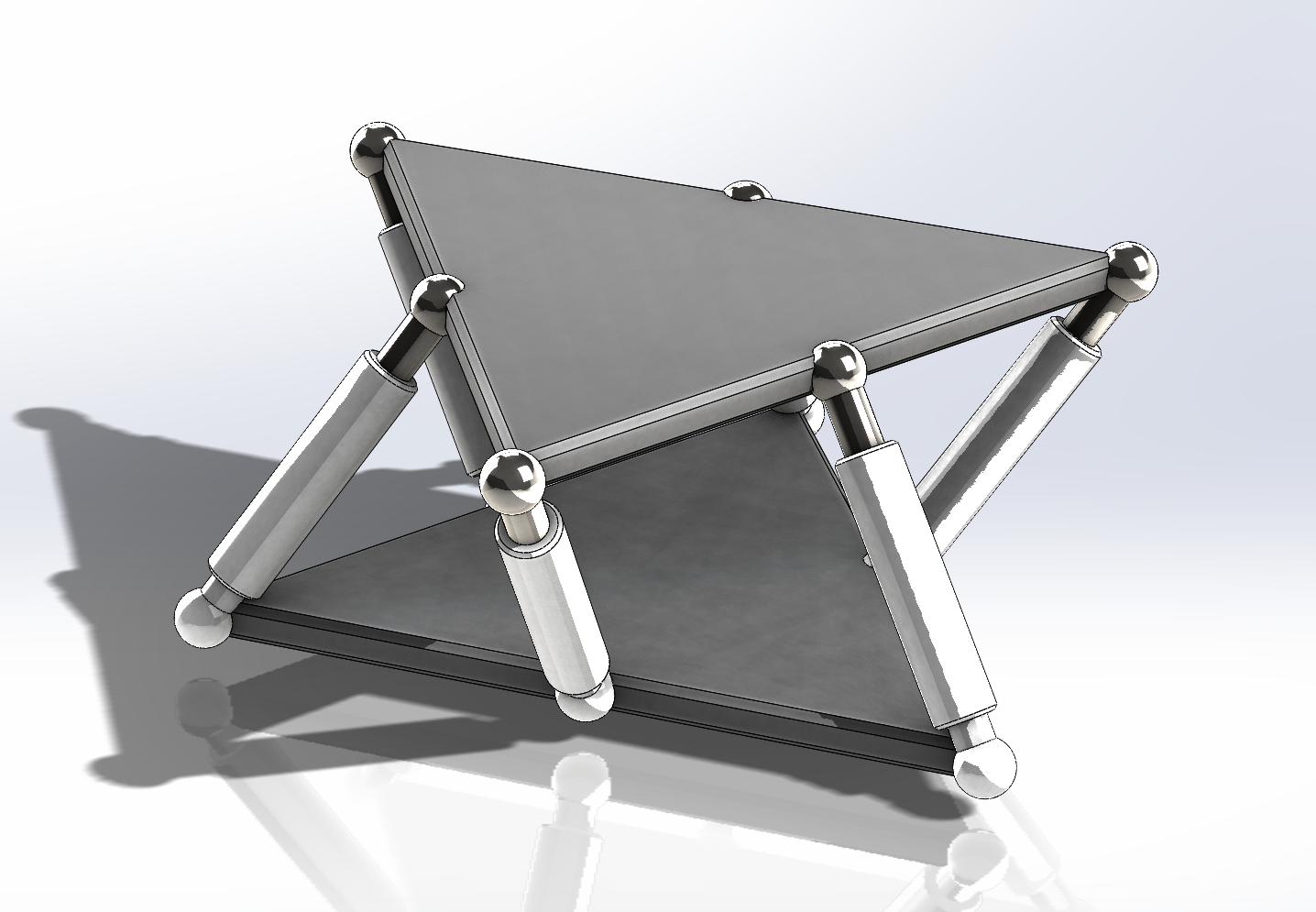}
   \caption{A Stewart platform}
   \label{fig:stewart platform}
\end{figure}
\begin{table*}[htb]
  \centering
  \begin{tabular}{ccccccc}
    \toprule
    Leg\# & 1 &2 &3 &4 &5 & 6\\ [0.5ex]
    \midrule
    Error& $0.073\cdot 10^{-6}$&$-0.713\cdot 10^{-6}$ &$0.267\cdot 10^{-6}$&$-1.048\cdot 10^{-6}$ &$-1.242\cdot 10^{-6}$ & $0.089\cdot 10^{-6}$\\
    \bottomrule
  \end{tabular}
  \caption{The Stewart platform in \cite{dietmaier1998stewart} is re-modeled with its 6 fixed-length legs replaced with prismatic joints. The above lists the average difference between the solved prismatic joint extensions and the fixed leg lengths when the end-effector poses of 40 known configurations are used as targets.}\label{tab:leg difference}
\end{table*}

We first select the shape of the robot with a geometrical parameter from \cite{griffis1993method}, which is used as an example in \cite{porta2009linear}, see Table \ref{tab:stewart geo} (left side) for details. We run IKSPARK for 100 end-effector poses randomly sampled in the space
    $\{\Mat{x,y,z}^\trans\mid x\in [0.2,0.8],\quad y\in [-0.3,0.3],\quad z\in [0.8,1.05]\}$ for translations, and $\{\mR_z(\alpha)\mR_y(\beta)\mR_x(\gamma)\mid \alpha\in [-\frac{\pi}{3},\frac{\pi}{3}],\beta\in [-\frac{\pi}{12},\frac{\pi}{12}],\gamma\in [-\frac{\pi}{12},\frac{\pi}{12}]\}$ for the rotations.
    Our solver is able to find solutions for all of the poses, with an average execution time of $0.2998$ seconds and average number of $15.83$ iterations.

We perform another test with another shape for the Stewart platform given in \cite{dietmaier1998stewart}, the parameters of which are presented in Table \ref{tab:stewart geo} (right side). This robot is known to have 40 different configurations for a given set of $6$ leg lengths. With the legs treated as prismatic joints, and using the 40 corresponding end-effector poses as input to IKSPARK, we expect that the extensions found for the prismatic joints match the given fixed values. The average difference between the extensions of the solved prismatic joints and the fixed leg length is in the order of numerical tolerances, as shown in Table \ref{tab:leg difference}, thus validating the precision of the IK solver.

\subsection{Sawyer 7R robot arm in a cluttered environment}\label{sec:sawyer collision}
\begin{figure*}[h]
   \centering
   \subfloat[Configuration solved using IKSPARK]{
    \begin{minipage}{
        0.49\textwidth}
      \centering
      \includegraphics[width=0.98\linewidth]{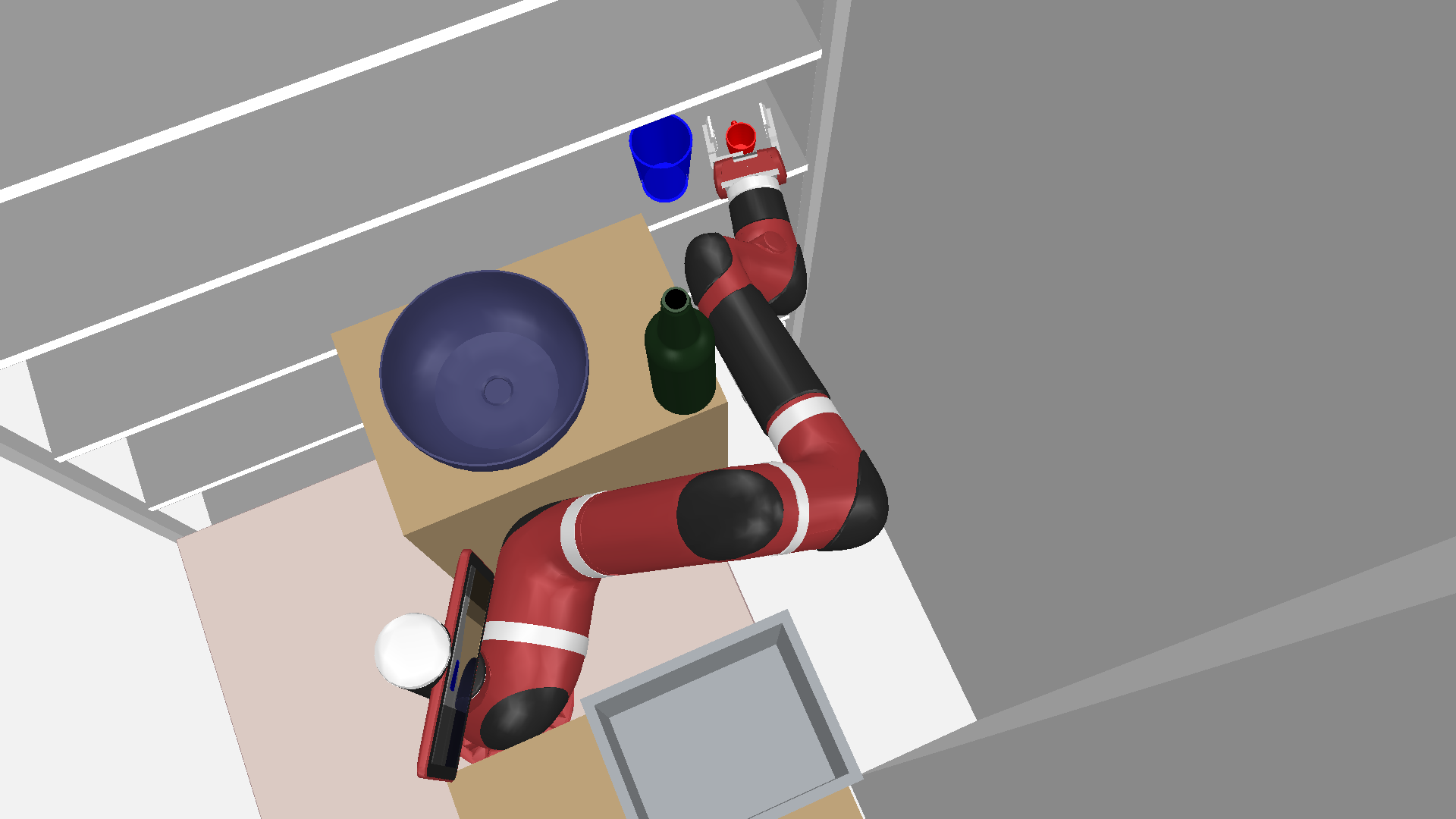}\label{fig:sawyer_ik}
    \end{minipage}}
  \subfloat[Collision bodies and free convex polyhedra]{
    \begin{minipage}{
        0.49\textwidth}
      \centering
     \includegraphics[width=0.98\linewidth]{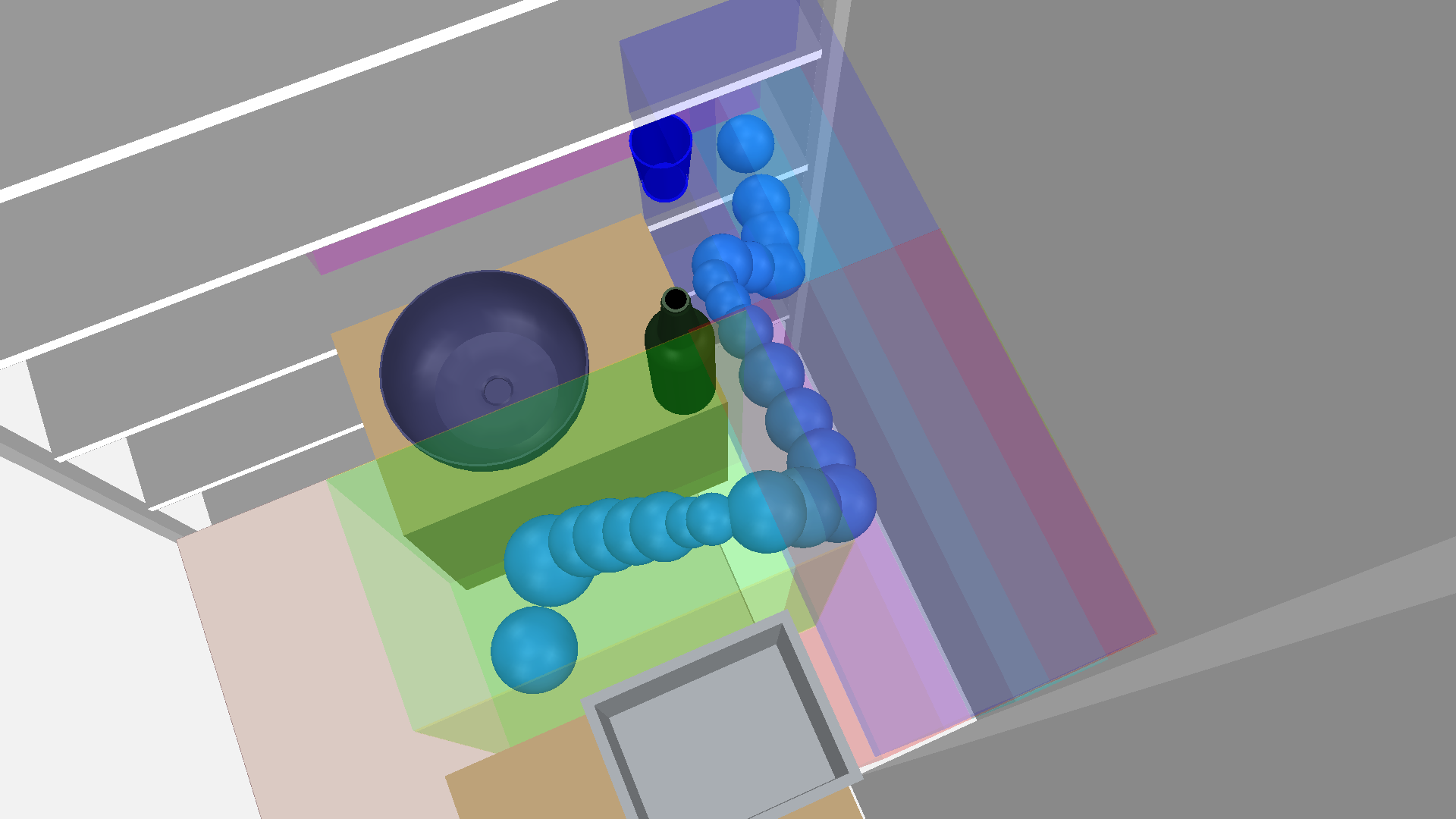}\label{fig:sawyer_ik_set_view}
    \end{minipage}}
   \caption{Sawyer grasping a mug from the shelf.}
   \label{fig:sawyer_obs_ik}
\end{figure*}

\begin{figure}[h]
  \centering
  \includegraphics[width=0.95\linewidth]{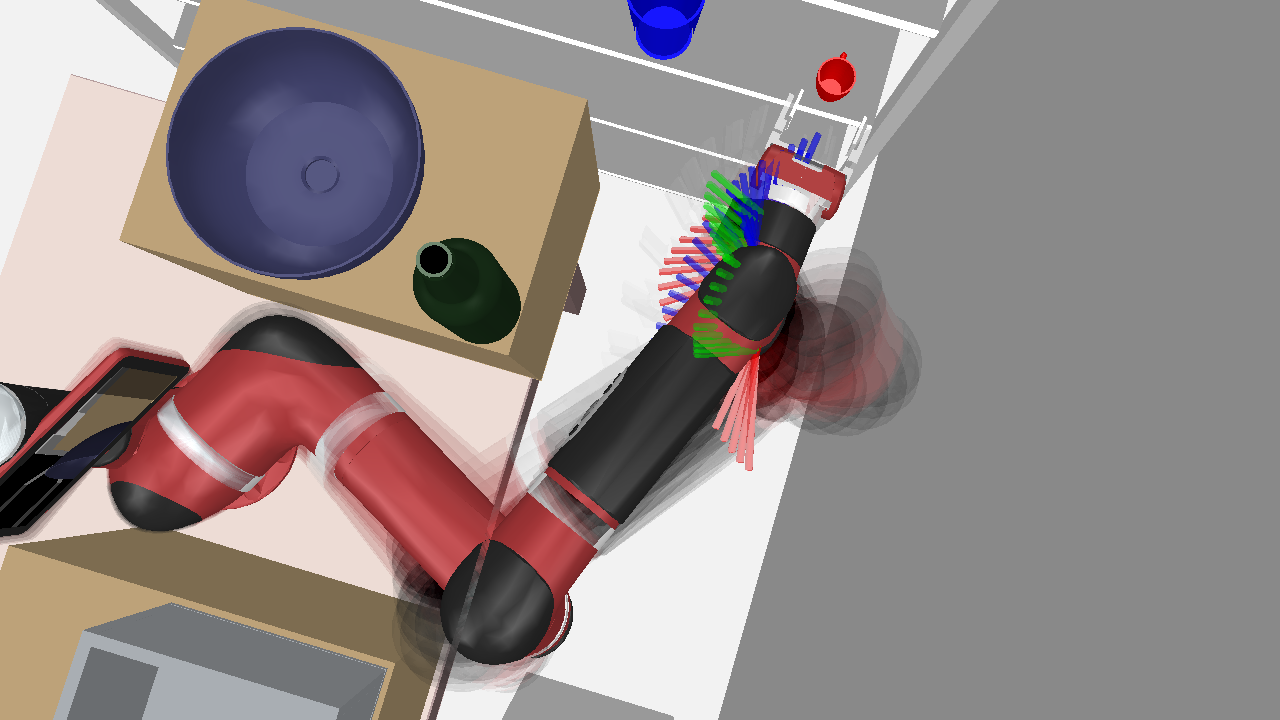}
  \caption{A motion of the Sawyer arm of the mug grasping task solved by IKSPARK.}
  \label{fig:sawyer_motion}
\end{figure}


\subsubsection{Grasping task demonstration}

We demonstrate that IKSPARK can compute obstacle-avoiding configurations in cluttered scenes using the shelf-grasping task in Figure \ref{fig:sawyer_obs_ik}, where a mug must be grasped from a shelf. The platform is a 7-DOF Rethink Robotics Sawyer arm.
Figure \ref{fig:sawyer_ik} shows a feasible posture in which the arm threads through a narrow passage to reach the mug. As illustrated in Figure \ref{fig:sawyer_ik_set_view}, collision bodies are modeled as spheres rigidly attached to the robot, and the free space is represented as the union of six overlapping cuboids. In this example, the mug is treated as part of the free space, while an additional collision body is placed on the gripper to avoid contact with the adjacent blue cup and shelf boards.

In a modified environment with slightly more free space, we use linear interpolation to obtain a trajectory of end-effector poses (indicated by the red, green, and blue axes) from the initial configuration to the final grasping configuration. 
We use IKSPARK to compute complete joint configurations from the end-effector poses while enforcing motion continuity as described in Section \ref{sec:motion}. The result is shown in Figure \ref{fig:sawyer_motion}. The resulting motion traverses the narrow passage without colliding with any obstacles.

\subsubsection{Randomly generated obstacle-aware IK problems}
IKSPARK is tested in two different sets of obstacle-rich environments. We first discuss the environments, then show how optimization problems are built followed by results.
\paragraph{Random environments setup.}
The first set of environments contains randomly generated obstacles and end-effector targets. Specifically, for each environment, we first sample a random Sawyer configuration and use its end-effector pose as the IK target. We then place 20--30 random obstacles while ensuring that none intersect the robot. This construction guarantees that each obstacle-aware IK instance is feasible.
An example random environment is shown in Figure \ref{fig:sawyer_random_env}. 

\paragraph{Workcell environment with random poses.}
The second set of environments includes fixed obstacle placements but random end-effector target poses, simulating the manipulation tasks of the arm in a ``workcell'', as shown in Figure \ref{fig:sawyer_fixed_env}. We generate 600 feasible end-effector target poses across the workcell interior by randomly sampling configurations of Sawyer and choosing the end-effector poses of the ones that do not collide with the obstacles.

\paragraph{Collision bodies.}
To model the robot collision geometry, we rigidly attach different numbers of spheres to the robot, yielding collision models with different levels of geometric fidelity. Figure \ref{fig:collision_models} visualizes these models. 
Higher-density models provide more accurate geometric approximations, but also introduce more variables and greater computational cost. 
In practice, the user can choose among these models to balance accuracy and efficiency. 

\paragraph{Decomposition strategies.} 
We provide different strategies for decomposing the free space for the two sets of environments. 
For each of the random-obstacle environments, we decompose the free space into $n_c$ convex cuboids by applying an inflation algorithm to randomly sampled seed points. For this step, we consider either IRIS or a customized cuboid inflation strategy (see \ref{appendix:inflation} for details). After inflation, duplicated polyhedra are removed. We then reduce the resulting problem size using the preprocessing method described in Section \ref{sec:reduce_convex_poly}.
If, during preprocessing, the condition in Remark \ref{rem:infeasibility if no feasible for j body} is detected, we declare the problem unsuccessfully built and proceed to the next environment. In this case, the generated convex polyhedra do not provide sufficient coverage of the free space. Coverage can be improved, for example, by increasing the number of sampled seed points or by biasing the seed distribution toward the robot.
For the latter strategy, we first solve an IK problem without considering obstacles and then sample seed points within a radius $r_s$ of the center of each collision body in the resulting configuration. This encourages the generated convex polyhedra to cover the region around the robot and, more specifically, the manipulator workspace. We compare the resulting sampling and inflation strategies in Table \ref{tab:build_results}.
We use the ``biased seeds + cuboid inflation'' strategy for the problem building of the random environments for its better performance.

We use a much \emph{simpler} decomposition strategy for the workcell environment.
For every target pose, the problem uses the same convex polyhedra, thus avoiding the need for repeated sampling and inflation.
Using only four manually selected axis-aligned cuboids, the problem is successfully built for all of the 600 target poses. Incidentally, this shows that the problem-building process can be efficient when the polyhedra are well designed.

\begin{table*}[htb]
  \centering
  \resizebox{\textwidth}{!}{
  \begin{tabular}{cccccccc}
    \toprule
    Scenario & Method & \makecell{Success ct. (\%), \\solving} & \makecell{ Success ct., \\ building} & \makecell{Avg. solving \\time (s)} &\makecell{Avg. building \\time (s)} & Avg. pos. err.$^{\dagger\dagger}$ & Avg. rot. err.$^{\ddagger}$\\
    \midrule
    \multirow{8}{*}{\makecell{Random\\obstacles\\random\\targets}} 
    & IKSPARK (Low-Density) & 379/500 (75.8\%) &495/500 & 0.6512 (17.60 iter.) & 3.68 & $2.25\cdot 10^{-7}$ & $5.63\cdot10^{-7}$ \\
    & IKSPARK (Mid-Density) & 338/500 (67.6\%) &490/500 & 0.9170 (21.12 iter.) & 7.88 & $2.24\cdot10^{-7}$ & $5.65\cdot10^{-7}$  \\
    & IKSPARK (High-Density) & 314/500 (62.8\%) &473/500 & 1.1229 (20.36 iter.) & 17.08 & $2.23\cdot10^{-7}$ & $5.91\cdot10^{-7}$\\
    & IKSPARK (Skip VR) & 282/500 (56.4\%) & - & 0.7806 (20.27 iter.) & 0.12 & $2.14\cdot10^{-7}$ & $5.36\cdot10^{-7}$\\
     & SNOPT & 298/500 (59.6\%) & - & 0.0039 & - & $2.85\cdot10^{-9}$ & $2.14\cdot10^{-5}$\\
     & IPOPT & 332/500 (66.4\%) & - & 0.0325 & - & $2.52\cdot10^{-9}$ & $8.51\cdot10^{-5}$ \\
     & NLOPT & 85/500 (17.0\%) & - & 0.0035 & - & $1.23\cdot10^{-7}$ & $7.50\cdot10^{-4}$\\
     & IKSPARK$^{\dagger}$ + IPOPT & 483/500 (96.6\%) & -& $0.9170+0.0093$ & - & $1.81\cdot10^{-9}$ & $6.77\cdot10^{-5}$\\
    \midrule
    \multirow{8}{*}{\makecell{Fixed\\workcell\\random\\targets}}
    & IKSPARK (Low-Density) & 499/600 (83.2\%) & 600/600 & 0.5066 (12.14 iter.) & 3.49 & $1.75\cdot10^{-7}$ & $4.27\cdot10^{-7}$\\
    & IKSPARK (Mid-Density) & 416/600 (69.3\%) & 600/600 & 0.6449 (15.22 iter.) & 6.74 & $2.12\cdot10^{-7}$ & $4.75\cdot10^{-7}$\\
    & IKSPARK (High-Density) & 409/600 (68.2\%) & 600/600 & 0.8284 (16.84 iter.) & 18.06 & $1.82\cdot10^{-7}$ & $4.61\cdot10^{-7}$ \\
    & IKSPARK (Skip VR) & 286/600 (47.7\%) & - & 0.6384 (14.78 iter.) & 0.12 & $1.89\cdot10^{-7}$ & $5.02\cdot10^{-7}$\\
     & SNOPT & 139/600 (23.2\%) & - & 0.0043 & - & $6.71\cdot10^{-10}$ & $1.02\cdot10^{-5}$ \\
     & IPOPT & 243/600 (40.5\%) & - & 0.1214 & - & $1.42\cdot10^{-9}$ & $8.97\cdot10^{-5}$ \\
     & NLOPT & 87/600 (14.5\%) & - & 0.0044 & - & $1.35\cdot10^{-7}$ & $7.64\cdot10^{-4}$\\
     & IKSPARK$^{\dagger}$ + IPOPT & 515/600 (85.8\%) & - & $0.6449 + 0.0154$ & - & $4.42\cdot10^{-9}$ & $7.57\cdot10^{-5}$\\
    \bottomrule
  \end{tabular}
  }
  \footnotesize{$^{\dagger}$Solved with mid density collision model to warm start IPOPT $\mid$ $^{\dagger\dagger}$Euclidean norm (m) $\mid$ $^{\ddagger}$Norm of angle of relative rotation (radians)}
  \caption{Performance of different IK methods in environments with random obstacles and fixed obstacles with different random end-effector targets. The results of IKSPARK with collision models of different densities and with/without variable reduction are shown. The solution time does not include problem setup time. For IKSPARK, the sum of solver time for solving the SDPs is used.}\label{tab:sawyer results combined}
  \label{tab:collision ik compare}
\end{table*}

\begin{table}[htb]
    \centering
    \resizebox{0.48\textwidth}{!}{
    \begin{tabular}{cccc}
    \toprule
        Build Method & $n_c^\dagger$ & Success ct. $(\%)$ & Avg. Build Time \\
        \midrule
        Unbiased Seeds + IRIS & 10 &280/500 (56.0\%) & 19.70 (s) \\
        Biased Seeds + IRIS & 10 &427/500 (85.4\%) & 17.50 (s) \\
        Biased Seeds + Cuboid Inflation & 3.98 &490/500 (98.0\%) & 7.88 (s) \\
        \bottomrule
    \end{tabular}
    }
    \footnotesize{$^\dagger$The average number of polyhedra after removing duplicates.}
    \caption{Results of problem construction for the same random environment with obstacles using different sampling and inflation strategies. The results are based on mid-density robot collision geometry.}
    \label{tab:build_results}
\end{table}

\begin{figure*}[htb]
  \centering
  \subfloat[Random environment]{
    \begin{minipage}{
        0.49\textwidth}
      \centering
      \includegraphics[width=1\linewidth]{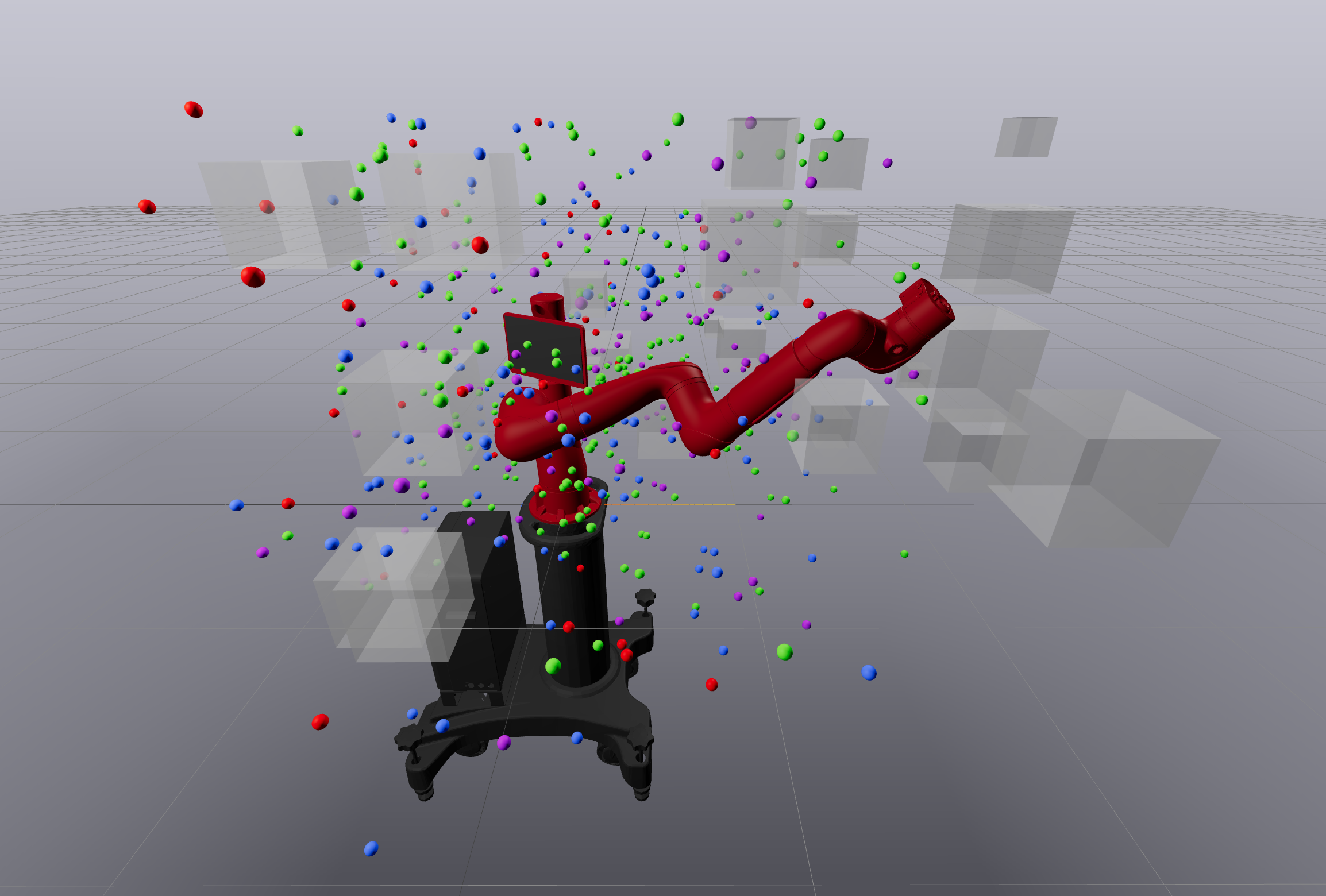}\label{fig:sawyer_random_env}
    \end{minipage}}
  \subfloat[Fixed ``workcell" environment]{
    \begin{minipage}{
        0.49\textwidth}
      \centering
      \includegraphics[width=1\linewidth]{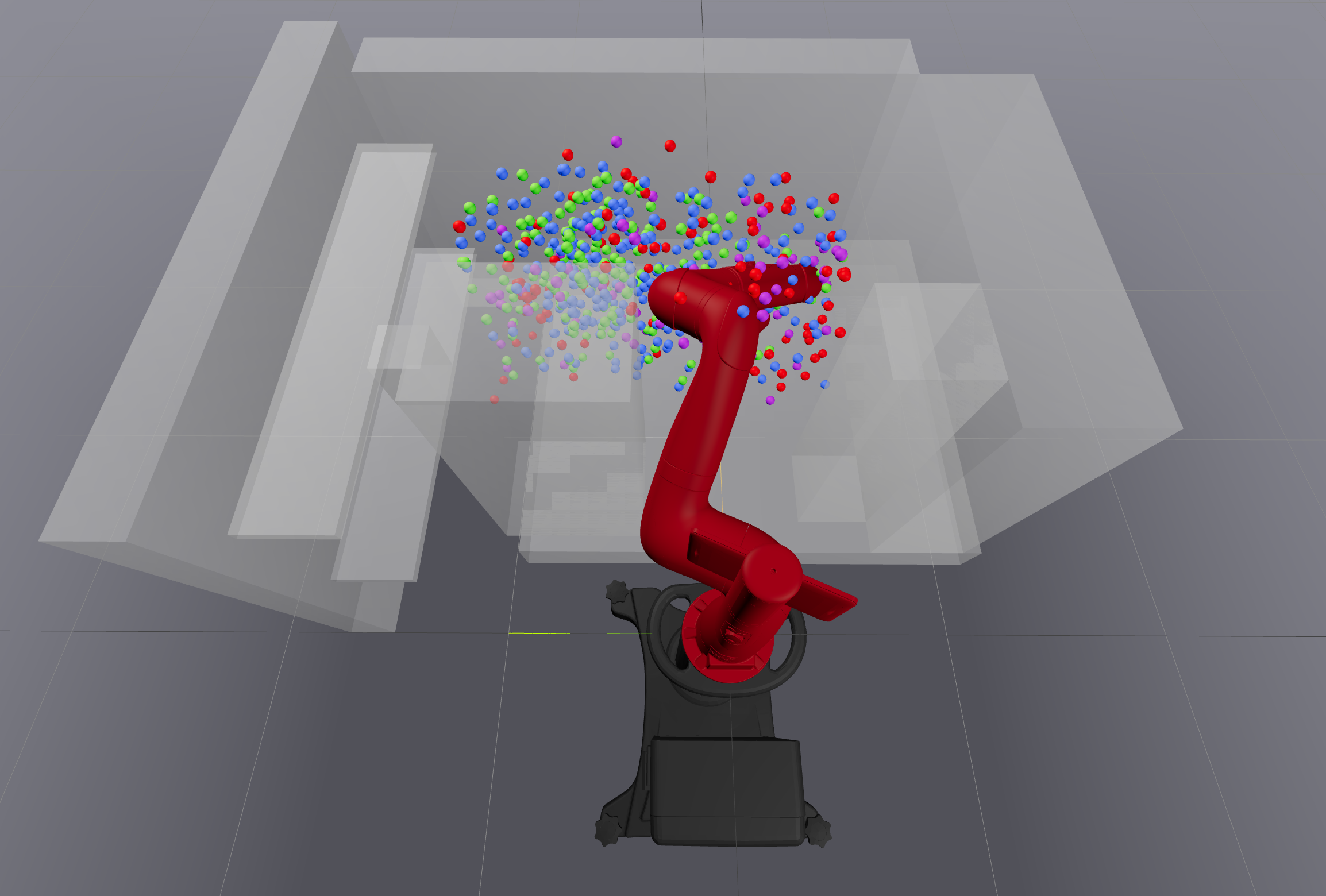}\label{fig:sawyer_fixed_env}
    \end{minipage}}
  \caption{Two different environments for the Sawyer arm. End-effector target positions are color-coded by outcome: green denotes success for both methods, blue denotes success only for IKSPARK (with mid-density collision model), purple denotes success only for IPOPT, and red denotes failure for both methods.}
  \label{fig:two envs sawyer}
\end{figure*}

\paragraph{Results.}
\begin{figure*}[htb]
  \centering
  {
    \rotatebox{90}{\tiny{\quad\makecell{\qquad Random Environment,\\ \qquad Low-Density}}}
    \includegraphics[width=0.82\linewidth, trim=6cm 0cm 6cm 0cm, clip]{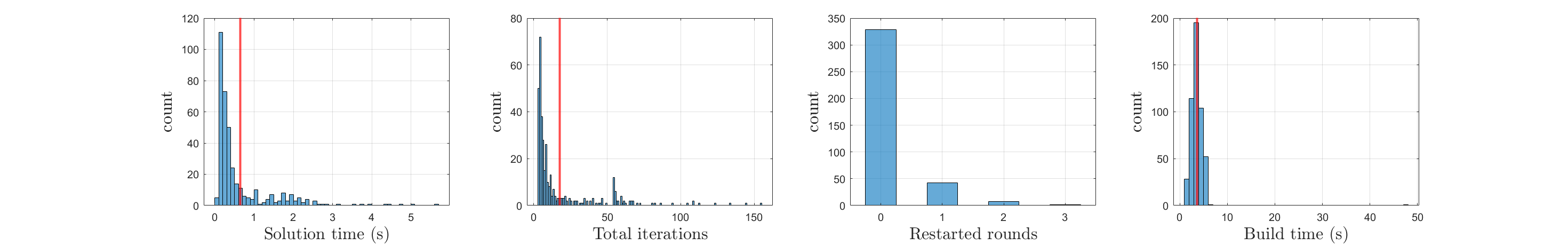}

    \rotatebox{90}{\tiny{\quad\makecell{\qquad Random Environment,\\ \qquad  Mid-Density}}}
    \includegraphics[width=0.82\linewidth, trim=6cm 0cm 6cm 0cm, clip]{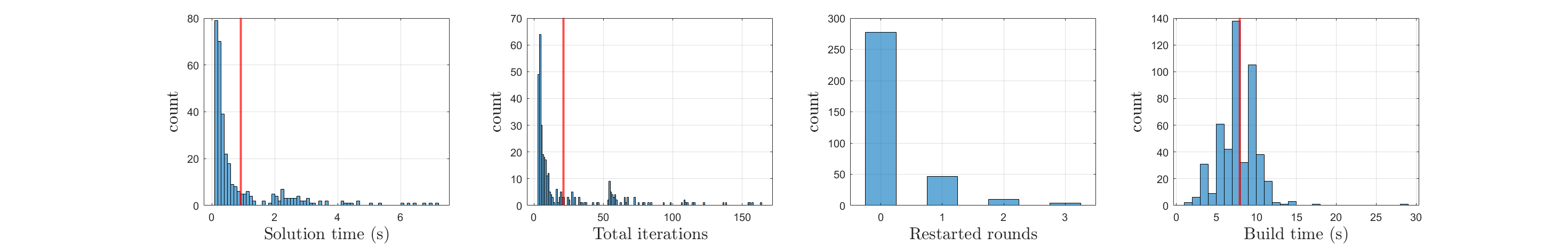}

    \rotatebox{90}{\tiny{\quad\makecell{\qquad Random Environment,\\ \qquad  High-Density}}}
    \includegraphics[width=0.82\linewidth, trim=6cm 0cm 6cm 0cm, clip]{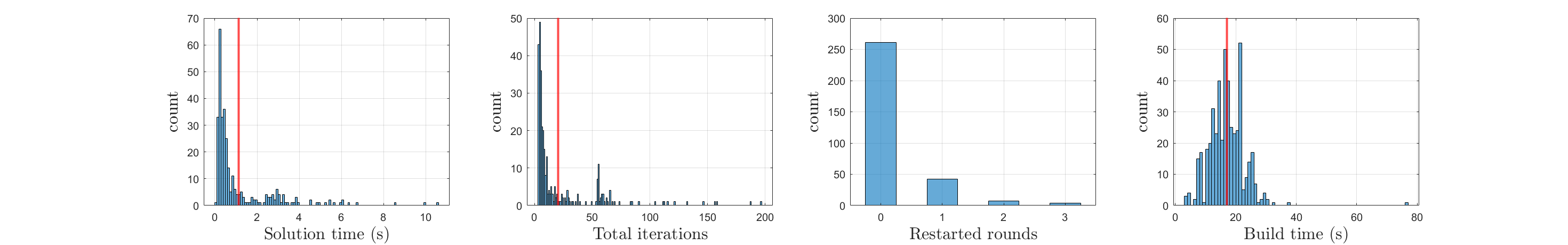}

    \rotatebox{90}{\tiny{\quad\makecell{\qquad Random Environment,\\ \qquad Mid-Density, skip VR}}}
    \includegraphics[width=0.82\linewidth, trim=6cm 0cm 6cm 0cm, clip]{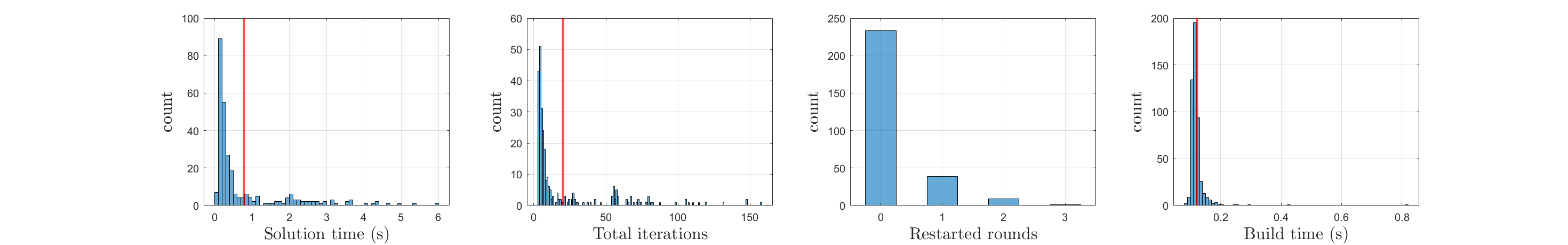}

    \rotatebox{90}{\tiny{\quad\makecell{\qquad Workcell Environment,\\ \qquad Low-Density}}}
    \includegraphics[width=0.82\linewidth, trim=6cm 0cm 6cm 0cm, clip]{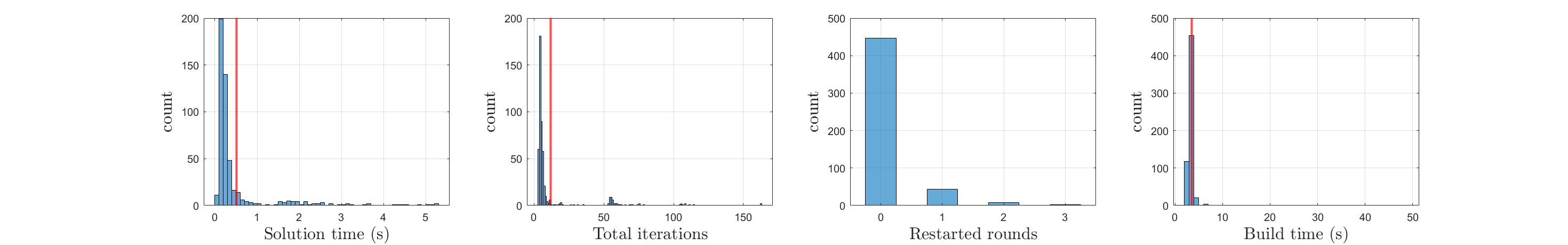}

    \rotatebox{90}{\tiny{\quad\makecell{\qquad Workcell Environment,\\ \qquad Mid-Density}}}
    \includegraphics[width=0.82\linewidth, trim=6cm 0cm 6cm 0cm, clip]{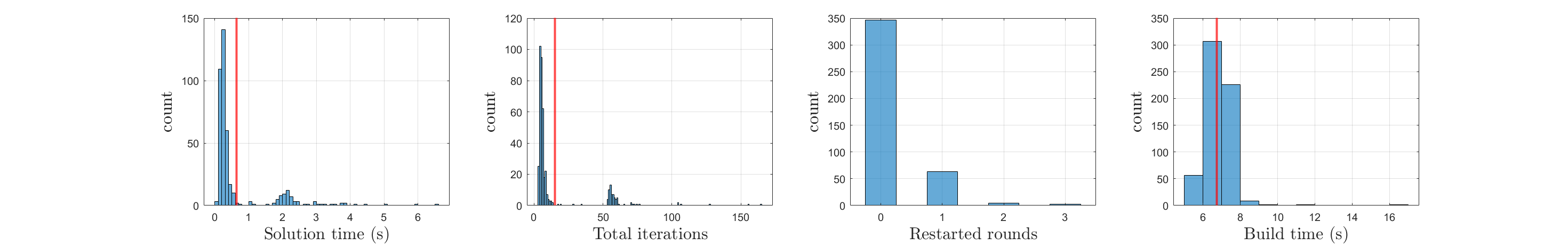}

    \rotatebox{90}{\tiny{\quad\makecell{\qquad Workcell Environment,\\ \qquad High-Density}}}
    \includegraphics[width=0.82\linewidth, trim=6cm 0cm 6cm 0cm, clip]{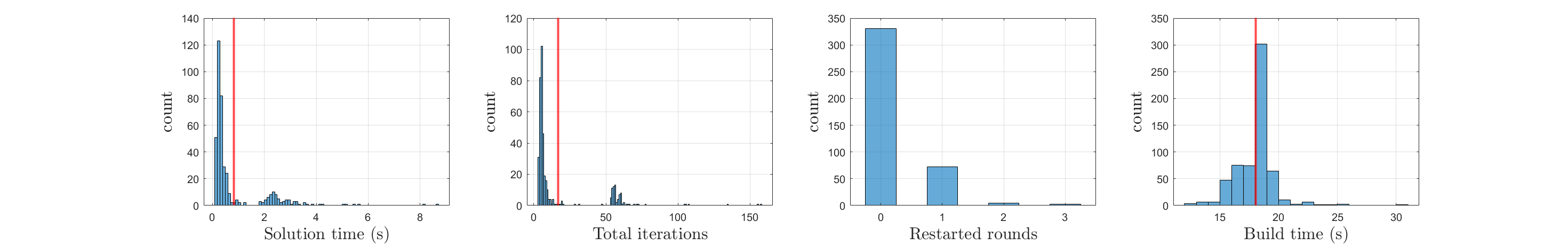}

    \rotatebox{90}{\tiny{\quad\makecell{\qquad Workcell Environment,\\ \qquad Mid-Density, skip VR}}}
    \includegraphics[width=0.82\linewidth, trim=6cm 0cm 6cm 0cm, clip]{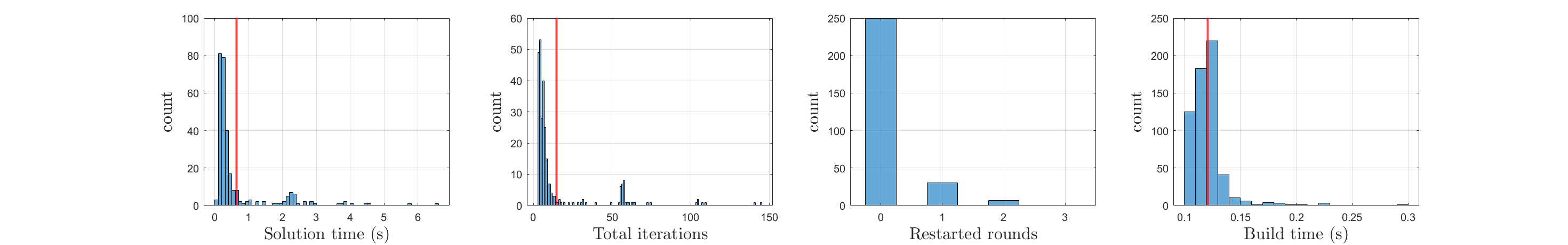}
  }
  \caption{Performance statistics of IKSPARK for solving IK with different settings in Table \ref{tab:collision ik compare}. The columns show the time taken for solving the SDPs, the rank-minimization iterations, the number of restarts, and the time for building the problem and performing the feasibility checks for all pairs of collision bodies and convex polyhedra (for ``skip VR'' results, the feasibility checks are skipped).}
  \label{fig:time_plots}
\end{figure*}

IKSPARK is evaluated on two problem sets with different collision-geometry settings and compared with Drake’s \texttt{InverseKinematics} module. In Drake, the obstacle-aware IK problem is formulated as a nonconvex optimization problem and solved using nonlinear programming solvers such as SNOPT \cite{gill2005snopt}, IPOPT \cite{wachter2006ipopt}, and NLOPT \cite{johnson2014nlopt}. We disable the self-collision constraints and use the zero angles as the initial guess of the Drake solvers. We report the time taken for solving the nonlinear optimization problems, and count a solution as successful if the solver returns with a status of ``success'' and the resulting collision bodies (see Figure \ref{fig:density_urdf}) are not in collision with the obstacles. 

For IKSPARK, we use quaternion-based variables and apply Algorithm \ref{alg:ik} with Problem \ref{prob:IK3} added with \eqref{eq:ik_obs_bilinear}--\eqref{eq:ik_obs_link_cb} as the rank minimization update; we allow up to three restarts, each with a limit of 50 iterations, for each problem. A solution is counted as successful if it is rank-1 (up to tolerance $\epsilon_1=10^{-5}$), and the collision bodies are not in collision with the obstacles. The time consumed for solving the SDPs are recorded, along with the iterations taken in the rank-minimization process. 
We also test the performance of IKSPARK using mid-density collision models and skipping the variable reduction step described in Section \ref{sec:reduce_convex_poly}, shown as ``IKSPARK (Skip VR)'' in the results. 

Table~\ref{tab:collision ik compare} shows that, in random obstacle environments, IKSPARK achieves success rates comparable to those of Drake’s SNOPT and IPOPT solvers. In fixed workcell environments, however, IKSPARK attains a substantially higher success rate. Figure \ref{fig:two envs sawyer} visualizes the target positions and the corresponding outcomes, while Figure \ref{fig:time_plots} summarizes the performance statistics of IKSPARK in both random and fixed environments. Although the nonlinear-programming-based solvers are computationally faster than IKSPARK, all methods achieve small end-effector position and orientation errors in successful cases.

Table~\ref{tab:collision ik compare} also reports the number of environments in which the polyhedral approximation provides sufficient obstacle coverage to successfully construct the IKSPARK optimization problem, reported as ``success ct., building.'' Furthermore, as shown in the rows labeled ``IKSPARK+IPOPT,'' initializing Drake’s IPOPT solver with the IKSPARK solution results in a higher success rate than either individual method.

When the variable-reduction step is skipped, all generated convex polyhedra are retained for every collision body, and the success rate of IKSPARK decreases. In this setting, the larger problem size appears to increase the frequency with which the solver fails to converge within the iteration limit, rather than significantly increasing the average solution time for successful cases. This suggests that the variable-reduction step improves robustness and overall success rate, whereas skipping it may reduce runtime at the cost of more frequent failures.

Another notable observation is that, in all successful IKSPARK solutions in this study, the binary variables in the obstacle-avoidance constraints are nearly integral. Specifically, the final solutions show $\min_j(\max_i \delta_{ij}) \geq 0.9990$. Since $\sum_{i=1}^{n_c} \delta_{ij}=1,\forall j$, this implies that the remaining nonmaximal values satisfy $\delta_{ij}\leq 0.001,\forall j$. Although these variables are generally fractional in the initial relaxed solution, they become effectively binary during the rank-minimization process. We attribute this behavior to the specific implementation of the underlying MOSEK solver rather than to a property specific to IKSPARK, although more investigation in this sense is needed.

\subsection{Shadow Dexterous Hand grasping}
\begin{figure*}[h]
  \centering
  \subfloat[Collision geometry, front view]{%
    \includegraphics[width=0.332\textwidth]{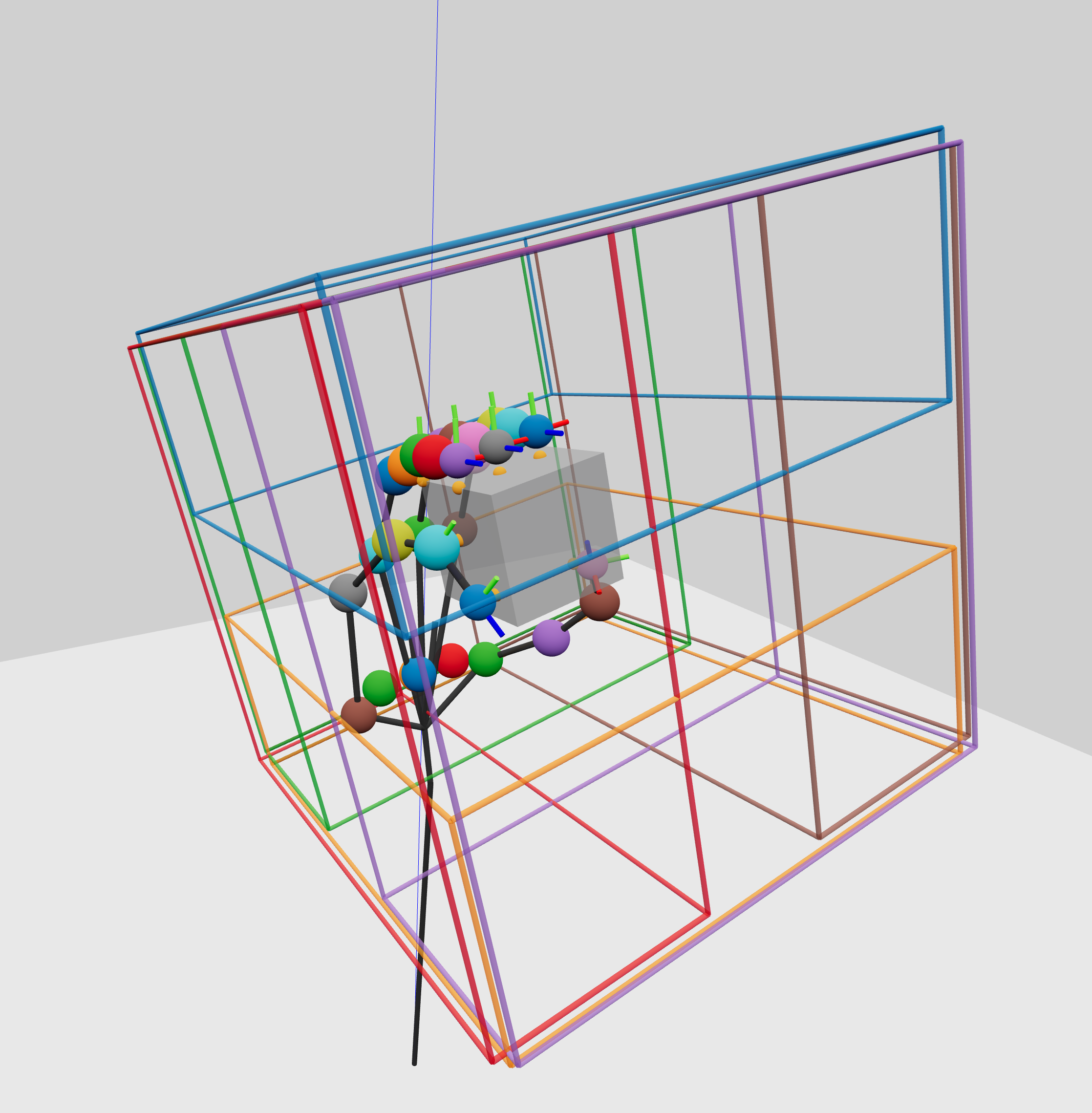}%
    \label{fig:hand_grasp_demo_s1}}
  \subfloat[Grasping posture]{%
    \includegraphics[width=0.248\textwidth]{shadow_hand_cube.png}%
    \label{fig:hand_grasp_demo_0}}
  \subfloat[Collision geometry, rear view]{%
    \includegraphics[width=0.30\textwidth]{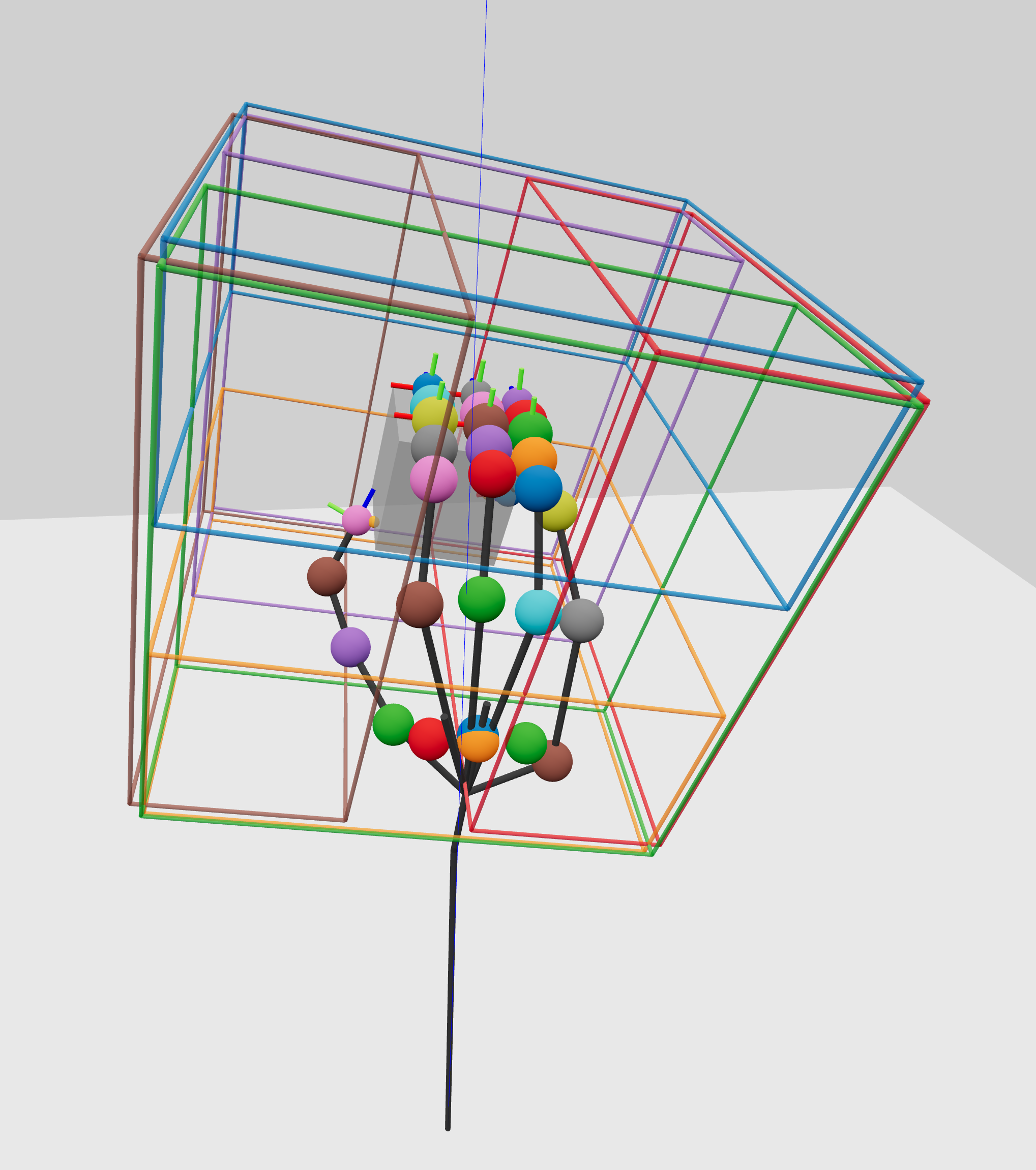}%
    \label{fig:hand_grasp_demo_s2}}
  \caption{IKSPARK finds a grasping posture for the Shadow Dexterous Hand by enforcing assigned finger-face contacts on the cube while avoiding collision, with free space modeled as cuboids (colored wire frames) and the hand collision geometry approximated by spheres.}
  \label{fig:shadow-hand}
\end{figure*}

We demonstrate that IKSPARK can be applied to grasp planning, involving both closed kinematic chains and contacts. Specifically, we consider the Shadow Dexterous Hand and the grasp posture shown in Figure \ref{fig:shadow-hand}: in this example, the 24-DOF hand is required to make contact with the cube while avoiding unintended collision or interpenetration elsewhere. The joint limits are enforced for all 24 revolute joints.

To enforce contact, we assign each finger to a face of the cube and impose, for each finger-face pair, a convex constraint requiring the designated contact point on the finger link to lie on that face. For collision avoidance, we model the free region as a union of cuboids around the cube and approximate the hand collision geometry by spheres constrained to remain within these cuboids, as illustrated in Figures~\ref{fig:hand_grasp_demo_s1} and~\ref{fig:hand_grasp_demo_s2}. 

Constructing the optimization problem takes $15.78$ seconds, including $15.61$ seconds for solving the feasibility problems \eqref{eq:ik_obs_feas} for all pairs. Algorithm \ref{alg:ik} then takes $3.16$ seconds to compute the grasp posture shown in Figure \ref{fig:hand_grasp_demo_0}.

\section{Conclusions}
This paper presented IKSPARK, an obstacle-aware inverse kinematics solver that reformulates inverse kinematics as a convex semidefinite program with additional rank-1 constraints through the introduction of new decision variables. To handle the nonconvexity of the rank-1 constraints, we proposed two rank-minimization schemes that maximize the largest eigenvalues under constant traces of the decision variables. 
The resulting unified framework accommodates diverse joint types, structural constraints, and obstacle avoidance requirements. 
The solver can certify infeasibility of IK problems by solving the convex relaxation of the reformulated problem.
We demonstrated the effectiveness of IKSPARK on various robots, including a dual-arm Baxter robot, a Stewart platform with prismatic joints, a Sawyer arm in cluttered environments, and a Shadow Dexterous Hand performing grasping tasks.

\subsection{Limitations}
At present, the main limitation of IKSPARK is its computational cost, which is typically higher than that of nonlinear programming solvers. The primary bottleneck is the efficiency of semidefinite programming solvers, both in the feasibility checks required for problem construction and in the optimization problems solved in Algorithm~\ref{alg:ik}. Improving SDP solution methods would therefore directly enhance the practical efficiency of the approach.
The restart scheme proposed in Algorithm \ref{alg:reproj} can improve the success rate of IKSPARK, but it also increases the overall runtime. In practice, the user can choose the number of restarts and the iteration limit for each restart to balance the success rate and runtime.
For fixed workspaces, the cost of problem construction may be further reduced by accelerating the feasibility checks. In particular, precomputed lookup tables or learning-based predictors could be used to estimate the feasibility of each collision-body-polyhedron pair, reducing the need to solve SDPs exhaustively for all pairs.

A further limitation is that the current formulation does not incorporate self-collision avoidance, which is important for many practical applications.

\subsection{Future work}
For future work, we plan to extend IKSPARK to handle self-collision avoidance and to explore more efficient algorithms for the rank-minimization step. It might be possible to extend our formulation to collision bodies formed by convex hulls of spheres and points. We also aim to embed IKSPARK within a trajectory optimization framework to solve motion planning problems with complex constraints such as the ones arising in robot dynamics.






\bibliographystyle{ieee}
\bibliography{main}

@inproceedings{wu2023cdc,
  author    = {Wu, Liangting and Tron, Roberto},
  booktitle = {2023 62nd IEEE Conference on Decision and Control (CDC)},
  title     = {An SDP Optimization Formulation for the Inverse Kinematics Problem},
  year      = {2023},
  volume    = {},
  number    = {},
  pages     = {4731-4738},
  doi       = {10.1109/CDC49753.2023.10384035}
}

@article{wu2025certifiably,
  title   = {Certifiably Optimal Estimation and Calibration in Robotics via Trace-Constrained Semi-Definite Programming},
  author  = {Wu, Liangting and Tron, Roberto},
  journal = {arXiv preprint arXiv:2509.23656},
  year    = {2025}
}

@article{siciliano2009modelling,
  title     = {Modelling, planning and control},
  author    = {Siciliano, Bruno and Sciavicco, Lorenzo and Villani, Luigi and Oriolo, Giuseppe},
  journal   = {Advanced Textbooks in Control and Signal Processing. Springer,},
  year      = {2009},
  publisher = {Springer}
}

@article{lee1988displacement,
  title     = {Displacement analysis of the general spatial 7-link 7R mechanism},
  author    = {Lee, Hong-You and Liang, Chong-Gao},
  journal   = {Mechanism and machine theory},
  volume    = {23},
  number    = {3},
  pages     = {219--226},
  year      = {1988},
  publisher = {Elsevier}
}

@article{husty2007new,
  title     = {A new and efficient algorithm for the inverse kinematics of a general serial 6R manipulator},
  author    = {Husty, Manfred L and Pfurner, Martin and Schr{\"o}cker, Hans-Peter},
  journal   = {Mechanism and machine theory},
  volume    = {42},
  number    = {1},
  pages     = {66--81},
  year      = {2007},
  publisher = {Elsevier}
}

@article{raghavan1993inverse,
    author = {Raghavan, M. and Roth, B.},
    title = {Inverse Kinematics of the General 6R Manipulator and Related Linkages},
    journal = {Journal of Mechanical Design},
    volume = {115},
    number = {3},
    pages = {502-508},
    year = {1993},
    month = {09},
    doi = {10.1115/1.2919218},
}

@phdthesis{diankov2010automated,
  title={Automated construction of robotic manipulation programs},
  author={Diankov, Rosen},
  year={2010},
  school={Carnegie Mellon University, USA}
}

@article{kenwright2012inverse,
  title     = {Inverse kinematics--cyclic coordinate descent (CCD)},
  author    = {Kenwright, Ben},
  journal   = {Journal of Graphics Tools},
  volume    = {16},
  number    = {4},
  pages     = {177--217},
  year      = {2012},
  publisher = {Taylor \& Francis}
}

@article{aristidou2011fabrik,
  title     = {FABRIK: A fast, iterative solver for the Inverse Kinematics problem},
  author    = {Aristidou, Andreas and Lasenby, Joan},
  journal   = {Graphical Models},
  volume    = {73},
  number    = {5},
  pages     = {243--260},
  year      = {2011},
  publisher = {Elsevier}
}

@article{muller2007triangulation,
  title={Triangulation: A new algorithm for inverse kinematics},
  author={Muller-Cajar, R and Mukundan, Ramakrishnan},
  journal={Proceedings of Image and Vision Computing New Zealand},
  volume={2007},
  pages={181--186},
  year={2007}
}

@inproceedings{Beeson2015Trac,
  author    = {Beeson, Patrick and Ames, Barrett},
  booktitle = {2015 IEEE-RAS 15th International Conference on Humanoid Robots (Humanoids)},
  title     = {TRAC-IK: An open-source library for improved solving of generic inverse kinematics},
  year      = {2015},
  volume    = {},
  number    = {},
  pages     = {928-935}
}

@article{dai2019global,
  title     = {Global inverse kinematics via mixed-integer convex optimization},
  author    = {Dai, Hongkai and Izatt, Gregory and Tedrake, Russ},
  journal   = {The International Journal of Robotics Research},
  volume    = {38},
  number    = {12-13},
  pages     = {1420--1441},
  year      = {2019},
  publisher = {SAGE Publications Sage UK: London, England}
}

@inproceedings{yenamandra2019convex,
  title        = {Convex optimisation for inverse kinematics},
  author       = {Yenamandra, Tarun and Bernard, Florian and Wang, Jiayi and Mueller, Franziska and Theobalt, Christian},
  booktitle    = {2019 International Conference on 3D Vision (3DV)},
  pages        = {318--327},
  year         = {2019},
  organization = {IEEE}
}

@article{saunderson2015sdp,
  author  = {Saunderson, J. and Parrilo, P. A. and Willsky, A. S.},
  title   = {Semidefinite Descriptions of the Convex Hull of Rotation Matrices},
  journal = {SIAM Journal on Optimization},
  volume  = {25},
  number  = {3},
  pages   = {1314-1343},
  year    = {2015}
}

@article{bandeira2017tightness,
  title   = {Tightness of the maximum likelihood semidefinite relaxation for angular synchronization},
  author  = {Bandeira, Afonso S and Boumal, Nicolas and Singer, Amit},
  journal = {Mathematical Programming},
  volume  = {163},
  pages   = {145--167},
  year    = {2017}
}

@inproceedings{yang2019quaternion,
  title     = {A quaternion-based certifiably optimal solution to the {W}ahba problem with outliers},
  author    = {Yang, Heng and Carlone, Luca},
  booktitle = {International Conference on Computer Vision (ICCV)},
  pages     = {1665--1674},
  year      = {2019}
}

@inproceedings{li2020robot,
  title     = {Robot-to-robot relative pose estimation based on semidefinite relaxation optimization},
  author    = {Li, Ming and Liang, Guanqi and Luo, Haobo and Qian, Huihuan and Lam, Tin Lun},
  booktitle = {International Conference on Intelligent Robots and Systems (IROS)},
  pages     = {4491--4498},
  year      = {2020}
}

@phdthesis{yang2022certifiable,
  title  = {Certifiable Outlier-Robust Geometric Perception},
  author = {Yang, Heng},
  year   = {2022},
  school = {Massachusetts Institute of Technology}
}

@article{yang2020teaser,
  title     = {{TEASER}: Fast and certifiable point cloud registration},
  author    = {Yang, Heng and Shi, Jingnan and Carlone, Luca},
  journal   = {IEEE Transactions on Robotics},
  volume    = {37},
  number    = {2},
  pages     = {314--333},
  year      = {2020},
  publisher = {IEEE}
}

@inproceedings{peng2022semidefinite,
  title     = {Semidefinite Relaxations of {T}runcated {L}east-{S}quares in Robust Rotation Search: Tight or Not},
  author    = {Peng, Liangzu and Fazlyab, Mahyar and Vidal, Ren{\'e}},
  booktitle = {European Conference on Computer Vision (ECCV)},
  year      = {2022},
  pages     = {673--691}
}

@article{Naour2019kinematics,
  title   = {Kinematics in the metric space},
  author  = {Le Naour, Thibaut and Courty, Nicolas and Gibet, Sylvie},
  journal = {Computers \& Graphics},
  volume  = {84},
  pages   = {13--23},
  year    = {2019}
}

@article{magnus1985,
  title     = {On Differentiating Eigenvalues and Eigenvectors},
  volume    = {1},
  journal   = {Econometric Theory},
  publisher = {Cambridge University Press},
  author    = {Magnus, Jan R.},
  year      = {1985},
  pages     = {179–191}
}

@manual{mosek,
  author = {{MOSEK ApS}},
  title  = {The MOSEK optimization toolbox for MATLAB manual. Version 10.0.},
  year   = 2022,
  url    = {http://docs.mosek.com/10.0/toolbox/index.html}
}

@article{umeyama1991least,
  title   = {Least-squares estimation of transformation parameters between two point patterns},
  author  = {Umeyama, Shinji},
  journal = {IEEE Transactions on Pattern Analysis \& Machine Intelligence},
  volume  = {13},
  number  = {04},
  pages   = {376--380},
  year    = {1991}
}

@article{giamou2022convex,
  title     = {Convex iteration for distance-geometric inverse kinematics},
  author    = {Giamou, Matthew and Mari{\'c}, Filip and Rosen, David M and Peretroukhin, Valentin and Roy, Nicholas and Petrovi{\'c}, Ivan and Kelly, Jonathan},
  journal   = {IEEE Robotics and Automation Letters},
  volume    = {7},
  number    = {2},
  pages     = {1952--1959},
  year      = {2022},
  publisher = {IEEE}
}

@article{stewart1965platform,
  title  = {A Platform with six degrees of freedom},
  journal= {Proceedings of the Institute of Mechanical Engineers},
  author = {Stewart, D},
  volume = {180},
  number = {1},
  pages  = {371-386.},
  year   = {1965}
}

@misc{griffis1993method,
  title     = {Method and apparatus for controlling geometrically simple parallel mechanisms with distinctive connections},
  author    = {Griffis, Michael W and Duffy, Joseph},
  year      = {1993},
  month     = jan # {~12},
  publisher = {Google Patents},
  note      = {US Patent 5,179,525}
}

@article{porta2009linear,
  title     = {A linear relaxation technique for the position analysis of multiloop linkages},
  author    = {Porta, Josep M and Ros, Lluis and Thomas, Federico},
  journal   = {IEEE Transactions on Robotics},
  volume    = {25},
  number    = {2},
  pages     = {225--239},
  year      = {2009},
  publisher = {IEEE}
}

@incollection{dietmaier1998stewart,
  title     = {The Stewart-Gough platform of general geometry can have 40 real postures},
  author    = {Dietmaier, Peter},
  booktitle = {Advances in robot kinematics: Analysis and control},
  pages     = {7--16},
  year      = {1998},
  publisher = {Springer}
}

@article{marcucci2024shortest,
  title     = {Shortest paths in graphs of convex sets},
  author    = {Marcucci, Tobia and Umenberger, Jack and Parrilo, Pablo and Tedrake, Russ},
  journal   = {SIAM Journal on Optimization},
  volume    = {34},
  number    = {1},
  pages     = {507--532},
  year      = {2024},
  publisher = {SIAM}
}

@inproceedings{deits2015computing,
  title        = {Computing large convex regions of obstacle-free space through semidefinite programming},
  author       = {Deits, Robin and Tedrake, Russ},
  booktitle    = {Algorithmic Foundations of Robotics XI: Selected Contributions of the Eleventh International Workshop on the Algorithmic Foundations of Robotics},
  pages        = {109--124},
  year         = {2015},
  organization = {Springer}
}

@inproceedings{di_lillo_safety-related_2018,
  title     = {Safety-{Related} {Tasks} {Within} the {Set}-{Based} {Task}-{Priority} {Inverse} {Kinematics} {Framework}},
  booktitle = {2018 {IEEE}/{RSJ} {International} {Conference} on {Intelligent} {Robots} and {Systems} ({IROS})},
  author    = {Di Lillo, Paolo and Arrichiello, Filippo and Antonelli, Gianluca and Chiaverini, Stefano},
  year      = {2018},
  pages     = {6130--6135}
}

@article{khatib_task_2020,
  title    = {Task {Priority} {Matrix} at the {Acceleration} {Level}: {Collision} {Avoidance} {Under} {Relaxed} {Constraints}},
  volume   = {5},
  number   = {3},
  journal  = {IEEE Robotics and Automation Letters},
  author   = {Khatib, Maram and Al Khudir, Khaled and De Luca, Alessandro},
  year     = {2020},
  pages    = {4970--4977}
}

@article{weingartshofer2023optimization,
  title     = {Optimization-based path planning framework for industrial manufacturing processes with complex continuous paths},
  author    = {Weingartshofer, Thomas and Bischof, Bernhard and Meiringer, Martin and Hartl-Nesic, Christian and Kugi, Andreas},
  journal   = {Robotics and Computer-Integrated Manufacturing},
  volume    = {82},
  pages     = {102516},
  year      = {2023},
  publisher = {Elsevier}
}

@article{yang2024obstacle,
  title     = {An obstacle-avoidance inverse kinematics method for robotic manipulator in overhead multi-line environment},
  author    = {Yang, Pengju and Shen, Feng and Xu, Dingjie and Chen, Bingxing and Liu, Ronghai and Wang, Hongwu},
  journal   = {Engineering Science and Technology, an International Journal},
  volume    = {53},
  pages     = {101686},
  year      = {2024},
  publisher = {Elsevier}
}

@misc{tedrake2023manipulation,
  author       = {Tedrake, Russ},
  title        = {Robotic Manipulation},
  year         = {2023},
  howpublished = {\url{https://manipulation.csail.mit.edu/}},
  note         = {MIT course notes}
}

@misc{drake,
  author = {Russ Tedrake and the Drake Development Team},
  title  = {Drake: Model-based design and verification for robotics},
  year   = 2019,
  url    = {https://drake.mit.edu}
}

@article{gill2005snopt,
  author  = {Gill, Philip E. and Murray, Walter and Saunders, Michael A.},
  title   = {SNOPT: An SQP Algorithm for Large-Scale Constrained Optimization},
  journal = {SIAM Review},
  volume  = {47},
  number  = {1},
  pages   = {99--131},
  year    = {2005},
  doi     = {10.1137/S0036144504446096}
}

@article{wachter2006ipopt,
  author  = {W{\"a}chter, Andreas and Biegler, Lorenz T.},
  title   = {On the Implementation of an Interior-Point Filter Line-Search Algorithm for Large-Scale Nonlinear Programming},
  journal = {Mathematical Programming},
  volume  = {106},
  number  = {1},
  pages   = {25--57},
  year    = {2006},
  doi     = {10.1007/s10107-004-0559-y}
}

@misc{johnson2014nlopt,
  author       = {Johnson, Steven G.},
  title        = {The NLopt Nonlinear-Optimization Package},
  year         = {2014},
  howpublished = {\url{https://nlopt.readthedocs.io}},
  note         = {Available at https://github.com/stevengj/nlopt}
}

@inproceedings{marangoz2023dawnik,
  title        = {DawnIK: Decentralized collision-aware inverse kinematics solver for heterogeneous multi-arm systems},
  author       = {Marangoz, Salih and Menon, Rohit and Dengler, Nils and Bennewitz, Maren},
  booktitle    = {2023 IEEE-RAS 22nd International Conference on Humanoid Robots (Humanoids)},
  pages        = {1--8},
  year         = {2023},
  organization = {IEEE}
}

@misc{cvx,
  author       = {CVX Research, Inc.},
  title        = {{CVX}: Matlab Software for Disciplined Convex Programming, version 2.0},
  howpublished = {\url{https://cvxr.com/cvx}},
  month        = aug,
  year         = 2012
}

@incollection{gb08,
  author    = {M. Grant and S. Boyd},
  title     = {Graph implementations for nonsmooth convex programs},
  booktitle = {Recent Advances in Learning and Control},
  series    = {Lecture Notes in Control and Information Sciences},
  editor    = {V. Blondel and S. Boyd and H. Kimura},
  publisher = {Springer-Verlag Limited},
  pages     = {95--110},
  year      = 2008,
  note      = {\url{http://stanford.edu/~boyd/graph_dcp.html}}
}

\appendices
\makeatletter
\renewcommand{\thesection}{Appendix~\Alph{section}}
\makeatother

\begin{table*}[h]
\centering
\caption{Geometric parameters of the Stewart platforms}\label{tab:stewart geo}
\begin{tabular}{l|ll|lll}
\toprule
{} &  \multicolumn{2}{c|}{Griffis/Duffy} & \multicolumn{3}{c}{Dietmaier}\\
i   & $A_i$   & $B_i$  & $A_i$ & $B_i$ &$l_i$\\
\midrule
1   &  $(0,0,0)$ & $(0,0,0)$ & $(0,0,0)$ & $(0,0,0)$ & 1\\
2   &  $(c,s,0)$ & $(-c,s,0)$ &  $(1.107915, 0, 0)$  & $(0.542805, 0, 0)$ &$0.645275$\\
3   &  $(2c,2s, 0)$  &  $(c,s,0)$  & $(0.549094, 0. 756063, 0)$ & $(0.956919, - 0.528915, 0)$ &$1.086284$\\
4   &  $(1+c,s,0)$  &  $(3c,s,0)$   & $(0. 735077, - 0.223935, 0.525991)$  & $(0.665885, - 0.353482, 1.402538)$ &$1.503439$\\
5   &  $(2, 0,0)$  &  $(2c,0,0)$ & $(0.514188, - 0.526063, - 0.368418)$  & $(0.478359, 1.158742, 0.107672)$ &$1.281933$\\
6   &  $(1 , 0, 0)$  &  $(c,-s,0)$ & $(0.590473, 0.094733, - 0.205018)$  & $(- 0.137087, - 0.235121, 0.353913)$ &$0.771071$\\
\bottomrule
\end{tabular}\\
\footnotesize{The parameters $c=\cos(\pi/3)$ and $s = \sin(\pi/3)$.}
\end{table*}
\section{Proof for Proposition \ref{prop:prismatic trans Y_tau}}\label{appendix:proof_prismatic_trans}
We start with the following lemma and proof.
\begin{lemma}\label{lem:rank1 Ytau sqrt}
    Any rank-1 $\mY_{\tau i}$ satisfying \eqref{eq:linearconstraintsYtau} can be written as
    \begin{equation}\label{eq:Ytaui s1 s2}
        \begin{aligned}
            \mY_{\tau i} &= \Mat{s_1\sqrt{t}\vy\\ s_1\sqrt{(1-t)}\vy\\ s_2\sqrt{t}\\ s_2\sqrt{1-t}}\Mat{s_1\sqrt{t}\vy\\ s_1\sqrt{(1-t)}\vy\\ s_2\sqrt{t}\\ s_2\sqrt{1-t}}^\trans,\\
            s_1\text{, }s_2&=\pm 1.
        \end{aligned}
    \end{equation}
  where $\trace(\vy\vy^\trans)=1$, $t\in [0,1]$.
\end{lemma}
\begin{proof}
    Since \(\mY_{\tau i}\succeq 0\) and \(\operatorname{rank}(\mY_{\tau i})=1\),
there exists a vector
\[
\capgreek{\xi} =
\begin{bmatrix}
\vx\\ \vz\\ \alpha\\ \beta
\end{bmatrix},
\qquad \vx,\vz\in\mathbb{R}^3,\quad \alpha,\beta\in\mathbb{R},
\]
such that
\[
\mY_{\tau i}=\capgreek{\xi}\capgreek{\xi}^\trans .
\]
We now translate the linear constraints in \eqref{eq:linearconstraintsYtau} into constraints on
\(\capgreek{\xi}\). From the trace constraint \ref{itm:ytau1} and the block-trace constraints \ref{itm:ytau2},
\[
\|\vx\|^2+\|\vz\|^2+\alpha^2+\beta^2=2,
\]
and
\[
\|\vx\|^2=\alpha^2,\qquad \|\vz\|^2=\beta^2.
\]
Therefore,
\[
\alpha^2+\beta^2=1.
\]
Let
\[
t:=\alpha^2.
\]
Then \(t\in[0,1]\) and \(\beta^2=1-t\). Moreover, the constraint
\(\mY_{\tau i}(7,8)\ge 0\) gives
\[
\alpha\beta\ge 0.
\]
Hence \(\alpha\) and \(\beta\) have the same sign, allowing zeros, so
there exists \(s_2\in\{-1,1\}\) such that
\[
\alpha=s_2\sqrt{t},
\qquad
\beta=s_2\sqrt{1-t}.
\]

It remains to show that \(\vx\) and \(\vz\) can be written using the same
unit vector. The constraint 
$\mY_{\tau i}(4:6,7)=\mY_{\tau i}(1:3,8)$ 
is equivalent to
\[
\alpha \vz=\beta \vx.
\]

First suppose \(t\in(0,1)\). Then \(\alpha\neq0\) and \(\beta\neq0\).
Thus
$\frac{\vx}{\alpha}=\frac{\vz}{\beta}$.
Define
\[
\vr:=\frac{\vx}{\alpha}=\frac{\vz}{\beta}.
\]
Since \(\|\vx\|^2=\alpha^2\), we have \(\|\vr\|=1\). Hence
\[
\vx=\alpha \vr=s_2\sqrt{t}\,\vr,
\qquad
\vz=\beta \vr=s_2\sqrt{1-t}\,\vr.
\]
Choosing any \(s_1\in\{-1,1\}\) and setting \(\vy=s_1s_2\vr\) gives the desired representation.

Now consider the boundary case \(t=0\). Then \(\alpha=0\) and
\(|\beta|=1\). From \(\|\vx\|^2=\alpha^2\), we get \(\vx=\omat\). Also
\(\|\vz\|^2=\beta^2=1\), so \(\vz\) is a unit vector. Choose
\(s_2=\operatorname{sign}(\beta)\), choose any \(s_1\in\{-1,1\}\), and
set \(\vy=s_1 \vz\). Then
\[
\vx=s_1\sqrt{0}\vy,
\quad
\vz=s_1\sqrt{1}\vy,
\quad
\alpha=s_2\sqrt{0},
\quad
\beta=s_2\sqrt{1}.
\]
Thus the desired representation also holds when \(t=0\).

Finally, consider the boundary case \(t=1\). Then \(|\alpha|=1\) and
\(\beta=0\). From \(\|z\|^2=\beta^2\), we get \(\vz=\omat\). Also
\(\|\vx\|^2=\alpha^2=1\), so \(\vx\) is a unit vector. Choose
\(s_2=\operatorname{sign}(\alpha)\), choose any \(s_1\in\{-1,1\}\), and
set \(\vy=s_1 \vx\). Then
\[
\vx=s_1\sqrt{1}\vy,
\quad
\vz=s_1\sqrt{0}\vy,
\quad
\alpha=s_2\sqrt{1},
\quad
\beta=s_2\sqrt{0}.
\]
Hence the representation holds for \(t=1\) as well.

Combining the cases \(t\in(0,1)\), \(t=0\), and \(t=1\), the result
follows.
\end{proof}
We now prove Proposition \ref{prop:prismatic trans Y_tau}. 
For the ``if'' part, using Lemma \ref{lem:rank1 Ytau sqrt}, we have that any rank-1 $\mY_{\tau i}$ satisfying \eqref{eq:linearconstraintsYtau} can be written as \eqref{eq:Ytaui s1 s2}. Evaluating the l.h.s of \ref{itm:ytau7} gives us $s_1s_2 t\vy+s_1s_2 (1-t)\vy=s_1s_2\vy$. When $\rank(\mY_{i})=1$ and $\mY_i$ satisfies \eqref{eq:A structure}, by Proposition \ref{prop:SO3}, the r.h.s of \ref{itm:ytau7} equals $\mR^{(3)}_i$ and \ref{itm:ytau7} becomes $s_1s_2\vy=\mR^{(3)}_i$. Therefore $\mY_{\tau i}(1:3,7)=s_1s_2 t\vy=t \mR^{(3)}_i$ and  \eqref{eq:prismatic_trans_Y_tau} becomes \eqref{eq:prismatic_trans}. 
  For the ``only if'' part, given rotation $\mR_i$, translation $\mT_i$ and scalar $\tau_i$ that satisfy \eqref{eq:prismatic_trans}, we can use \eqref{eq:Y_itau definition} to construct a rank-1 $\mY_{\tau i}\succeq 0$ that satisfies \eqref{eq:linearconstraintsYtau}. And by Proposition \ref{prop:SO3} we have $\mY_i\succeq 0$ satisfies \eqref{eq:A structure} and $\rank(\mY_{i})=1$.
\section{Full expression of $f(\mY,\mY_\tau)$}\label{sec:full f}
The lifted objective can be written as
\begin{multline}\label{eq:appendix_full_f}
f(\mY,\mY_\tau):=\left\|\mA_r\vec(\mY)-\vb_r\right\|_2^2\\
+\left\|\mA_t\vec(\mY)+\mB_t\vec(\mY_\tau)-\vb_t\right\|_2^2.
\end{multline}

The matrices and vectors in \eqref{eq:appendix_full_f} are given by
\[
\mA_r=\mE_{ee}\mG,
\quad
\vb_r=\vec(\mR_{goal}),
\]
\begin{equation}
    \begin{aligned}
        \mA_t=
\sum_{(i,j)\in\ave_{fk}\cap(\cE_r\cup \cE_s)}({}^i\mT_j^\trans&\otimes\ident_3)\mG_i\mS_i\\
&+\sum_{(i,j)\in\ave_{fk}\cap\cE_p}\tau_l\,\mC_3\mS_i,
    \end{aligned}
\end{equation}
\[
\mB_t=
\sum_{(i,j)\in\ave_{fk}\cap\cE_p}(\tau_u-\tau_l)\mC_\tau\mS_i^\tau,
\quad
\vb_t=\mT_{goal}-\mT_{base}.
\]

Here, $\vec(\mY)$ and $\vec(\mY_\tau)$ denote the stacked vectors
\[
\vec(\mY)=
\begin{bmatrix}
\vec(\mY_{i_1})\\
\vdots\\
\vec(\mY_{i_{n_r}})
\end{bmatrix},
\quad
\vec(\mY_\tau)=
\begin{bmatrix}
\vec(\mY_{\tau k_1})\\
\vdots\\
\vec(\mY_{\tau k_{n_p}})
\end{bmatrix},
\]
where $\{i_1,\dots,i_{n_r}\}=\cV_r$ and $\{k_1,\dots,k_{n_p}\}=\cV_p$.

For each $i\in\cV_r$, define
\[
\mG_i=
\begin{bmatrix}
\mC_1\\
\mC_2\\
\mC_3
\end{bmatrix}\in\IR^{9\times 49},
\quad
\vec(\mR_i)=\mG_i\vec(\mY_i),
\]
where
\[
\mC_1\vec(\mY_i)=\mY_i(1\!:\!3,7),\qquad
\mC_2\vec(\mY_i)=\mY_i(4\!:\!6,7),
\]
and
\[
\mC_3\vec(\mY_i)=
\begin{bmatrix}
\mY_i(2,6)-\mY_i(3,5)\\
\mY_i(3,4)-\mY_i(1,6)\\
\mY_i(1,5)-\mY_i(2,4)
\end{bmatrix}.
\]
The matrix $\mG=\operatorname{blkdiag}(\mG_{i_1},\dots,\mG_{i_{n_r}})$ is the block-diagonal map satisfying
\[
g(\mY)=\mG\vec(\mY).
\]

For each $i\in\cV_p$, define
\[
\mC_\tau\vec(\mY_{\tau i})=\mY_{\tau i}(1\!:\!3,7).
\]

Finally, $\mS_i$ and $\mS_i^\tau$ are block selection matrices that extract $\vec(\mY_i)$ from $\vec(\mY)$ and $\vec(\mY_{\tau i})$ from $\vec(\mY_\tau)$, respectively:
\[
\mS_i\vec(\mY)=\vec(\mY_i),
\quad
\mS_i^\tau\vec(\mY_\tau)=\vec(\mY_{\tau i}).
\]

  \section{Proving local convergences}
\subsection{Proof for Proposition \ref{prop:local converge eig max}}\label{appendix:proof_eig_max_converge}
  Consider another version of Problem \ref{prob:IK2c} (we refer it as Problem \label{prob:IK2d}\hyperref[prob:IK2d]{2d}) where the constraint \eqref{eq:ik2cargmin} is replaced with $f_{t,ee}(\mY,\mY_\tau)=\omat$ and $f_{r,ee}(\mY,\mY_\tau)=\omat$. 
  Every optimal solution to Problem \hyperref[prob:IK2d]{2d} is also optimal to Problem~\ref{prob:IK2c} 
  because the new constraints imply that $\mY$, $\mY _\tau$ is the minimizer of the convex function $f$. By Lemma \ref{lem:lambda1convex}, the objective function of this problem is convex in $\mY$ and $\mY_\tau$, respectively. As a result, Problem \hyperref[prob:IK2d]{2d} is a maximization of a convex function over a convex set.
  Algorithm \ref{alg:ik} can be seen as a gradient approach to Problem \hyperref[prob:IK2d]{2d}. Since $\bar{\cY}$ is bounded, when $k\rightarrow +\infty$, we have 
  $(\mY,\mY_\tau)\rightarrow \partial \bar{\cY}$ 
  and $\tilde{\mY}^*,\tilde{\mY}^*_\tau$ is a local maximizer.
  To see why, for any point $\mY,\mY_\tau$ in the neighborhood $N(\tilde{\mY}^*,\tilde{\mY}^*_\tau)$ such that $\mY,\mY_\tau\in\bar{\cY}$, $f_{t,ee}(\mY,\mY_\tau)=\omat$, and $f_{r,ee}(\mY,\mY_\tau)=\omat$, it holds that $\lambda_1(\tilde{\mY}^*)\geq \lambda_1(\mY)$ and $\lambda_1(\tilde{\mY}^*_\tau)\geq \lambda_1(\mY_\tau)$ because by contradiction if there were a $\hat{\mY},\hat{\mY}_\tau\in N(\tilde{\mY}^*,\tilde{\mY}^*_\tau)$ and $\hat{\mY}=\mY^{k-1}+\hat{\mU}_k,\hat{\mY}_\tau=\mY^{k-1}_\tau+\hat{\mU}^k_\tau$ such that $\lambda_1(\hat{\mY})\geq \lambda_1(\tilde{\mY}^*),\lambda_1(\hat{\mY}_\tau)\geq \lambda_1(\tilde{\mY}^*_\tau)$, then the fact that 
  \begin{multline}
      \{\mU^k,\mU^k_\tau\}{\in}\argmax\big{(}\sum_{i\in\cV_r} Z(\mU^k_{i},\mV^{k-1,(1)}_{i})\\
      {+}\sum_{j\in\cV_p} Z(\mU^{k}_{\tau j},\mV^{k-1,(1)}_{\tau j})\big{)}
  \end{multline}
  would not hold. 

  \subsection{Proof for Proposition \ref{prop:local converge eig max adaptive}}\label{appendix:proof_adaptive_converge}
    We show that in each of the two cases, as $k$ increases, the local convergence holds true.
    
    For the case when the counter $p$ is a finite number, we can find a $c_p<1$ such that Problem \ref{prob:IK4} is feasible.
    According to Theorem \ref{thm:V convergence}, $\lim_{k\rightarrow\infty}w(\mY^k)=\lim_{k\rightarrow\infty}w(\mY^k_\tau)=0$ when \eqref{eq:lin approx c Y} and \eqref{eq:lin approx c Ytau} are satisfied, meaning that $\lim_{k\rightarrow\infty}\sum_{i\in \cV_r}\lambda_1(\mY^k_i)=3n_r$ and $\lim_{k\rightarrow\infty}\sum_{i\in\cV_p}\lambda_1(\mY^k_{\tau i})=2n_p$. As a result, $\{\tilde{\mY}^*,\tilde{\mY}_\tau^*\}$ becomes a maximizer of \eqref{eq:ik2c_obj}. It is a local maximizer because for any $\mY,\mY_\tau$ in the neighborhood $N(\tilde{\mY}^*,\tilde{\mY}^*_\tau)$ that satisfies all the constraints in Problem \ref{prob:IK4}, it holds that $\lambda_1(\tilde{\mY}^*)\geq \lambda_1(\mY)$ and $\lambda_1(\tilde{\mY}^*_\tau)\geq \lambda_1(\mY_\tau)$. 

    For the case when $p\rightarrow\infty$, it holds that 
    \aleq{
    w(\mY^{k-1}) {-}\sum_i^{n_r}\langle\nabla \lambda_{1}(\mY^{k-1}_{i}),\mY^{k}_i-\mY^{k-1}_{i}\rangle> c_pw(\mY^{k-1})\label{eq:W complement}
    }
    for every $p\in\mathbb{N}$.
    We then claim that $\mY^{k-1}$ is a local minimum of $w(\mY)$. We prove this claim by way of contradiction. Assume that $\mY^{k-1}$ is \emph{not} a local minimum; by definition, then, there exist an arbitrarily small neighborhood $N(\mY^{k-1})$, a point $\hat{\mY}^k\in N(\mY^{k-1})$, and a sufficiently small $0<\epsilon<w(\mY^{k-1})$ such that the cost can be improved by $\epsilon$, i.e.
    \begin{equation}
        \begin{aligned}
            w(\mY^{k-1}) -\sum_i^{n_r}\langle\nabla \lambda_{1}(\mY^{k-1}_{i}),\hat{\mY}^{k}_i-\mY^{k-1}_{i}\rangle\\
    \leq w(\mY^{k-1})-\epsilon=\hat{c}_pw(\mY^{k-1}),\label{eq:W not local min}
        \end{aligned}
    \end{equation}
    where $\hat{c}_p=\frac{w(\mY^{k-1})-\epsilon}{w(\mY^{k-1})}<1$. With this choice of $\hat{c}_p$, however, \eqref{eq:W not local min} contradicts \eqref{eq:W complement}, proving our claim that $\tilde{\mY}^*,\tilde{\mY}^*_\tau$ is a local maximizer of \eqref{eq:ik2c_obj}.
    Moreover, since we are minimizing a concave function $w(\mY)$, this local optimum needs to be on the boundary.

\section{Inflation strategy for polyhedron generation} \label{appendix:inflation}
Given seed points $\{s_i\}_{i=1}^N \subset \mathbb{R}^3$, workspace bounds $(\ell,u)$, and axis-aligned obstacle boxes $\{(\ell^o_k,u^o_k)\}_{k=1}^M$, we construct collision-free cuboids by expanding each seed along coordinate axes. Each cuboid is initialized at $s_i$ and iteratively enlarged by moving its faces outward while maintaining feasibility with respect to workspace containment and obstacle avoidance. A line search (via bisection) determines the maximal admissible expansion per face. The process terminates when no further expansion is possible.

\begin{algorithm}[htb]
\caption{Axis-Aligned Cuboid Inflation}
\begin{algorithmic}[1]
\Require $\{s_i\}_{i=1}^N$, $(\ell,u)$, $\{(\ell^o_k,u^o_k)\}_{k=1}^M$
\Ensure $\{(\ell_i,u_i)\}_{i=1}^N$

\State For all $i$: $\ell_i \gets s_i,\; u_i \gets s_i$; mark all faces active
\Repeat
    \State $\texttt{progress} \gets \texttt{false}$
    \For{each $i$ and each active face $f$}
        \State \parbox[t]{180pt}{$\alpha \gets \max\{\alpha \in [0,\Delta] : \text{expanding face } f \text{ by }\alpha \text{ is feasible}\}$}
        \If{$\alpha > 0$}
            \State \parbox[t]{140pt}{update $(\ell_i,u_i)$ along face $f$; $\texttt{progress} \gets \texttt{true}$}
        \Else
            \State \parbox[t]{140pt}{deactivate face $f$}
        \EndIf
    \EndFor
\Until{not $\texttt{progress}$}
\State \Return $\{(\ell_i,u_i)\}_{i=1}^N$
\end{algorithmic}
\end{algorithm}

\section{Collision models}
The collision geometries used for computing obstacle-aware IK are presented in Figure~\ref{fig:collision_models}.
\begin{figure}[hbt]
  \centering
  \subfloat[Low Density ($n_b=5$)]{%
    \includegraphics[width=0.78\columnwidth]{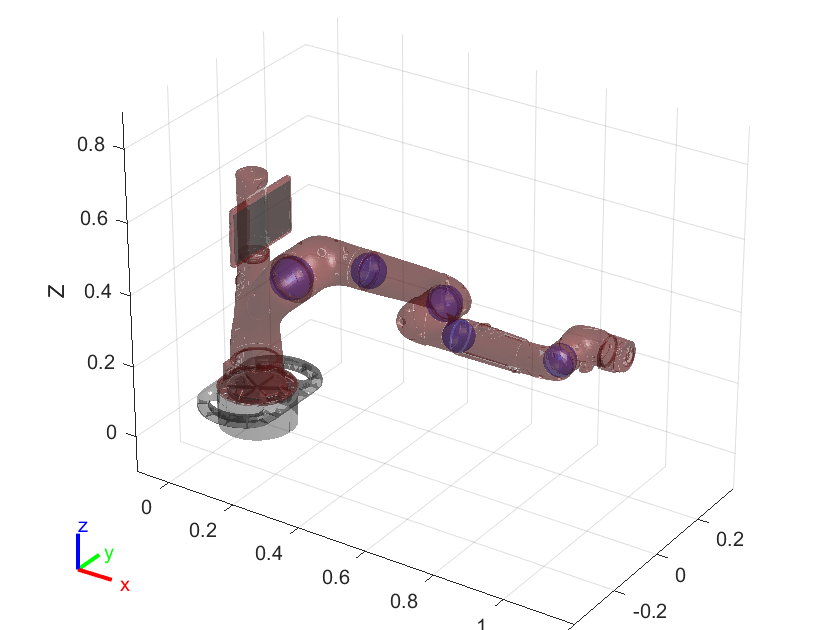}%
    \label{fig:density_1}}\\[2pt]
  \subfloat[Mid Density ($n_b=10$)]{%
    \includegraphics[width=0.78\columnwidth]{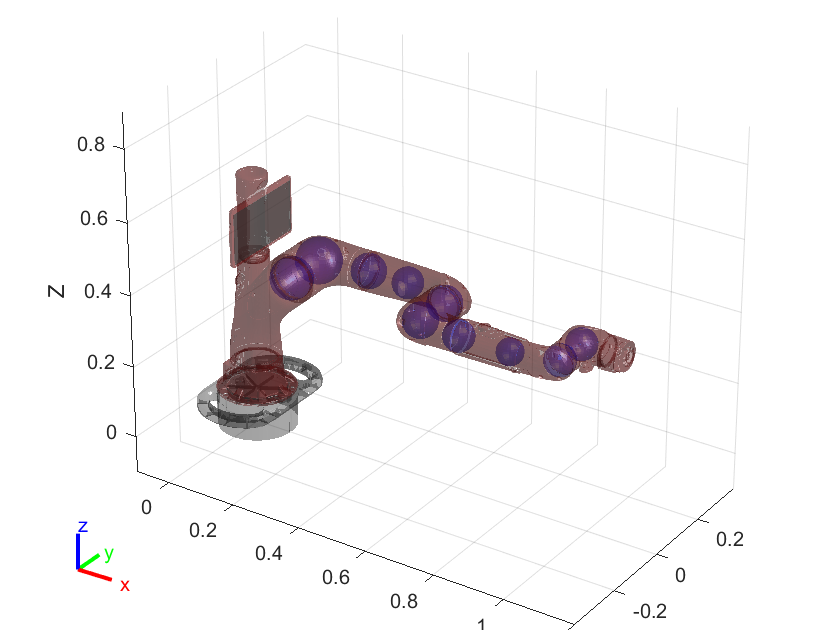}%
    \label{fig:density_2}}\\[2pt]
  \subfloat[High Density ($n_b=22$)]{%
    \includegraphics[width=0.78\columnwidth]{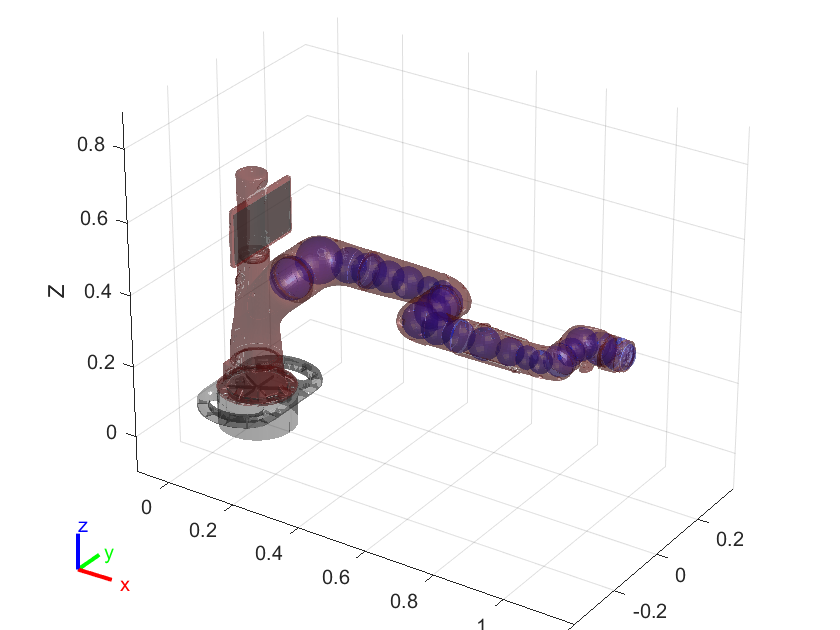}%
    \label{fig:density_3}}\\[2pt]
  \subfloat[Drake Collision]{%
    \includegraphics[width=0.78\columnwidth]{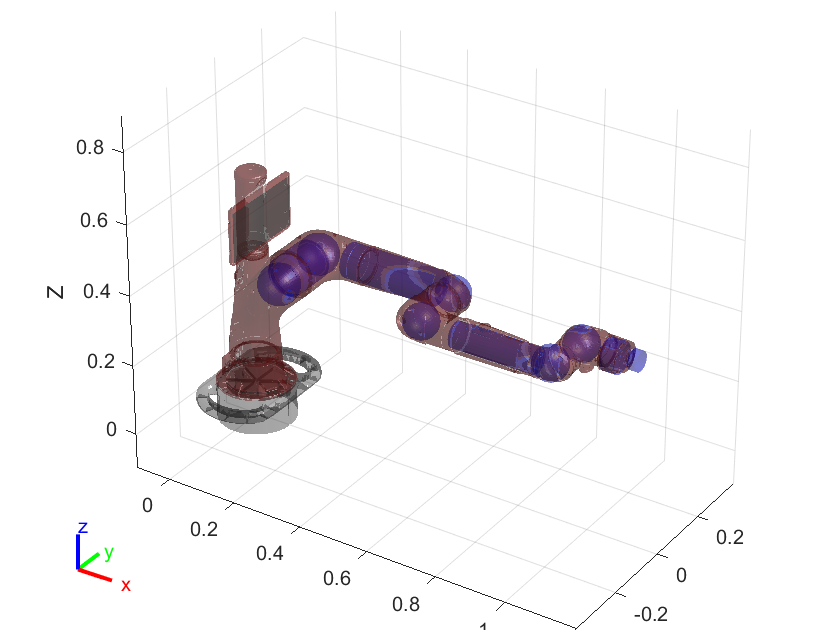}%
    \label{fig:density_urdf}}
  \caption{Different collision geometry models of Sawyer.}
  \label{fig:collision_models}
\end{figure}

\end{document}